\documentclass[%
  onecolumn
   , hidempi
]{mpi2015-cscpreprint}

\usepackage[american]{babel}

\usepackage{graphicx}
\usepackage{wrapfig}
\usepackage{amssymb}
\usepackage{amsthm}
\usepackage[mathscr]{eucal}
\usepackage{amsmath}
\usepackage{breqn}

\numberwithin{equation}{section}
\numberwithin{figure}{section}
\numberwithin{table}{section}

\usepackage{import}
\usepackage{xifthen}
\usepackage{pdfpages}
\usepackage{placeins}
\usepackage{transparent}
\usepackage{sidecap}    
\usepackage{floatrow}
\usepackage{float}
\usepackage{multicol}
\usepackage{todonotes}
\usepackage{algorithm}
\usepackage[noend]{algpseudocode}

\makeatletter
\def\algbackskip{\hskip-\ALG@thistlm}
\makeatother

\usepackage{caption}
\usepackage{subcaption}
\usepackage{arydshln}

\definecolor{lightgray}{gray}{0.9}
\usepackage{color}
\definecolor{bluegreen}{rgb}{0.0, 0.87, 0.87}

\newtheorem{definition}{Definition}
\newtheorem{example}{Example}

\newtheorem{proposition}{Proposition}
\newcommand*{\bl}[1]{\mathbf{#1}}

\newcommand{\Hamiltonian}{\mathcal{H}}
\DeclareMathOperator{\Diff}{d\!}

\usepackage{mymacros}
\usepackage{todonotes}

\begin{document}
  

\title{Data-Driven Identification of Quadratic Representations for Nonlinear Hamiltonian Systems using Weakly Symplectic Liftings}

\author[$\ast$]{Süleyman Y\i ld\i z}
\affil[$\ast$]{Max Planck Institute for Dynamics of Complex Technical Systems, 39106 Magdeburg, Germany.\authorcr
	\email{yildiz@mpi-magdeburg.mpg.de}, \orcid{0000-0001-7904-605X}
}

\author[$\ast\ast$]{Pawan Goyal}
\affil[$\ast\ast$]{Max Planck Institute for Dynamics of Complex Technical Systems, 39106 Magdeburg, Germany.\authorcr
  \email{goyalp@mpi-magdeburg.mpg.de}, \orcid{0000-0003-3072-7780}
}

\author[$\dagger$]{Thomas Bendokat}
\affil[$\dagger$]{Max Planck Institute for Dynamics of Complex Technical Systems, 39106 Magdeburg, Germany.\authorcr
	\email{bendokat@mpi-magdeburg.mpg.de}, \orcid{0000-0002-0671-6291}
}

\author[$\dagger\ddagger$]{Peter Benner}
\affil[$\dagger\ddagger$]{Max Planck Institute for Dynamics of Complex Technical Systems, 39106 Magdeburg, Germany.\authorcr
  \email{benner@mpi-magdeburg.mpg.de}, \orcid{0000-0003-3362-4103}
}
\affil[$\dagger\ddagger$]{Otto von Guericke University,  Universit\"atsplatz 2, 39106 Magdeburg, Germany\authorcr
  \email{peter.benner@ovgu.de} 
  \vspace{-0.5cm}
}
  
\shorttitle{Identification of Quadratic Symplectic Representations}
\shortauthor{S. Y\i ld\i z, P. Goyal, T. Bendokat, P. Benner}
\shortdate{}
  

  
\abstract{%
We present a framework for learning Hamiltonian systems using data. This work is based on a lifting hypothesis, which posits that nonlinear Hamiltonian systems can be written as nonlinear systems with cubic Hamiltonians. By leveraging this, we obtain quadratic dynamics that are Hamiltonian in a transformed coordinate system. To that end, for given generalized position and momentum data, we propose a methodology to learn quadratic dynamical systems, enforcing the Hamiltonian structure in combination with a weakly-enforced symplectic auto-encoder. The obtained Hamiltonian structure exhibits long-term stability of the system, while the cubic Hamiltonian function provides relatively low model complexity. For low-dimensional data, we determine a higher-dimensional transformed coordinate system, whereas for high-dimensional data, we find a lower-dimensional coordinate system with the desired properties. We demonstrate the proposed methodology by means of both low-dimensional and high-dimensional nonlinear Hamiltonian systems. 	
}

\novelty{
	\begin{itemize}
		\item Inspired by quadratic lifting, allowing to rewrite nonlinear systems as quadratic systems in a lifted coordinate system, we discuss a lifting principle for nonlinear Hamiltonian systems.
		\item We propose a data-driven approach to learning a quadratic Hamiltonian coordinate system by means of weakly-enforced symplectic auto-encoders so that
		\begin{itemize}
			\item the dynamics in the learned coordinate system can be given by a quadratic system, and
			\item the underlying Hamiltonian function is cubic.
		\end{itemize}
		\item For high-dimensional data, we discuss learning a reduced coordinate system so that the above goals are achieved. Precisely,  we learn the latent dynamics, which are the low-dimensional representation of the high-dimensional data, by using a weakly-enforced symplectic autoencoder and strongly enforced latent Hamiltonian dynamics using the structure of the Hamiltonian system from only the data. This then aligns with non-intrusive model-order reduction by nonlinear projection. 
		\item By means of several examples, including high-dimensional ones, we demonstrate the proposed methodology.
	\end{itemize}
} 
\maketitle

\section{Introduction}\label{sec:intro}

	Hamiltonian dynamics are ubiquitous as a powerful mathematical tool in modeling complex physical dynamical systems \cite{arnol1989mathematical}. Classically, they are used in topics ranging from celestial mechanics \cite{siegel1995lectures} over fluid mechanics \cite{salmon1988hamiltonian} to Schrödinger equations in quantum mechanics \cite{faou2012geometric}. The coordinates in which Hamiltonian systems operate are split into generalized positions and momenta, which need to be identified from given data in order to fit a physical model to observations.

	The construction of models that can accurately capture and predict the dynamics of highly complex systems has been of interest for several decades if not centuries; see, e.g., \cite{crutchfield2012between} and references therein.
	Recently, the powerful approximation capabilities of neural networks have brought researchers in many fields closer to understanding complicated systems. Neural networks have been successfully studied for predicting complex dynamical systems \cite{lusch2018deep,vlachas2018data}, improving turbulence models \cite{fang2020neural,duraisamy2019turbulence}, classifying time series \cite{karim2019multivariate}, and studying differential equations (DEs) \cite{raissi2017inferring,rudy2017data}. On the other hand, some studies have focused on learning energy functions \cite{chen2019symplectic,greydanus2019hamiltonian,finzi2020simplifying,offen2022symplectic} rather than dynamical systems, so that they have inherited the long-term stability properties of Hamiltonian systems.
	Recently, some studies with similar goals developed data-driven operator inference (OpInf) models that preserve the Lagrangian \cite{sharma2022preserving} and Hamiltonian structure \cite{sharma2022hamiltonian,gruber2023canonical}.
	
	In this work, we are interested in learning quadratic Hamiltonian systems explaining given trajectory data from two different perspectives: lifting transformations \cite{goyal2022generalized}, and nonlinear symplectic model order reduction \cite{buchfink2021symplectic} with weak enforcement of the transformation to be symplectic. In fact, we learn the quadratic Hamiltonian systems directly from data, without needing to resort to the orignal dynamical equations.
	
	In short, given data from a Hamiltonian system, we would like to learn the dynamics in a structure-preserving way, while achieving low model complexity and having the option to reduce the dimension for high-dimensional data. This is tackled by approximating structure-preserving mappings with weakly-enforced symplectic auto-encoders and modeling of the dynamics with a quadratic Hamiltonian system.

In \cite{goyal2022generalized}, a unified approach, namely lifting transformations, is used to approximate general nonlinear systems. In the case where the dynamical system is known, one can manually design lifted variables. However, lifting the dynamical system does not necessarily lead to a Hamiltonian system. Therefore, we weakly force the lifting transformations to be symplectic by exploiting symplectic embeddings and strictly enforce the Hamiltonian structure of the dynamics equations. 

The second application of our approach lies in dimensionality reduction of Hamiltonian systems. Learning reduced-order models for Hamiltonian systems comes with some practical challenges.
Without enforcing preservation of the Hamiltonian structure in the reduced-order model, it can quickly lose accuracy~\cite{peng2016symplectic}.
A well-established approach to preserve the symplectic structure is to use linear symplectic projections, i.e., proper symplectic decomposition \cite{peng2016symplectic,afkham2017structure}. However, for Hamiltonian systems with slow decay of the Kolmogorov-$n$-width, this approach may not be feasible due to the slow decay of singular values, thus creating efficiency problems. Furthermore, for non-linear Hamiltonian functions, hyperreduction methods such as \cite{pagliantini2022gradient} are required for efficient computability of the reduced-order model. On the other hand, some of the nonlinear symplectic model order reduction techniques developed in \cite{pagliantini2021dynamical,musharbash2020symplectic}. For a broader view on structure preserving model order reduction techniques, we refer to recent work \cite{hesthaven2022reduced}.

In \cite{buchfink2021symplectic}, the  authors construct a weakly symplectic deep convolutional autoencoder as an example to practically address this problem. We take a similar approach and use a weakly-symplectic auto-encoder to map data from a Hamiltonian system to a learned \emph{quadratic} Hamiltonian system, thereby reducing model complexity considerably. The simple quadratic structure allows direct learning of the Hamiltonian system, without learning the Hamiltonian itself and without the need to calculate its gradient, which makes it possible to learn Hamiltonian dynamics even when there is no information about the underlying system. For the purpose of learning reduced dynamics, similar as in MOR, we show by numerical examples that this quadratic Hamiltonian system can be of much lower dimension than the original full order model. Learning the quadratic Hamiltonian system from data has the further advantage that no hyper-reduction methods are needed for nonlinear systems, and the reduced-order model can thus be efficiently computed. Furthermore, as our approach learns the reduced dynamics directly, we do not need to take the gradient through the auto-encoder to simulate the learned models.

The recent article \cite{sharma2023symplectic} can be seen as a complementary approach to our method for reducing the order of Hamiltonian systems. While we learn a quadratic Hamiltonian system with a general non-linear weakly-enforced symplectic auto-encoder, in \cite{sharma2023symplectic} two different versions of quadratic symplectic auto-encoders are studied, which are then used for model order reduction leading to a general non-linear Hamiltonian system. Moreover, we learn the reduced dynamics directly from data, while \cite{sharma2023symplectic} studies model order reduction, i.e., resorting to the Hamiltonian of the full-order model. In future work, a combination of both approaches seems worthwhile.

The paper is structured as follows.
In Section~\ref{sec:Theory}, we introduce the necessary mathematical background to embed Hamiltonian systems in a structure-preserving way into a higher-dimensional space and define quadratic Hamiltonian systems.
In Section~\ref{sec:learningliftedsystem} we describe the auto-encoder structure to lift Hamiltonian systems.
In Section~\ref{sec:lowdimsystem} we adapt the theory to learn low-dimensional quadratic representations of high-dimensonal data with weakly-enforced symplectic auto-encoders.
In Section~\ref{sec:num} we show the applicability of the approach, for low-dimensional systems in Subsection~\ref{sec:lowdim} and for weakly symplectic reduction of high-dimensional systems in Subsection~\ref{sec:highdim}.
Section~\ref{sec:conc} concludes the paper.
Implementation details can be found in Appendix~\ref{appendix:A}.

\section{Background}\label{sec:Theory}
In this section, we provide the necessary theoretical background needed for the derivation of our learning approach for Hamiltonian systems.
\subsection{Hamiltonian Systems and Symplectic Embedding}
The governing equations of canonical Hamiltonian systems are Hamilton's equations, namely
\begin{equation}
    \label{eq:HamiltonianEquations}
    \dot{x}(t) = J_{2n} \nabla_x \Hamiltonian(x(t)) \in \R^{2n},
\end{equation}
where  $x(t) = (q(t),p(t)) \in \R^{2n}$ with the canonical coordinates $q$ and $p$ being generalized positions and momenta, respectively, 
$$J_{2n} := \begin{bmatrix} 0 & I_n\\ -I_n & 0 \end{bmatrix} \in \R^{2n \times 2n},$$
and $\nabla_x$ denotes the gradient with respect to $x$. Moreover,  we consider an initial condition $x(0) = x_0 = (q_0,p_0) \in \R^{2n}$.
The Hamiltonian function $\Hamiltonian \colon \R^{2n} \to \R$ describes the energy of the system and is preserved along the solution trajectories. Next, we discuss the definition of a symplectic embedding, which plays  an important role in our later discussions. 

\begin{definition}[Symplectic Embedding for Vector Spaces]
    A \emph{symplectic embedding} of $\R^{2n}$ into $\R^{2N}$ is a homeomorphism $\psi \colon \R^{2n} \to  \psi(\R^{2n})\subset \R^{2N}$ for which the Jacobian $\Diff \psi_x \in \R^{2N \times 2n}$ fulfills
    \begin{equation}
        \label{eq:symplecticEmbedding}
        (\Diff \psi_x)^T J_{2N} \Diff \psi_x = J_{2n}
    \end{equation}
    at every $x \in \R^{2n}$.
\end{definition}
It is immediate to see that a symplectic embedding is a smooth embedding in the sense of differential geometry \cite[Section 22, p. 568]{lee2012smooth}, as the Jacobian has full rank at every point. The Jacobian is therefore injective and $\psi$ is an immersion. This furthermore implies $N \geq n$. Therefore, a symplectic embedding is also called a \emph{symplectic lifting}.

The image $\psi(\R^{2n})$ of $\R^{2n}$ under $\psi$ is a manifold, for which $\psi^{-1}$ defines a global chart, by definition \cite{lee2012smooth}.

\begin{proposition}[Equivalent Embedded System]
    Let $\psi \colon \R^{2n} \to \R^{2N}$ be a symplectic embedding and define $z_0 := \psi(x_0) \in \R^{2N}$.
    Then, the system~\eqref{eq:HamiltonianEquations} is equivalent to the embedded system
    \begin{equation}
        \label{eq:EmbeddedSymplecticSystem}
        \dot z(t) = J_{2N} \nabla_z \Hamiltonian(\psi^{-1}(z(t))),
    \end{equation}
    in the sense that $z(t) := \psi(x(t))$ solves \eqref{eq:EmbeddedSymplecticSystem} for all $t \in [0,\infty)$.
\end{proposition}
\begin{proof}
    For any $z \in \psi(\R^{2n})$, it holds with the chain rule that
    \begin{align*}
        J_{2N}\nabla_z \Hamiltonian(\psi^{-1}(z)) &= J_{2N}(\Diff\, (\Hamiltonian \circ \psi^{-1})_z(z))^T \\&= J_{2N}(\Diff {\psi^{-1}}_z)^T (\Diff \Hamiltonian_x(\psi^{-1}(z)))^T\\
		&= J_{2N} (\Diff {\psi^{-1}}_z)^T \nabla_x \Hamiltonian(\psi^{-1}(z)) \\ &= \Diff \psi_x J_{2n} \nabla_x \Hamiltonian(\psi^{-1}(z)).
    \end{align*}
	The last equality comes from the fact that the pushforward of the Hamiltonian vector field of $\Hamiltonian$ by $\psi$ is equal to the Hamiltonian vector field of $\Hamiltonian \circ \psi$.
	On a matrix level, this means that $\psi^{-1}$ can be extended around $\psi(R^{2n})$ so that its Jacobian is equal to the symplectic left inverse of the Jacobian of $\psi$, i.e., $\Diff {\psi^{-1}}_{\psi(x)} = J_{2n}^T\Diff \psi_x^T J_{2N}$.
	
    As $z(t) := \psi(x(t))$ implies $\dot z(t) = \Diff \psi_{x(t)} \dot x(t) = \Diff \psi_{x(t)} J_{2n} \nabla_x\Hamiltonian(x(t))$, the claim follows.
\end{proof}

As $\psi$ is a symplectic lifting, the system for $z$ is called a \emph{symplectic lifting of the system for $x$}.

\subsection{Quadratic Hamiltonian Representations}
There are many possibilities to construct a symplectic embedding of a nonlinear Hamiltonian system. However, in this work, we are seeking to identify a particular higher dimensional or lifted space so that a quadratic system can describe the dynamics in the lifted space. Moreover, the Hamiltonian in the lifted space is a cubic polynomial function. To briefly describe the lifting procedure, we consider the system of ODEs
\begin{equation}\label{eq:ode}
\dot x(t) = f(x(t)), \quad x(t) \in \mathbb{R}^n.
\end{equation}
The quadratic lifting transformation \cite{savageau1987recasting,gu2011qlmor,qian2020lift,goyal2022generalized} can be obtained by defining a transformation $z(t) = \psi(x(t))\in \mathbb{R}^N$ for $N \geq n$ such that the transformed system \eqref{eq:ode} satisfies
\begin{equation}\label{eq:lifted_sys}
\dot z = \mathcal{A} + \mathcal{B}z + \mathcal{C}z \otimes z, \quad z(t) \in \mathbb{R}^{N}.
\end{equation}
We illustrate the quadratic lifting for nonlinear systems by means of a nonlinear oscillator example. 

\begin{example}[Nonlinear Oscillator]
	Consider the nonlinear (an-harmonic) oscillator \cite{mattheakis2022hamiltonian} with the Hamiltonian $\mathcal H (q,p)= \frac{p^2}{2}+\frac{q^2}{2}+\frac{q^4}{4}$. The associated Hamiltonian system for this problem is given by
	\begin{equation}\label{eqn:osc}
	\begin{aligned}
	&\dot{q}=p,\\
	&\dot{p}=-(q+q^3).
	\end{aligned}
	\end{equation}
	We demonstrate the lifting transformation by introducing the variables $w_1=q$, $w_2=p$, and $w_3=q^2$. With the new variable $w_3$, the equations of motion for the oscillator \cref{eqn:osc} can be written as lifted dynamical system
	\begin{equation}\label{eqn:lift_osc}
	\begin{aligned}
	&\dot{w}_1=w_2,\\
	&\dot{w}_2=-(w_1+w_1w_3),\\
	&\dot{w}_3=2w_1w_2,\\
	\end{aligned}
	\end{equation}
	which is a quadratic system. Moreover, one can also define an inverse mapping from $(w_1,w_2,w_3)$ to $(q,p)$. However, it is easy to note that the system in \cref{eqn:lift_osc} is not a canonical Hamiltonian system since it is odd-dimensional. We further note that even introducing new variables to make the lifted system \cref{eqn:lift_osc} even dimensional does not necessarily result in a Hamiltonian system.
	
	Notably, the theory of generating functions can be used to construct quadratic Hamiltonian systems. Generating functions relate two canonical Hamiltonian systems via canonical transformations. There are four fundamental generating functions. For a detailed overview of generating functions, we refer to the book \cite{strauch2009classical}. To illustrate this for the nonlinear oscillator, we suppose $\hat{p}=q^2$. Then, using a generating function of type 1, one can find $F_1(q,\hat{q})=-\hat{q}q^2$, so that $p=\frac{\partial F_1}{\partial q}=-2\hat{q}q$,  implying $ p=-2 \hat{q}\hat{p}^{1/2} $ and $ q=\hat{p}^{1/2} $. The new Hamiltonian with new variables becomes $\hat{\mathcal H}= 2\hat{q}^2\hat{p}+\frac{\hat{p}}{2}+\frac{\hat{p}^2}{4} $, which is cubic; hence, the underlying dynamics are given by a quadratic system.
\end{example}

Inspired by the above example, in this work, we seek to identify a symplectic space to lift to. The desired properties can be achieved when the lifted system \cref{eq:lifted_sys} satisfies \cref{eq:EmbeddedSymplecticSystem} with a symplectic lifting $\psi$ fulfilling \cref{eq:symplecticEmbedding}. For this, we first define quadratic Hamiltonian systems for our reference. 

\begin{definition}[Quadratic Hamiltonian System]
	A \emph{quadratic Hamiltonian system} is a Hamiltonian system \eqref{eq:HamiltonianEquations} for which the Hamiltonian function is cubic, i.e.,
	\begin{equation*}
	\Hamiltonian(x) = A^Tx + B^T(x \otimes x) + C^T (x \otimes x \otimes x),
	\end{equation*}
	where $A \in \R^{2n}$, $B \in \R^{(2n)^2}$, $C \in \R^{(2n)^3}$, and $\otimes$ denotes the Kronecker product.
\end{definition}
The simple structure of quadratic Hamiltonian systems allows enforcing the Hamiltonian condition directly onto the system, without having to compute the gradient of the Hamiltonian function.
\begin{proposition}
	A quadratic system of ODEs
	\begin{equation}
	\label[eq]{quadraticODEHamiltonian}
	\dot x = \mathcal{A} + \mathcal{B}x + \mathcal{C}(x \otimes x),\quad x(t) \in \R^{2n},
	\end{equation}
	where $\mathcal{A} \in \R^{2n}$, $\mathcal{B} \in \R^{2n \times 2n}$ and $\mathcal{C} \in \R^{2n \times (2n)^2}$,
	is a quadratic Hamiltonian system if and only if $J_{2n}^T \mathcal{B}$ is a symmetric matrix and there is a symmetric tensor $\mathcal{T}\in \R^{2n \times 2n \times 2n}$ for which
	\[
	\mathcal{T}_u (x \otimes x )= J_{2n}^T \mathcal{C} (x \otimes x)
	\]
	holds for all $x \in \R^{2n}$, where $\mathcal{T}_u \in \R^{2n \times (2n)^2}$ is the unfolding of $\mathcal{T}$ by frontal slices in the sense of \cite[Figure 2.2 (c)]{KolB09} by concatenating the frontal slices in a row. 
\end{proposition}
\begin{proof}
	By definition, \eqref{quadraticODEHamiltonian} is a quadratic Hamiltonian system if and only if there is a cubic function $\Hamiltonian(x) = A^Tx + B^Tx \otimes x + C^T x \otimes x \otimes x$ such that $\dot x = J_{2n} \nabla_x \Hamiltonian(x)$, which is equivalent to
	\begin{align*}
	\nabla_x \Hamiltonian(x) &= \nabla_x \left(A^Tx + B^T(x \otimes x) + C^T (x \otimes x \otimes x)\right) 
	\\ &=  J_{2n}^T (\mathcal{A} + \mathcal{B}x + \mathcal{C}x \otimes x).
	\end{align*}
	Since there is a bijection between homogenous polynomials and symmetric tensors \cite[p. 6]{comon2008symmetric}, the claim follows. The unfolding is done in order to be able to use the Kronecker product and fit the equations into the matrix framework.
\end{proof}

For many smooth nonlinear systems, there exist guaranteed liftings which allow us to rewrite nonlinear systems as quadratic systems, see, e.g., \cite{savageau1987recasting,gu2011qlmor}. However, there is currently no established result ensuring the existence of a symplectic lifting for nonlinear Hamiltonian systems to higher dimensions where the dynamics can be represented by quadratic Hamiltonian systems.  Exploring this aspect remains an intriguing theoretical endeavor for future research. In this work, however, we hypothesize the existence of such a system and focus on learning such a symplectic lifting/embedding by means of suitable optimization problems, which we discuss next. 

\section{Learning the Lifted Quadratic Symplectic Representation}
\label{sec:learningliftedsystem}
Here, we describe our methodology to learn a weakly symplectic lifting to map from a given canonical Hamiltonian system to a latent canonical quadratic Hamiltonian system, which we visualize in \Cref{fig:lifting}. 
\begin{figure*}[tb]
	\includegraphics{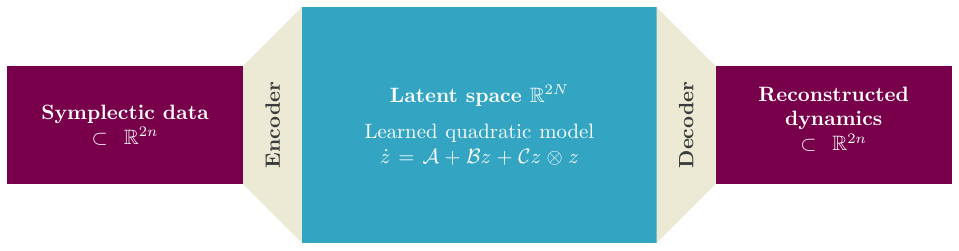}    
	\caption{The auto-encoder structure of the symplectic lifting method. Here, the encoder $\psi: \R^{2n} \to \R^{2N}$ is weakly enforced to be a symplectic mapping and the quadratic system is enforced to be Hamiltonian.}
	\label{fig:lifting}   
\end{figure*}
The first ingredient to it is to define lifted coordinates $z(t)$ using a classical auto-encoder loss as follows:
\begin{equation}\label{eq:loss1}
\mathcal L_{\text{encdec}}=\| x(t)-\phi(\psi(x(t)))\|,
\end{equation}
where $\psi(x(t))=z(t)=(\hat{q}(t),\hat{p}(t))$ and $\phi(z(t))=\tilde{x}(t) = (q(\hat{q}(t)),p(\hat{p}(t)))\approx x(t)$. However, the mapping obtained through \cref{eq:loss1} does not necessarily yield a symplectic mapping. To approximate a symplectic map, we use \eqref{eq:symplecticEmbedding} and define a symplectic loss as follows:
\begin{equation}\label{eq:loss2}
\mathcal L_{\text{symp}}=\| (\Diff \psi_x)^T J_{2N} \Diff \psi_x - J_{2n}\|.
\end{equation}
Furthermore, we assume that time derivatives of states are accessible. Thus, we compute the time derivatives of the lifted space $z$ using the chain-rule. Hence, we add the following term in the loss function:
\begin{equation}\label{eq:loss3}
\begin{aligned}
\mathcal L_{\dot z\dot x} &= \left\| \Diff \psi_{x(t)} \dot x(t) -J_{2N} \nabla_z \Hamiltonian(\psi^{-1}(z(t))) \right\| \\
&=\| \Diff \psi_{x(t)} \dot x(t) - (\mathcal{A} + \mathcal{B}z(t) + \mathcal{C}z(t) \otimes z(t))\| 
\end{aligned}
\end{equation}
with $z = \psi(x)$. Finally, to obtain a quadratic Hamiltonian system, we combine all these losses defined in \cref{eq:loss1,eq:loss2,eq:loss3}. Hence we have the total loss as a weighted sum of these loss functions, given by
\begin{equation}\label{eq:total_loss}
\mathcal L = \lambda_1\mathcal L_{\text{encdec}} +\lambda_2 \mathcal L_{\text{symp}}  +\lambda_3\mathcal L_{\dot z\dot x} ,
\end{equation}
where $\lambda_{\{1,2,3\}}$ are hyper-parameters. Finally, we optimize the total loss function with respect to auto-encoder parameters and $\mathcal{A}, \mathcal{B}, \mathcal{C}$.
The details of the implementation and auto-encoders are given in Appendix~\ref{appendix:A}. Finally, we optimise all parameters in \eqref{eq:total_loss} at the same time.

We remark that we do not enforce the homeomorphism property, but only enforce \eqref{eq:symplecticEmbedding}, i.e., the condition that the encoder approximates a symplectic immersion and is therefore locally invertible, and that the encoder is invertible on the training data.

However, note that using the loss functions only trains these properties on the training data. When the encoder is applied to data far outside the training data distribution, it might neither be invertible nor close to a symplectic one.

\section{Low-dimensional Quadratic Symplectic Representation of High-dimensional Data}
\label{sec:lowdimsystem}
Thus far, we have discussed how nonlinear Hamiltonian systems can be lifted to higher dimensional quadratic Hamiltonian systems and how they can be learned by means of data. However, there are many Hamiltonian systems which are high-dimensional, specially coming from partial differential equations. Furthermore, it is known that high dimensional dynamic data often evolve in a lower-dimensional subspace. Therefore, in these cases, we aim to learn lower-dimensional coordinates for high-dimensional data so that the learned dynamics are not only Hamiltonian but can also be used to describe the dynamics of the high-dimensional system, as depicted in \Cref{fig:reduction}.

\begin{figure*}[tb]
	\includegraphics{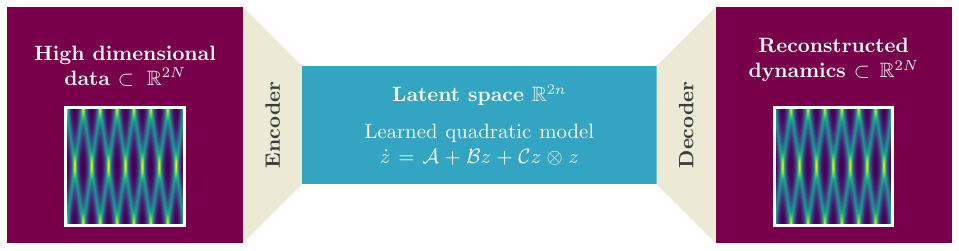}
	\caption{The auto-encoder structure of the symplectic reduction method. Here, the decoder $\phi: \R^{2n} \to \R^{2N}$ is weakly enforced to be a symplectic mapping and the quadratic system is enforced to be Hamiltonian.}
	\label{fig:reduction}   
\end{figure*} 

However, there is a subtlety compared to the method discussed in the previous section. It is worth noting that \cref{eq:loss2} will weakly enforce symplecticity in the case of symplectic lifting. On the other hand, for high-dimensional data, the quadratic system is of lower dimension than the original data; thus, we actually need to weakly enforce that the \emph{decoder} $\phi$, and not the encoder $\psi$, of the auto-encoder is a symplectic embedding from the quadratic model to the original high-dimensional system.

Since the high-dimensional data, particularly coming from partial differential equations, have a lot of spatial coherency, we make use of a deep convolutional auto-encoder (DCA), which is computationally efficient.
Moreover, to enforce the symplecticity condition in the loss function, we use a weakly symplectic deep convolutional auto-encoder with essentially the same conditions as in \cite[Section 3.3]{buchfink2021symplectic}. This means that in the symplectic reduction, instead of \cref{eq:loss2}, we use the following loss terms for symplectic loss:
\begin{equation}
\tilde{\mathcal L}_{\text{symp}}=\| (\Diff \phi_z)^T J_{2N} \Diff \phi_z - J_{2n}\|.
\end{equation}
As we have switched the role of $n$ and $N$ in the reduction case, \eqref{eq:loss3} is replaced by
\begin{equation}
	\begin{aligned}
	\tilde{\mathcal L}_{\dot z\dot x} &= \left\| \Diff \psi_{x(t)} \dot x(t) -J_{2n} \nabla_z \Hamiltonian(\psi^{-1}(z(t))) \right\| \\
	&=\| \Diff \psi_{x(t)} \dot x(t) - (\mathcal{A} + \mathcal{B}z(t) + \mathcal{C}z(t) \otimes z(t))\|.
	\end{aligned}
\end{equation}

	The loss of the auto-encoder $\mathcal L_{\text{encdec}}$ is given by \eqref{eq:loss1} as in the lifting case. Hence, the total loss is calculated via
\begin{equation}
\mathcal L = \lambda_1\mathcal L_{\text{encdec}} +\lambda_2 \tilde{\mathcal L}_{\text{symp}}  +\lambda_3\tilde{\mathcal L}_{\dot z\dot x},
\end{equation}
which is then used to learn a suitable embedding.

\section{Numerical Experiments}\label{sec:num}	
In this section, we examine the performance of the proposed methodology in two scenarios: low-dimen\-sio\-nal dynamical systems and high-dimensional dynamical systems.
For the low-dimensional case, we investigate three different examples: the simple pendulum, an an-harmonic oscillator, and the Lotka-Volterra equations. These examples allow the validation of the proposed methodology due to the known solutions (except for the Lotka-Volterra example), so that we can compare to the ground truth. For the high-dimensional case, we study the linear wave and nonlinear Schrödinger equations. All the experiments are done using \texttt{PyTorch} on a machine with an \texttt{Intel\textsuperscript{\tiny\textcopyright} Core\textsuperscript{\tiny TM} i5-12600K} CPU and \texttt{NVIDIA RTX\textsuperscript{\tiny TM} A4000(16GB)} GPU. To preserve the symplectic structure after time discretization, we have used the implicit midpoint rule as time integrator.
In the case of symplectic lifting for all low-dimensional examples, we set the dimension of the latent space---for which the dynamics are quadratic and have a constant cubic Hamiltonian---to four. In the case of symplectic reduction we set the dimension of the latent space of the linear wave equation to four and of the nonlinear Schrödinger equation to two.  All other hyperparameter settings and neural network architectures are listed in detail in \Cref{appendix:A} for each example. 
\subsection{Low-dimensional Systems}\label{sec:lowdim} 
Here, we discuss learning dynamical systems using low-dimensional data by means of three examples.

\subsubsection{Nonlinear Pendulum} Our first example of low-dimensional dynamics is a frictionless pendulum. Pendulums are non-linear oscillators and Hamiltonian systems, making them challenging to learn solely from data  due to their continuous spectra \cite{lusch2018deep}.

 The Hamiltonian for the pendulum can be given in non-dimensional form by
\begin{equation}
\Hamiltonian(q,p) =  1-p^2 +\cos (q).
\end{equation}
 Consequently, we can express the governing equations that define the evolutions of $p$ and $q$ as follows:
\begin{equation}
	\begin{bmatrix}
		\dot{q}(t)\\ \dot{p}(t) 
	\end{bmatrix} = \begin{bmatrix}
		p(t)\\ -\sin(q(t))
	\end{bmatrix}.
\end{equation}

To generate the training dataset, we consider initial conditions for variables $p$ and $q$ within the range of $[-2, 2] \times [-2, 2]$, encompassing the transition from linear to nonlinear dynamics in the system. However, we choose initial conditions from the range with energy $\Hamiltonian(q,p) < 2$ to avoid the pendulum to complete a circle, which prevents unbounded trajectories in phase space. We consider $10$ random initial conditions and take $50$ equidistant data points in the time interval $[0,10]$. We have pictorially shown the training data in \Cref{fig:pend_trainingdata}.

\begin{figure*}[tb]
	\centering
	\begin{subfigure}[t]{0.3\textwidth}
		\includegraphics[width=1\linewidth]{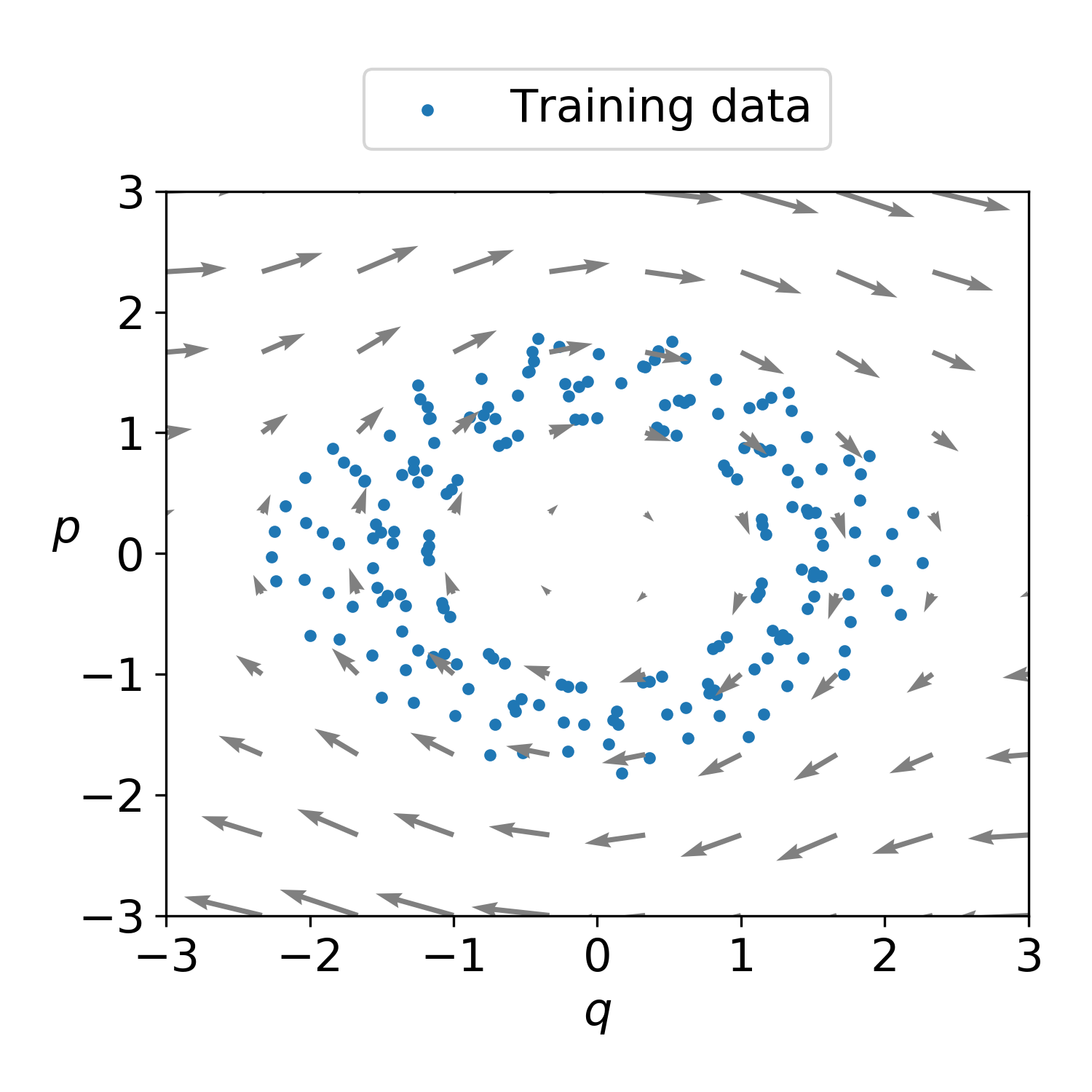}
		\caption{Training data.}
		\label{fig:pend_trainingdata}
	\end{subfigure}
	\begin{subfigure}[t]{0.3\textwidth}
		\includegraphics[width=1\linewidth]{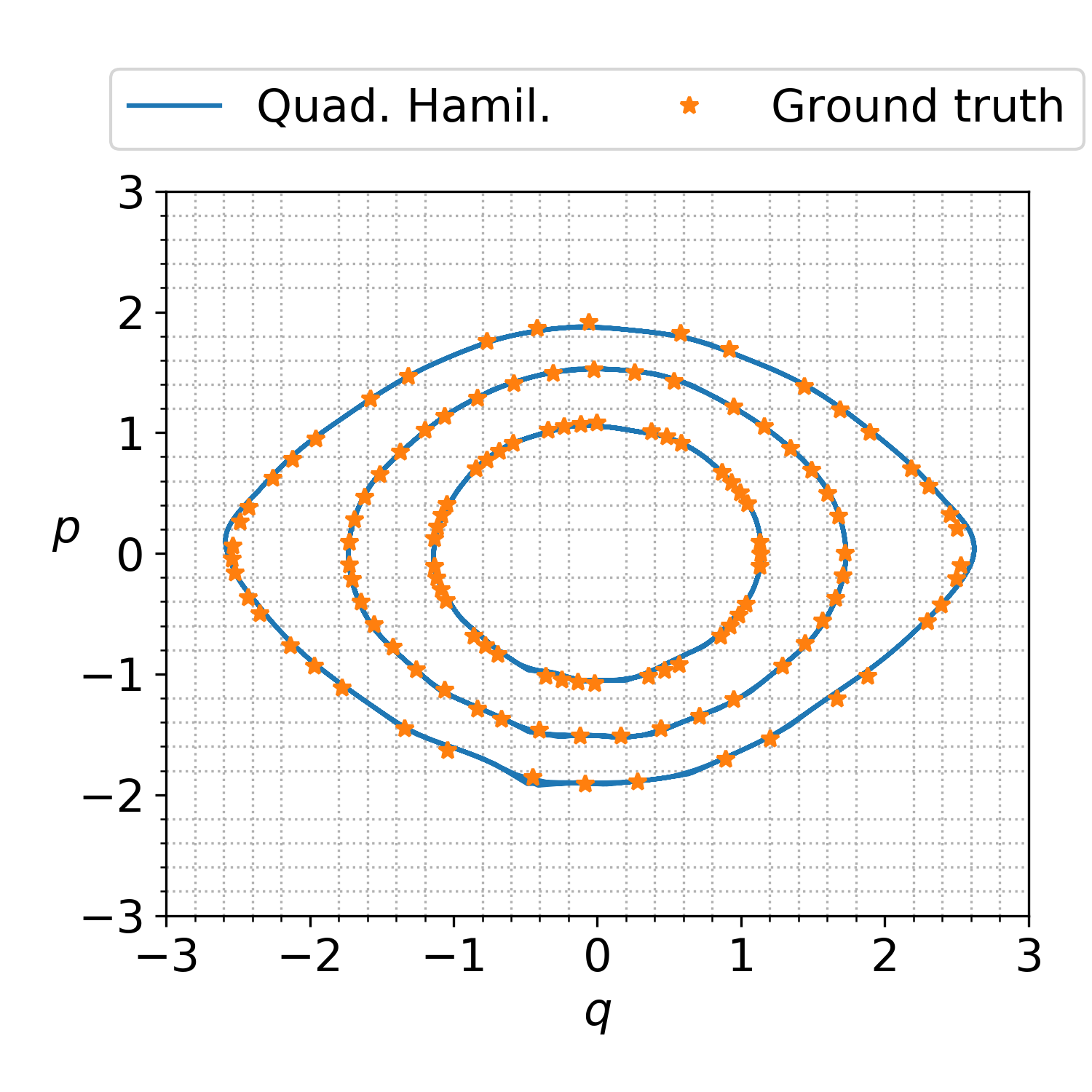}
		\caption{A comparison of the learned model \eqref{quadraticODEHamiltonian} with the ground truth in phase space.}
		\label{fig:pend-phase}
	\end{subfigure}	
	\begin{subfigure}[t]{0.3\textwidth}
	\includegraphics[width=1\linewidth]{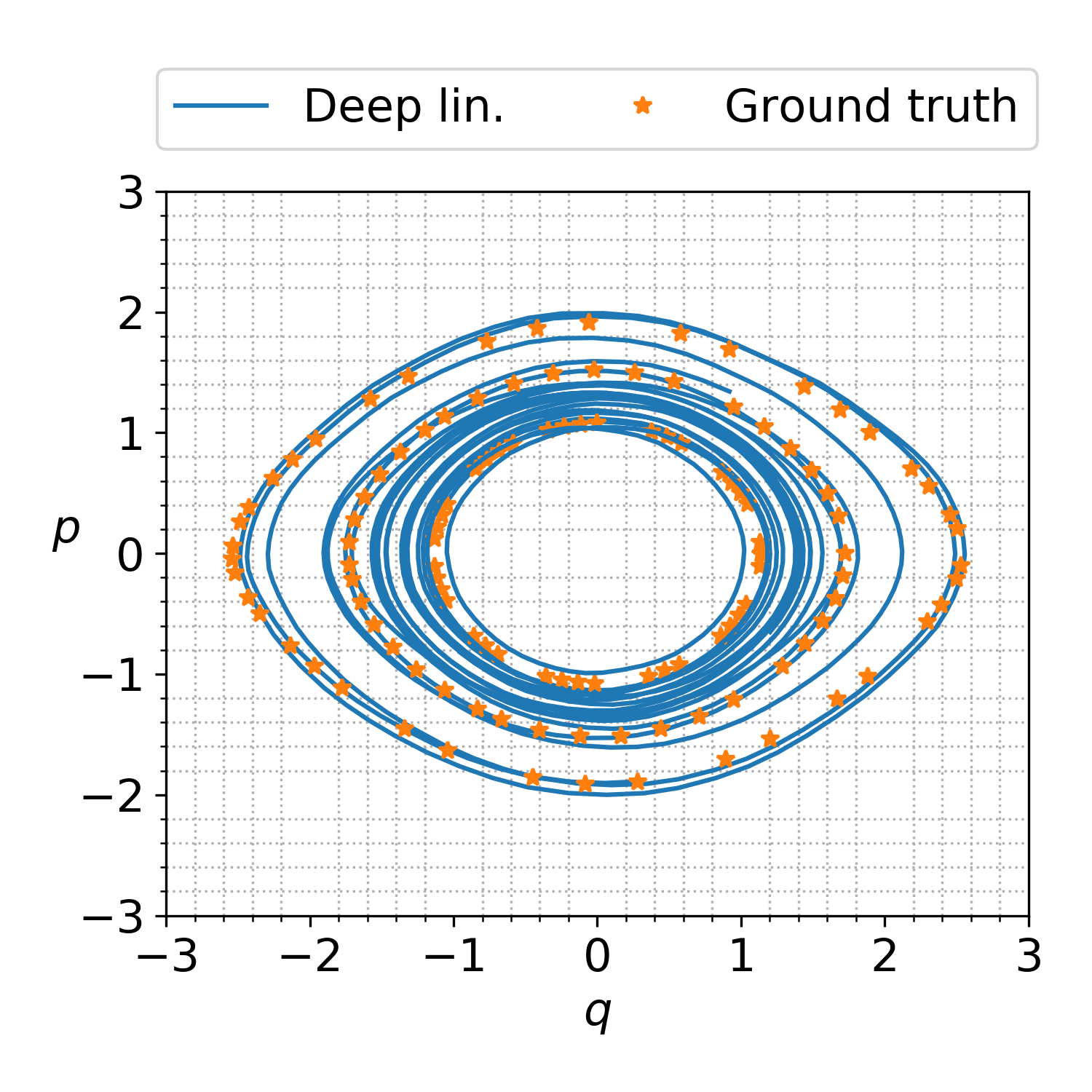}
	\caption{A comparison of the learned deep linear embedding with the ground truth model in phase space.}
	\label{fig:pend-phase-linear}
	\end{subfigure}
	
	\caption{Nonlinear pendulum: Plot (a) shows the training data with the blue dots in phase space together with the ground truth vector field, and Plot (b) shows a comparison of the learned model \eqref{quadraticODEHamiltonian} with the ground truth in phase space with three random initial conditions. Plot (c) shows a comparison of the learned deep linear embedding with the ground truth in phase space with three random initial test conditions.}
\end{figure*}

Next, we learn latent variables $\hat p$ and $\hat q$ with our desired objective, which is that the dynamics of the latent variables can be given by a quadratic system with a cubic Hamiltonian. Moreover, the latent variables are learned by means of an encoder, and the quantities-of-interest, namely $ p$ and $ q$, are identified using a decoder, which maps the latent variables to $p(\hat{q},\hat{p})$ and $q(\hat{q},\hat{p})$. With the training configuration given in \Cref{appendix:A}, we first demonstrate the learned dynamics in the phase space for three random test initial conditions in \Cref{fig:pend-phase} for the pendulum example, where the figure shows that the learned system is stable and orbiting at the same energy level as the ground truth model.

For comparison, we consider an approach, proposed in \cite{lusch2018deep}, which aims at learning a universal linear embedding, inspired by the Koopman theory.  The approach can be viewed as identifying a suitable observable space using deep learning where the dynamics can be governed by a linear system. Furthermore, the work discusses the approach for discrete systems, which we modify to learn continuous linear dynamical systems for the observable subspace or latent variables. In the following, we briefly summerize the modified version of the approach in \cite{lusch2018deep} for continuous systems. We obtain the time derivatives of the latent variables $z(t)$ by the chain rule $\dot z(t)=\Diff \psi_{x(t)} \dot x(t)$ and use it to obtain the dynamics of the linear model as follows:
\begin{equation}
		\begin{aligned}
			\tilde{\mathcal L}_{\dot z\dot x} =\| \Diff \psi_{x(t)} \dot x(t) - \mathcal{K}z(t) \|,
		\end{aligned}
\end{equation}	
where $\psi$ represents the encoder as in \cref{eq:loss1} and $\mathcal{K} \in \mathbb{R}^{n \times n}$ corresponds to the linear operator. The total loss function for the deep linear embedding is given as
\begin{equation}
	\mathcal L = \mathcal L_{\text{encdec}}   +\mathcal L_{\dot z\dot x} ,
\end{equation}
with the same auto-encoder loss $\mathcal L_{\text{encdec}}$ as in \cref{eq:loss1}.

In \Cref{fig:pend-phase-linear} we compare the linear embedding with the ground truth model, where the plot shows that the learned dynamics do not follow the energy level in the phase space because the model does not preserve the Hamiltonian structure. Furthermore, \Cref{tab:pendulum-loss} shows the final loss values for both models, which shows that both models have similar final loss values because both models were trained with the same hyper-parameters and auto-encoders. This then shows that the learned deep linear embedding model is not converging to the desired dynamics, even though the losses are similar.
Furthermore, in \Cref{fig:pend-time}, we compare time-domain simulations of the identified model with the ground truth model for a random initial condition that is different from the training set. \Cref{fig:pend-time} shows that the learned model is not only good at capturing the dynamics in a test case but also stable and accurate for long-time integration, on a time interval larger than the training interval $[0,10]$.
\begin{figure*}[tb]
	\centering
	\begin{subfigure}[b]{0.3\textwidth}
		\includegraphics[width=1\linewidth]{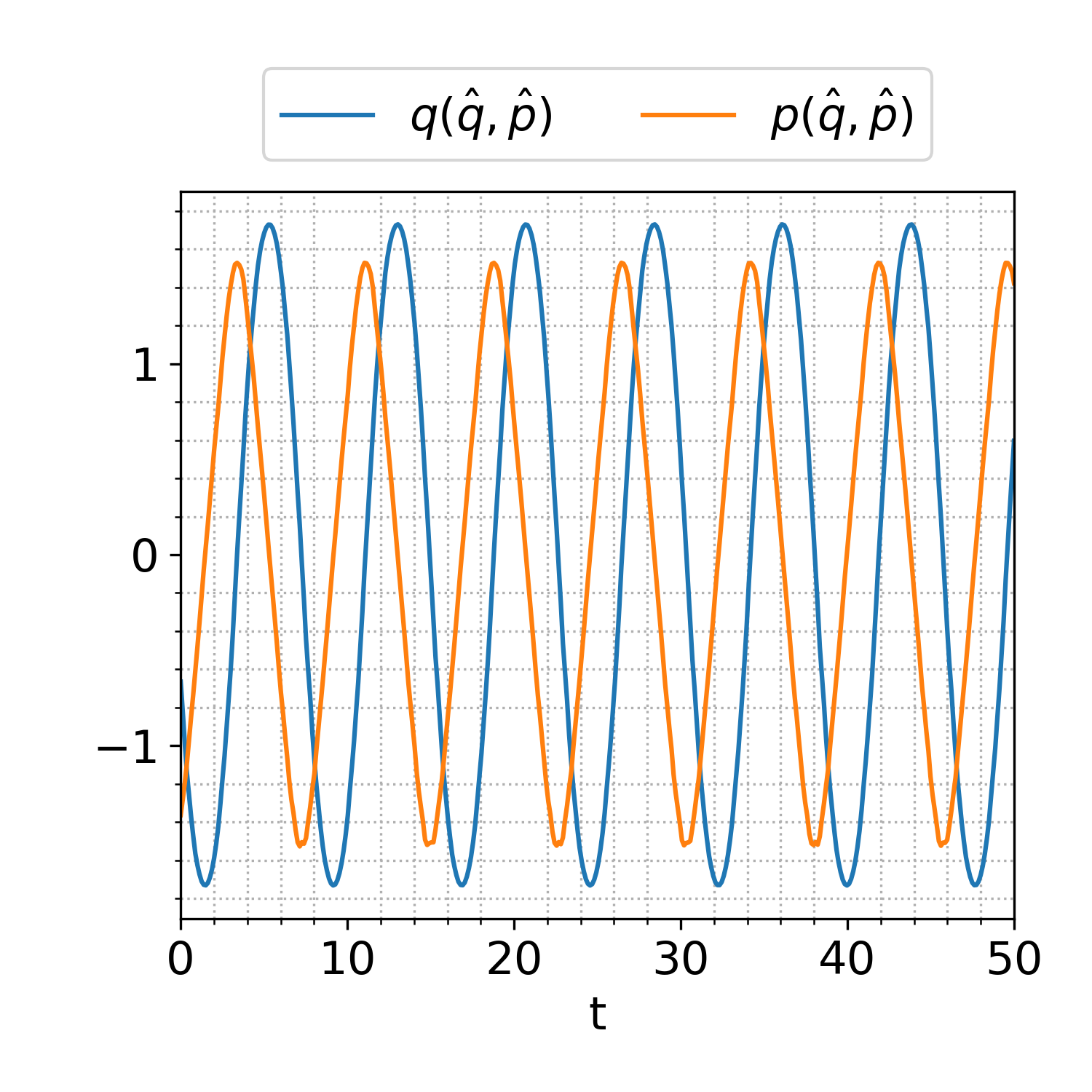}
		\caption{Learned model}
	\end{subfigure}
	\begin{subfigure}[b]{0.3\textwidth}
		\includegraphics[width=1\linewidth]{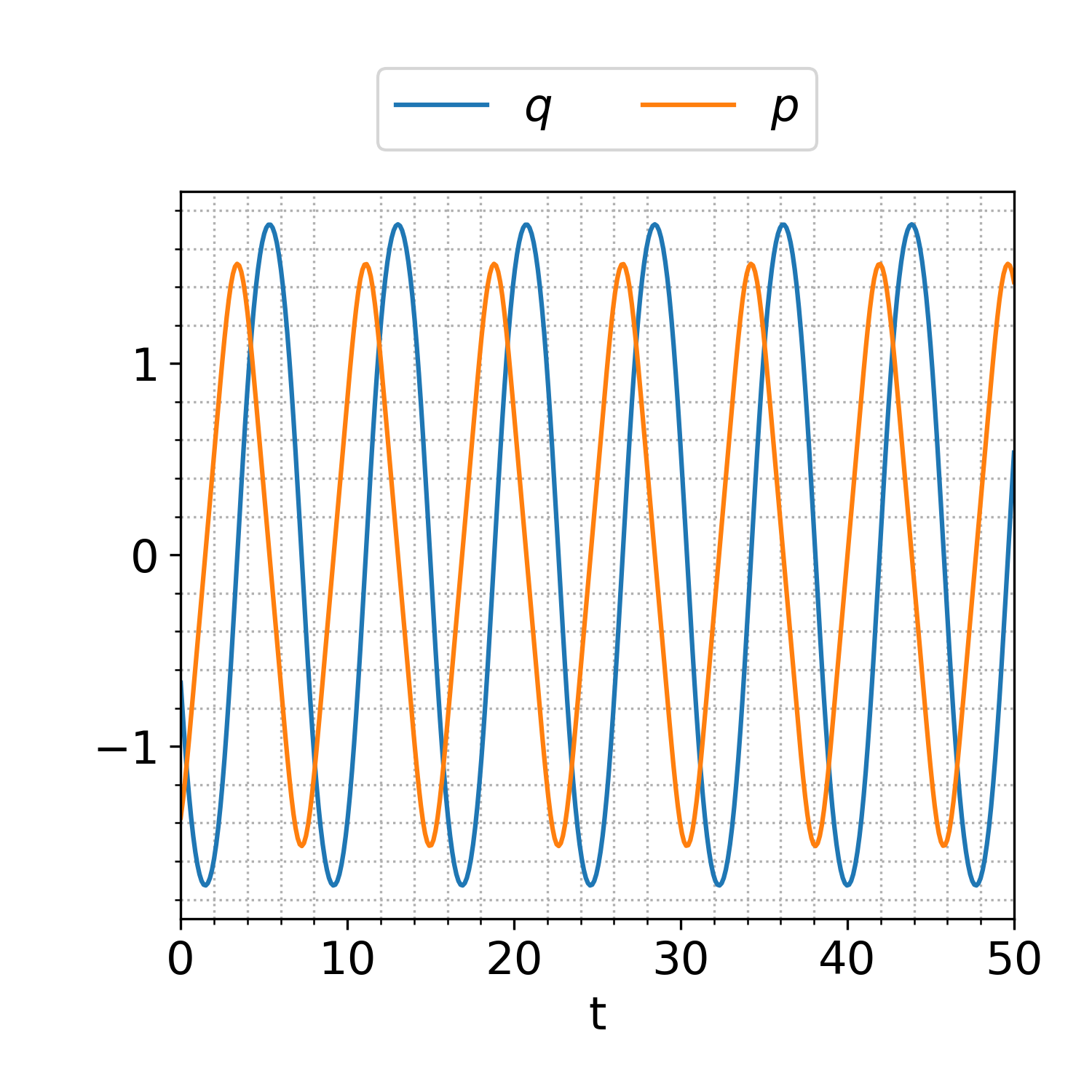}
		\caption{Ground truth}
	\end{subfigure}
	\begin{subfigure}[b]{0.3\textwidth}
		\includegraphics[width=0.91\linewidth]{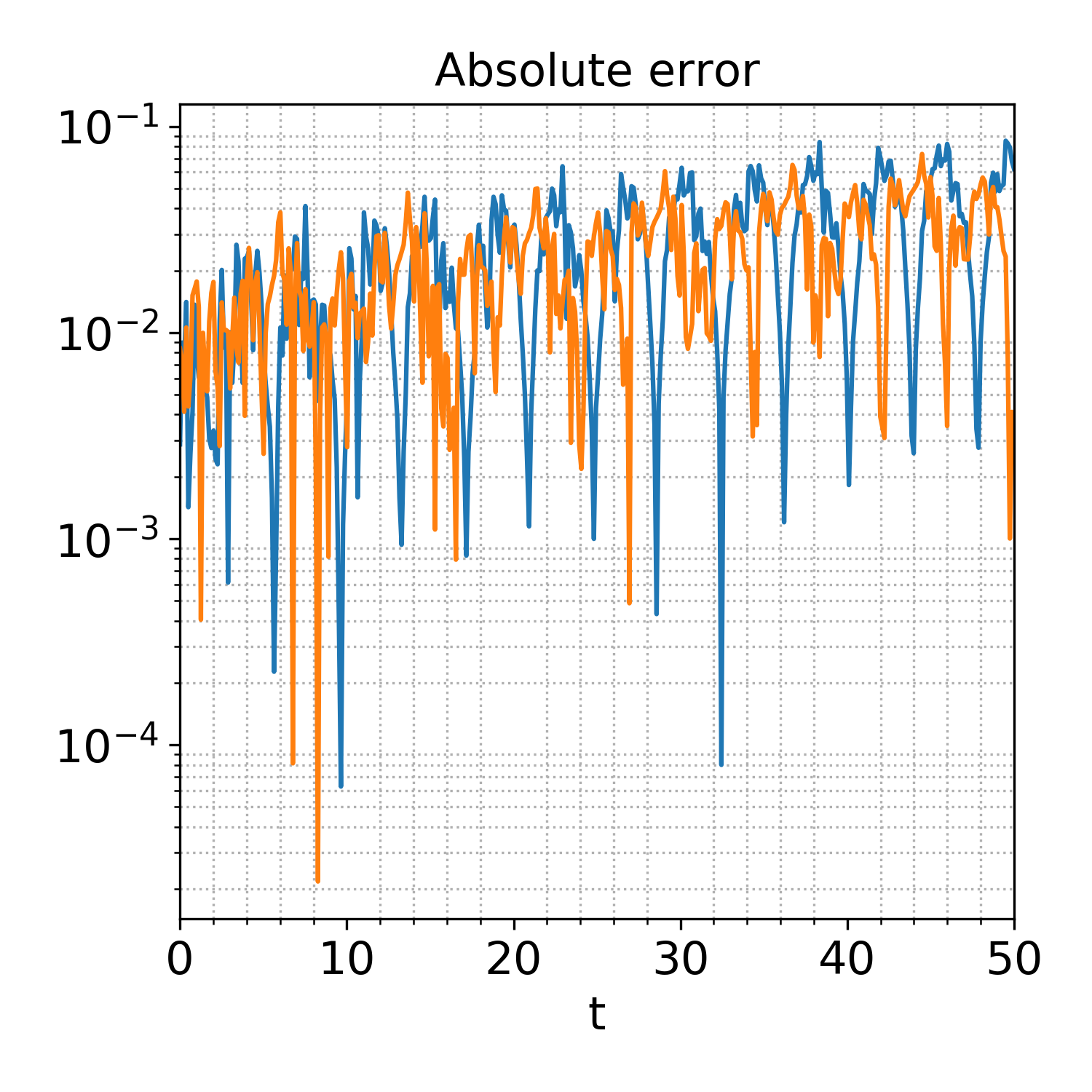}
		\caption{Absolute error}
	\end{subfigure}
	\caption{Nonlinear pendulum: A comparison of the learned model with the ground truth model for a random test condition.}
	\label{fig:pend-time}
\end{figure*}

In \Cref{fig:pend-ham}, we plot the learned latent and ground truth Hamiltonians, demonstrating that all Hamiltonians remain constant over time with minor oscillations. Notably, the learned latent Hamiltonian closely aligns with the ground truth Hamiltonian, even without the need for additional constraints in the optimization process, such as weakly enforcing the initial Hamiltonian value of the ground truth model. 

\begin{table}
\caption{The table contains all the loss values in the final epoch for the pendulum example.}
	\label{tab:pendulum-loss}
	\begin{tabular}{|c|c|c|c|}	
	\hline
   &$\mathcal L_{\text{encdec}}$	& $\tilde{\mathcal L}_{\text{symp}} $ &$\tilde{\mathcal L}_{\dot z\dot x} $ \\ 
   \hline
   Quadratic Hamiltonian System &$3.71\cdot 10^{-3}$	                        &  $ 1.68\cdot 10^{-5} $                                 & $2.10\cdot 10^{-5}$\\
   \hline
   Deep  linear embedding \cite{lusch2018deep}&$1.24\cdot 10^{-3}$&-&$4.58\cdot 10^{-4} $\\
   \hline
   \end{tabular}
\end{table}

\begin{figure*}[tb]
	\centering
	\begin{subfigure}[b]{0.3\textwidth}
		\includegraphics[width=1\linewidth]{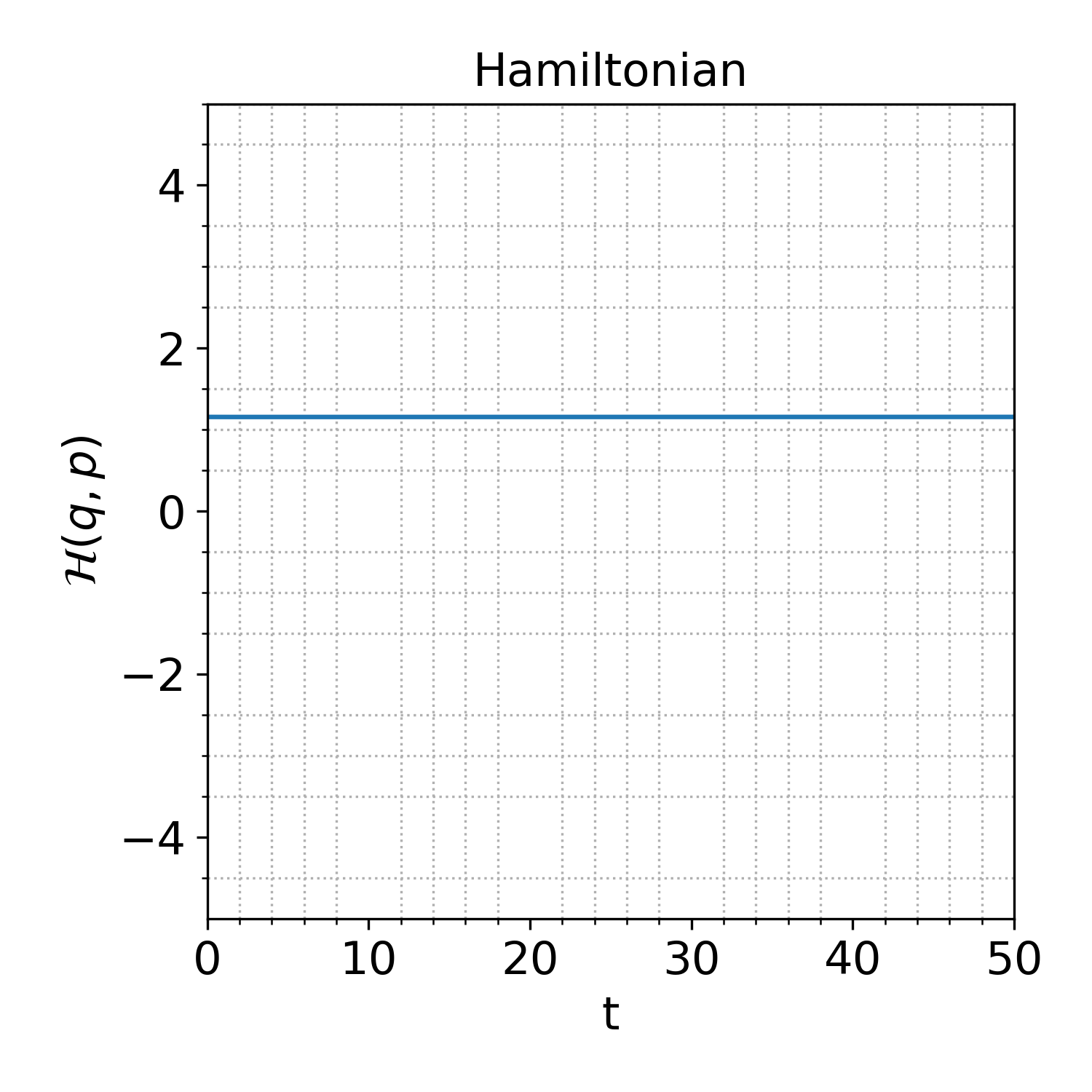}
		\caption{Ground truth}
	\end{subfigure}
	\begin{subfigure}[b]{0.3\textwidth}
		\includegraphics[width=1\linewidth]{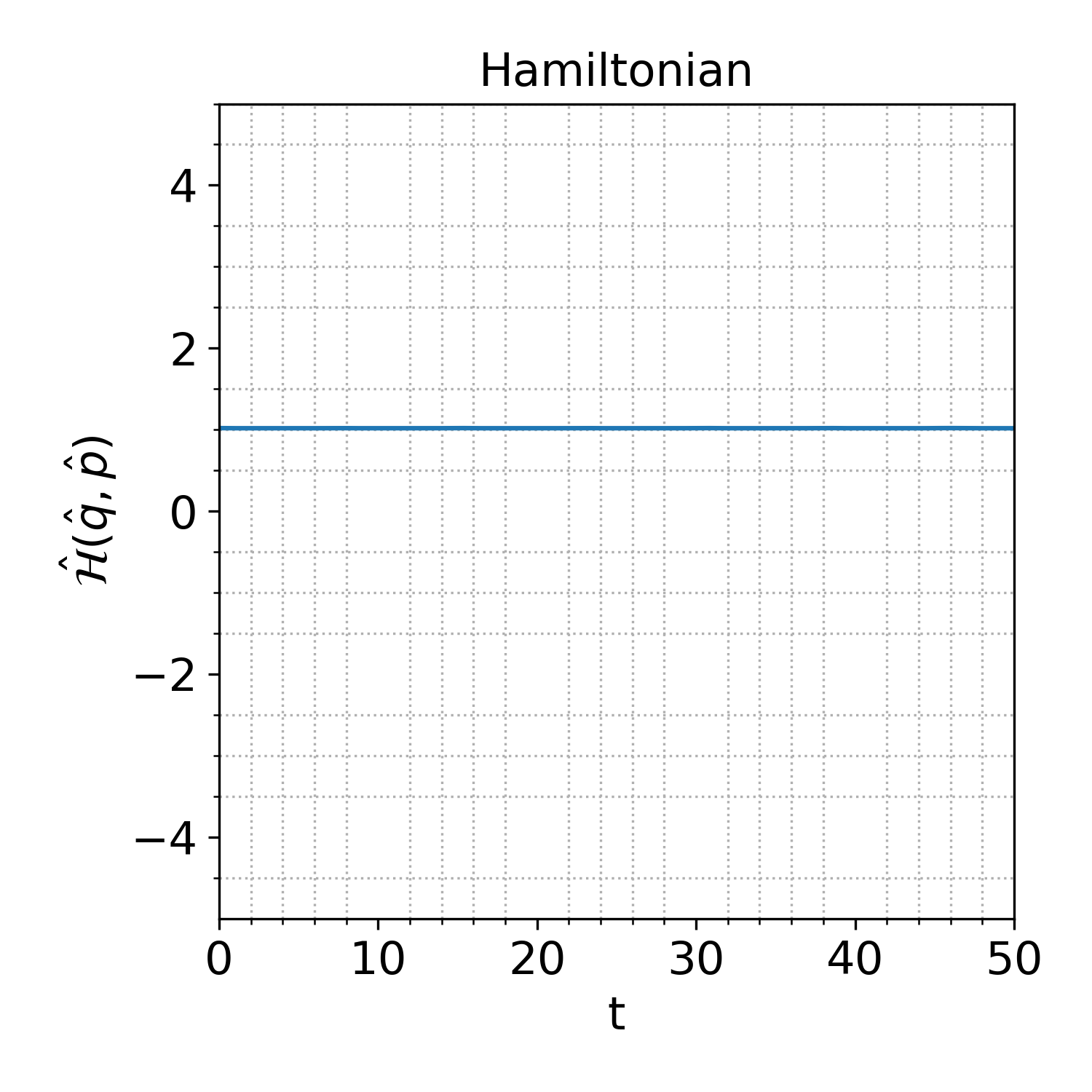}
		\caption{Learned model}
	\end{subfigure}
	\begin{subfigure}[b]{0.3\textwidth}
		\includegraphics[width=1\linewidth]{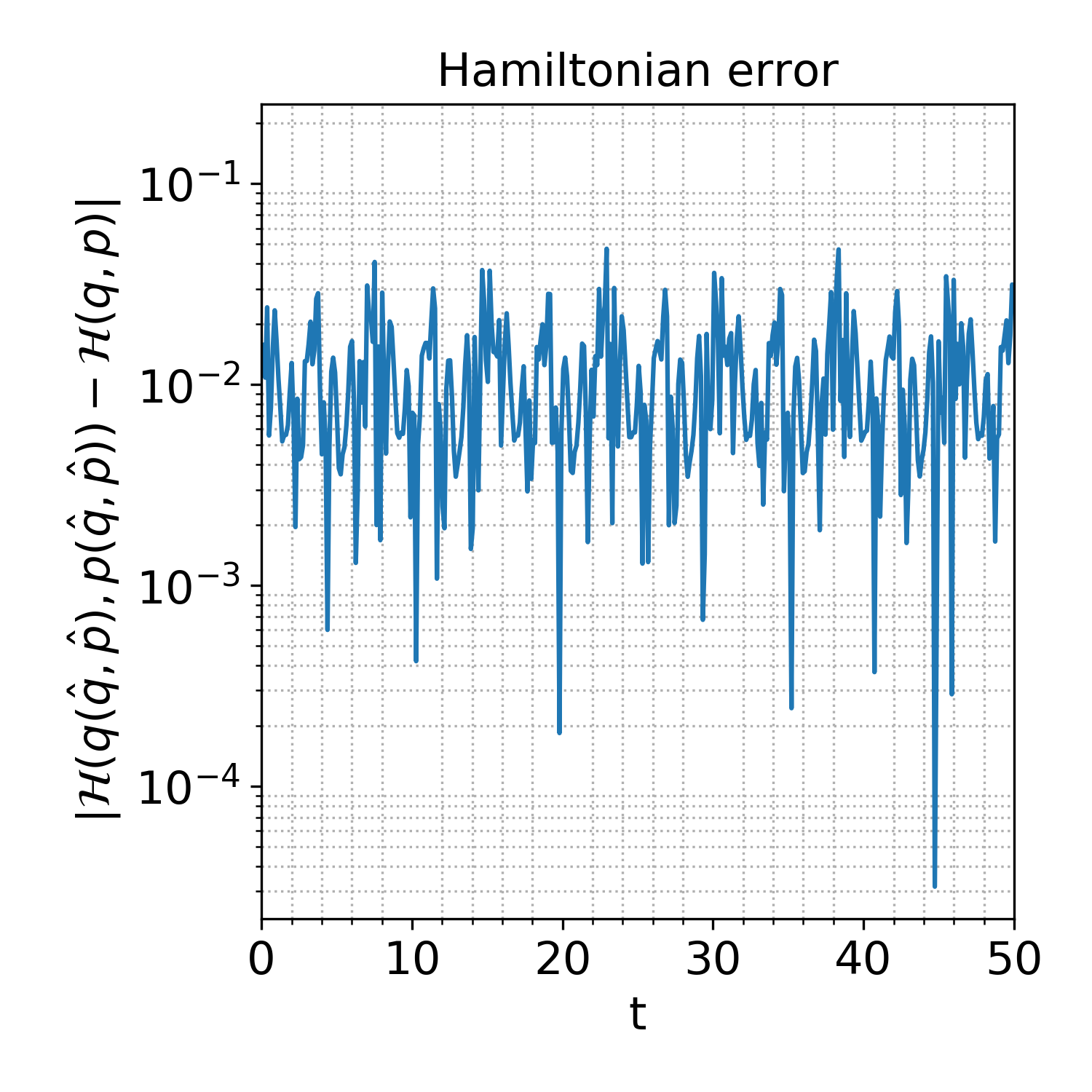}
		\caption{Absolute error}
	\end{subfigure}
	\caption{Nonlinear pendulum: A comparison of the Hamiltonian in canonical coordinates for the ground truth model $\Hamiltonian(q,p)$, the learned Hamiltonian $\hat{\Hamiltonian}(\hat q, \hat p)$ in the latent space, 
		and the difference between the ground truth model and the learned model in the original space $\Hamiltonian(q(\hat q, \hat p), p(\hat q, \hat p))$ along time using a random test initial condition.} 
	\label{fig:pend-ham}
\end{figure*}

\subsubsection{Lotka–Volterra Equations} Our second example of a low-dimensional system is the Lotka-Volterra system \cite{tong2021symplectic}. The Lotka-Volterra model is a well-known mathematical model used to describe predator-prey populations' dynamics. This model has been extensively applied to study the population dynamics of diverse species across various ecosystems. Moreover, it is an example of a system with an underlying Hamiltonian structure with Hamiltonian
\begin{equation*}
\Hamiltonian(q,p) = p-e^{p}+2q-e^q.
\end{equation*}

To learn the dynamics of the Lotka-Volterra equations, we constructed
a training set with $10$ trajectories. These trajectories were simulated up to a time of $T=4$, using a time-step size of $\Delta t=0.2$. For this experiment, we generated trajectories within the energy range $\Hamiltonian(q,p) \in   [-4,4]$.

After learning a suitable embedding with the given set-up in \Cref{appendix:A}, we plot the training data of the Lotka-Volterra equations in phase space in \Cref{fig:lv_trainingdata}. In this example, we focus on trajectories in phase space that do not complete a full orbit of the energy level.
Furthermore, we demonstrate the learned dynamics in the phase space for a random initial value 
in \Cref{fig:lv-phase} for the Lotka-Volterra equations. The figure shows that the learned model is accurate even in terms of predicting the orbit level of random test initial conditions.

\begin{figure*}[tb]
	\centering
	\begin{subfigure}[t]{0.42\textwidth}
		\includegraphics[width=0.8\linewidth]{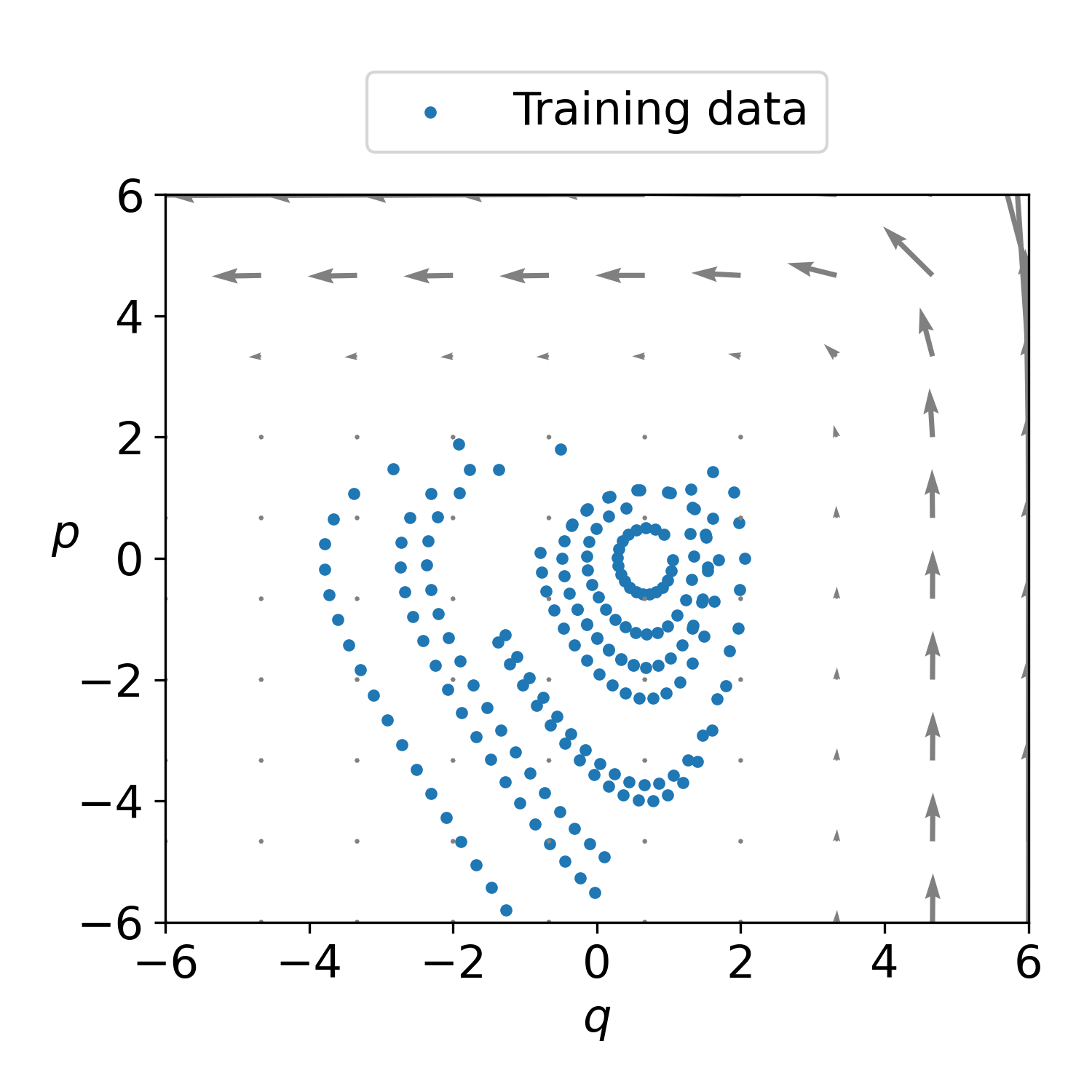}
		\caption{Training data.}
		\label{fig:lv_trainingdata}
	\end{subfigure}
	\begin{subfigure}[t]{0.42\textwidth}
		\includegraphics[width=0.8\linewidth]{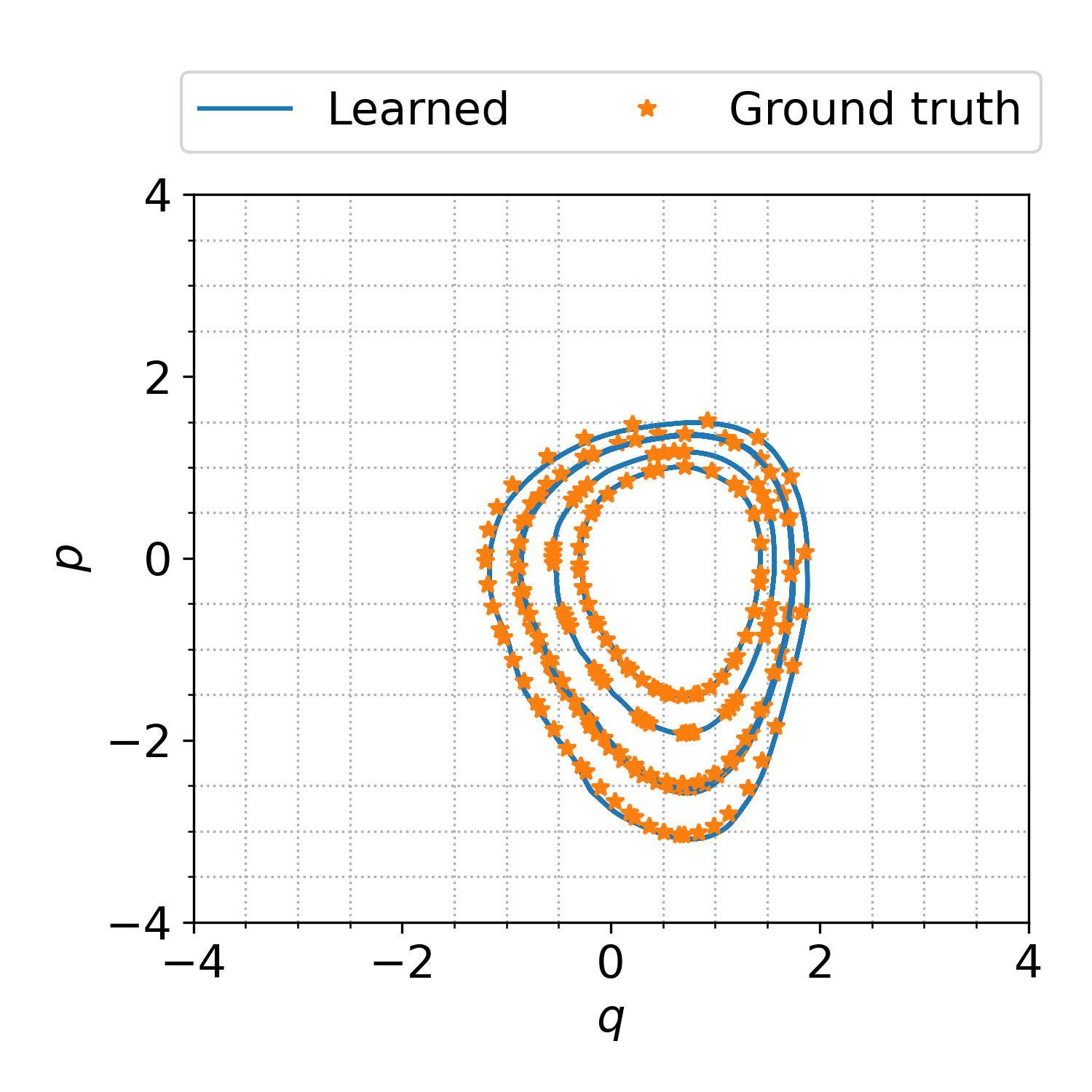}
		\caption{A comparison of the learned model with the ground truth in phase space.}
		\label{fig:lv-phase}
	\end{subfigure}
	\caption{Lotka–Volterra: plot (a) shows training data in phase space, and plot (b) shows a comparison of the learned model with the ground truth in phase space with five random initial test conditions.}
\end{figure*}

In \Cref{fig:lv-time}, we compare time-domain simulations between the learned and ground truth models for the Lotka-Volterra equations, along with the corresponding absolute error. The simulations were conducted using a random initial condition, distinct from the initial trajectories used in the training set. The results depicted in Figure \ref{fig:lv-time} demonstrate a high level of agreement between the dynamics of the Lotka-Volterra equations and the ground truth model, even after the final training time $T = 4$. The absolute error reaches values of order $10^{-1}$, which seems to mainly stem from a phase shift of the learned model. In \Cref{tab:lv-loss}, we show the values of the loss functions in the final epoch, where the values are similar to the pendulum example.  

\begin{table}
	\caption{The table contains all the loss values in the final epoch for the Lotka-Volterra example.}
		\label{tab:lv-loss}
		\begin{tabular}{|c|c|c|}	
			\hline
			$\mathcal L_{\text{encdec}}$	& $\tilde{\mathcal L}_{\text{symp}} $ &$\tilde{\mathcal L}_{\dot z\dot x} $ \\ 
			\hline
			$3.73\cdot 10^{-3}$	                        &  $ 2.00\cdot 10^{-4} $                                 & $8.35\cdot 10^{-4}$\\
			\hline
		\end{tabular}
\end{table}

\begin{figure*}[tb]
	\centering
	\begin{subfigure}[b]{0.3\textwidth}
		\includegraphics[width=1\linewidth]{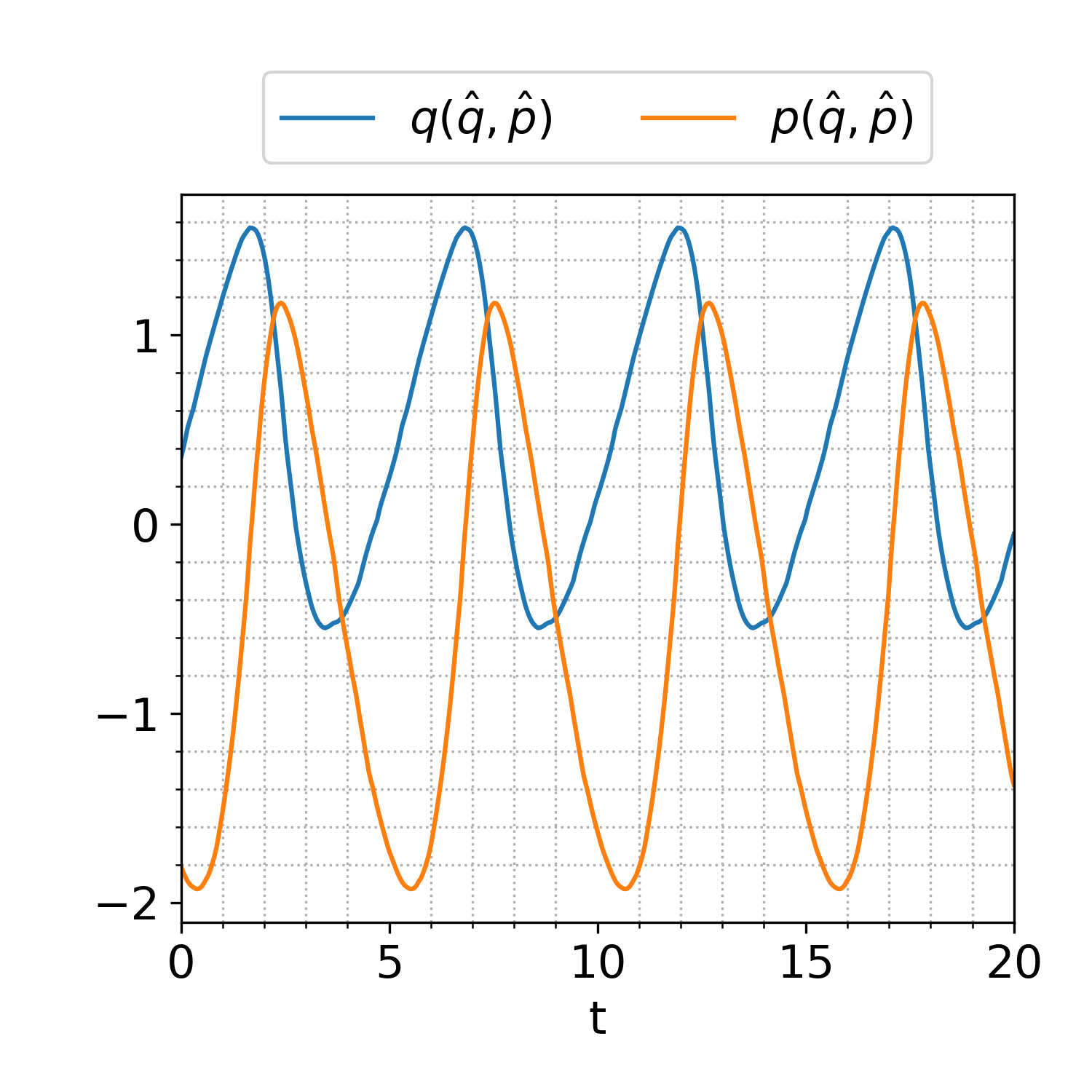}
		\caption{Learned model}
	\end{subfigure}
	\begin{subfigure}[b]{0.3\textwidth}
		\includegraphics[width=1\linewidth]{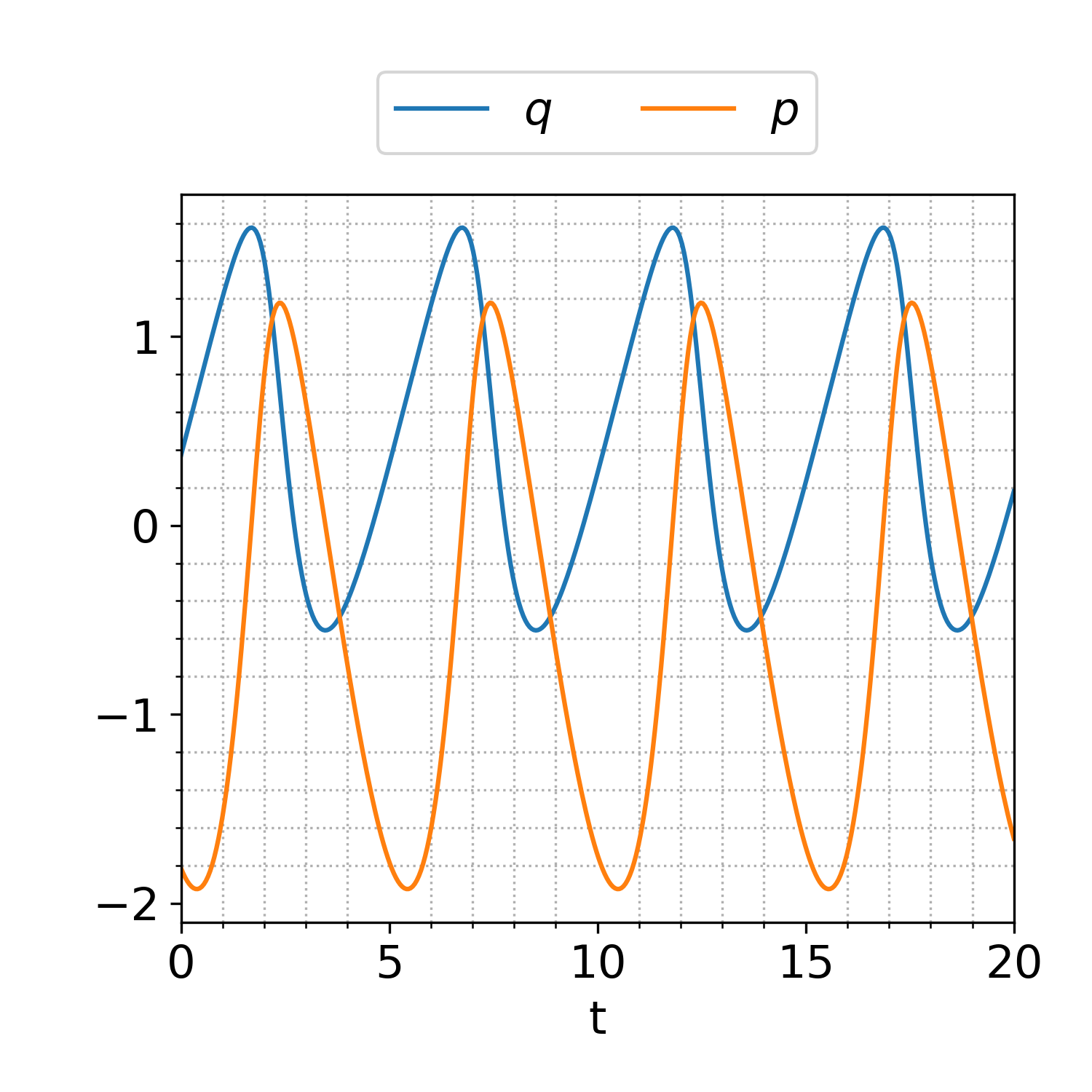}
		\caption{Ground truth}
	\end{subfigure}
	\begin{subfigure}[b]{0.3\textwidth}
		\includegraphics[width=0.91\linewidth]{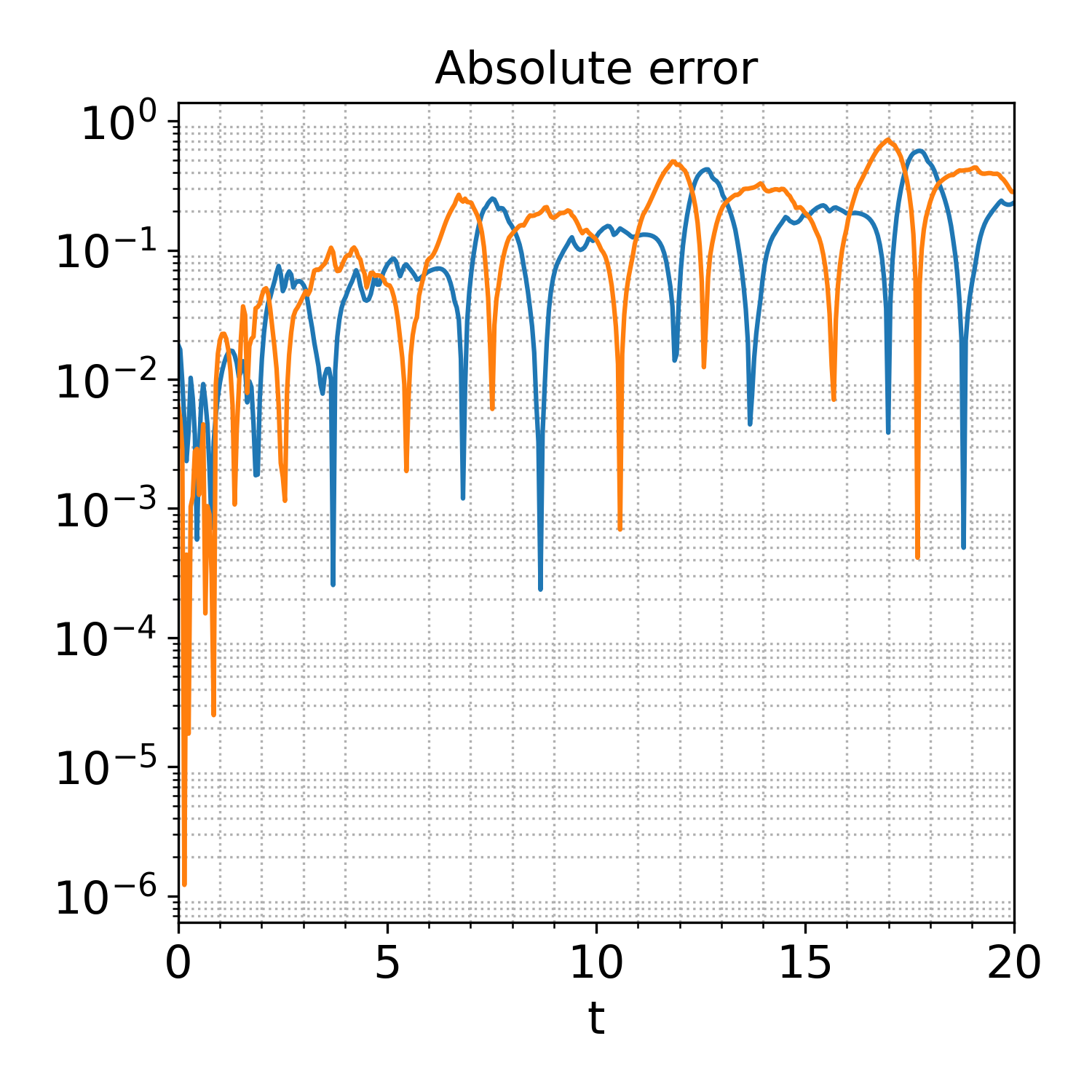}
		\caption{Absolute error}
	\end{subfigure}
	\caption{Lotka–Volterra: A comparison of time-domain simulation obtained using the learned model with the ground truth model 
		for the Lotka-Volterra equations and using a random test initial condition. 
	}
	\label{fig:lv-time}
\end{figure*}
In \Cref{fig:lv-ham}, we present the learned and ground truth Hamiltonians for the Lotka-Volterra equation. Evidently, all Hamiltonians remain constant over time with minor fluctuations. We observe that these fluctuations primarily arise from errors in the autoencoder component, as the symplecticity condition is weakly applied to the encoder part but not the decoder part. Additionally, the fluctuations in the Hamiltonian error could be attributed to the training data, which is constructed from trajectories with a shorter time span compared to the pendulum example.
The constant offset of the Hamiltonians corresponds to a different choice of energy null-level in the original space and the latent space, respectively. Since the Hamiltonian is a relative quantity, the overall performance of the learned model is linked to the error plot between the ground truth Hamiltonian and the learned Hamiltonian in the original space, where the \Cref{fig:lv-ham} shows that they coincide.

\begin{figure*}[tb]
	\centering
	\begin{subfigure}[b]{0.3\textwidth}
		\includegraphics[width=1\linewidth]{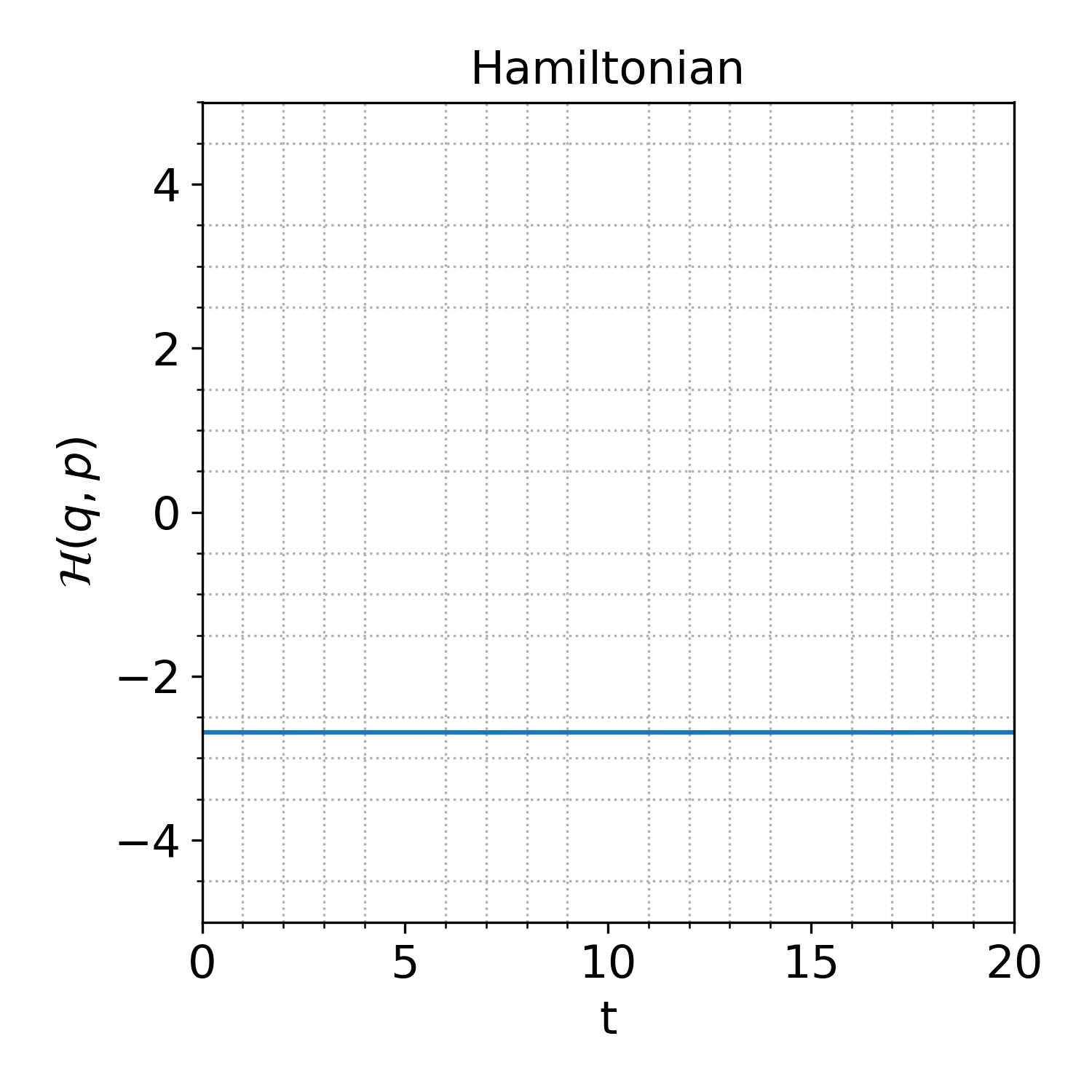}
		\caption{Ground truth}
	\end{subfigure}
	\begin{subfigure}[b]{0.3\textwidth}
		\includegraphics[width=1\linewidth]{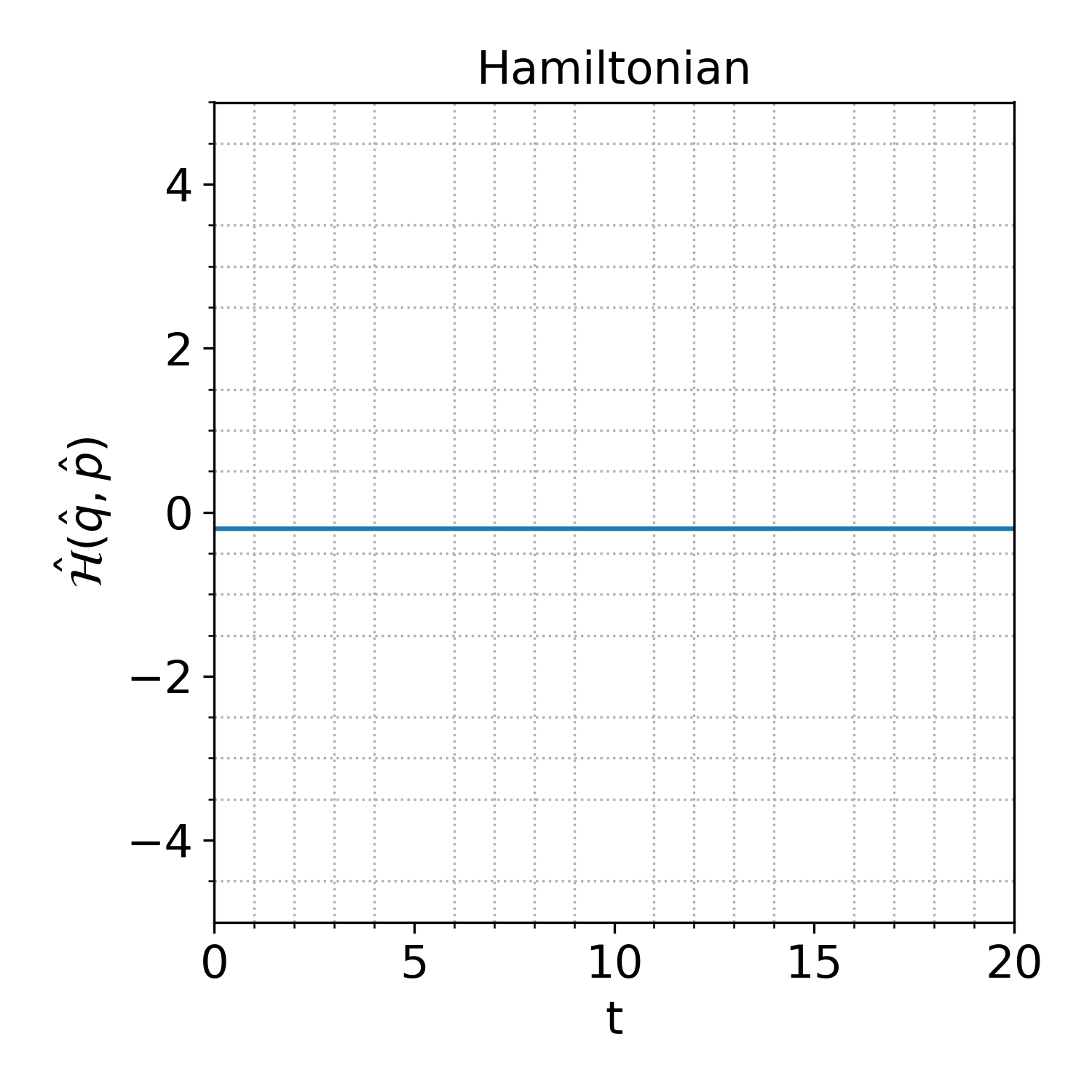}
		\caption{Learned model}
	\end{subfigure}
	\begin{subfigure}[b]{0.3\textwidth}
		\includegraphics[width=1\linewidth]{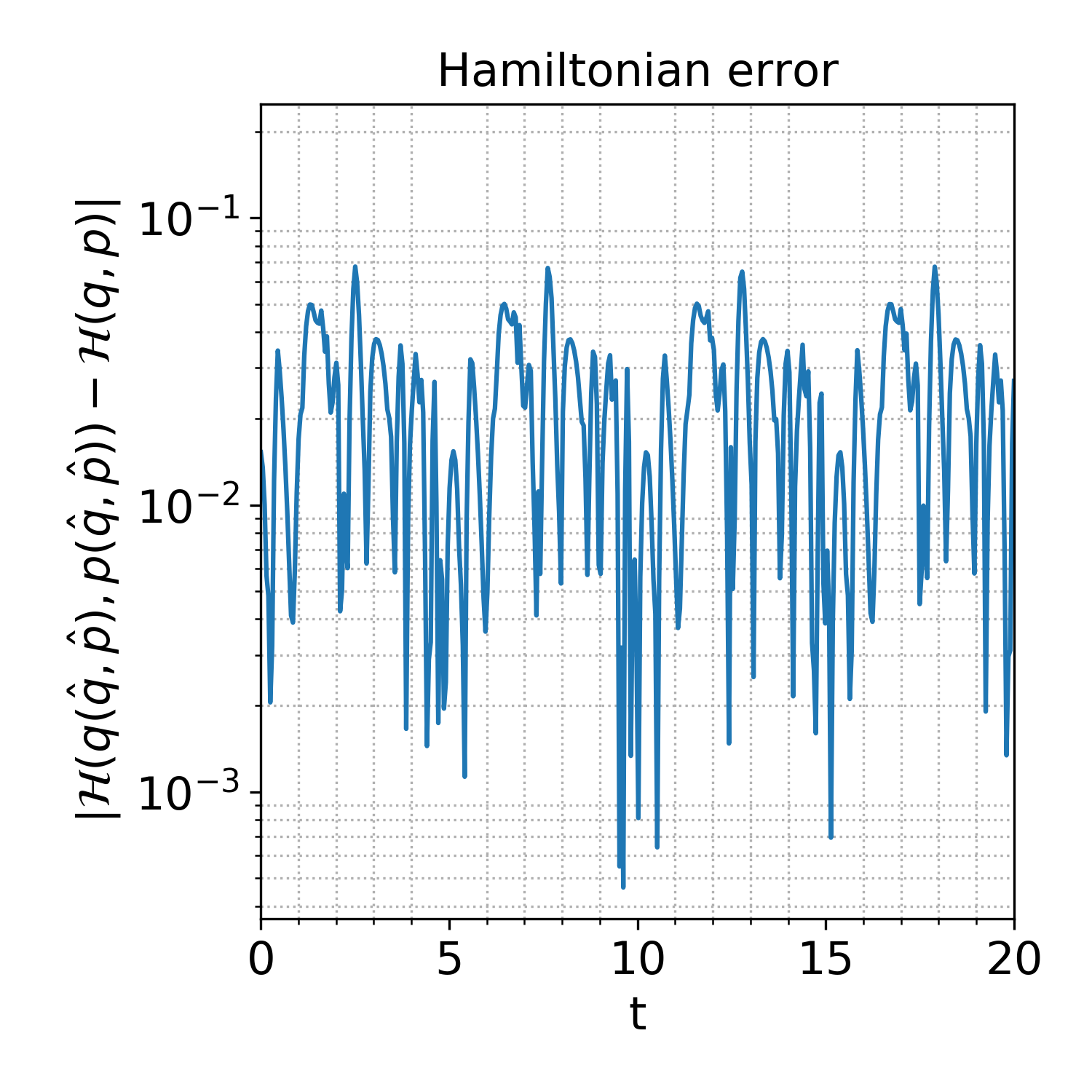}
		\caption{Absolute error}
	\end{subfigure}
	\caption{Lotka–Volterra: A comparison of the Hamiltonian in canonical coordinates for the ground truth model $\Hamiltonian(q,p)$, the learned Hamiltonian $\hat{\Hamiltonian}(\hat q, \hat p)$ in the latent space, 
	and the difference between the ground truth model and the learned model in the original space $\Hamiltonian(q(\hat q, \hat p), p(\hat q, \hat p))$ along time using a random test initial condition.}
	\label{fig:lv-ham}
	
\end{figure*}

\subsubsection{Nonlinear Oscillator} Another low-dimensional example  is a nonlinear (an-harmonic) oscillator with Hamiltonian
\begin{equation}
\Hamiltonian(q,p) = \frac{p^2}{2}+\frac{q^2}{2}+\frac{q^4}{4},
\end{equation}
where the natural frequency and the mass of the oscillator are considered to be unity.

To learn the dynamics from data, we initially generated 
$20$ random trajectories which are simulated up to final time $T=4$ with step-size $\Delta t=0.14$. The generated trajectories are in the energy range of 
$\Hamiltonian(q,p) \in   [-1, 1]$ for this task. In \Cref{fig:no_trainingdata}, we plot the training data of the nonlinear oscillator in phase space.

Having learned the desired embeddings, we next present a comparison of the learned model over three random initial points, which are different from the initial conditions used for training in \Cref{fig:oscil-phase}, 
where the model captures the dynamics of the learned model with good accuracy.
\begin{figure*}[tb]
	\centering
	\begin{subfigure}[t]{0.3\textwidth}
		\includegraphics[width=1\linewidth]{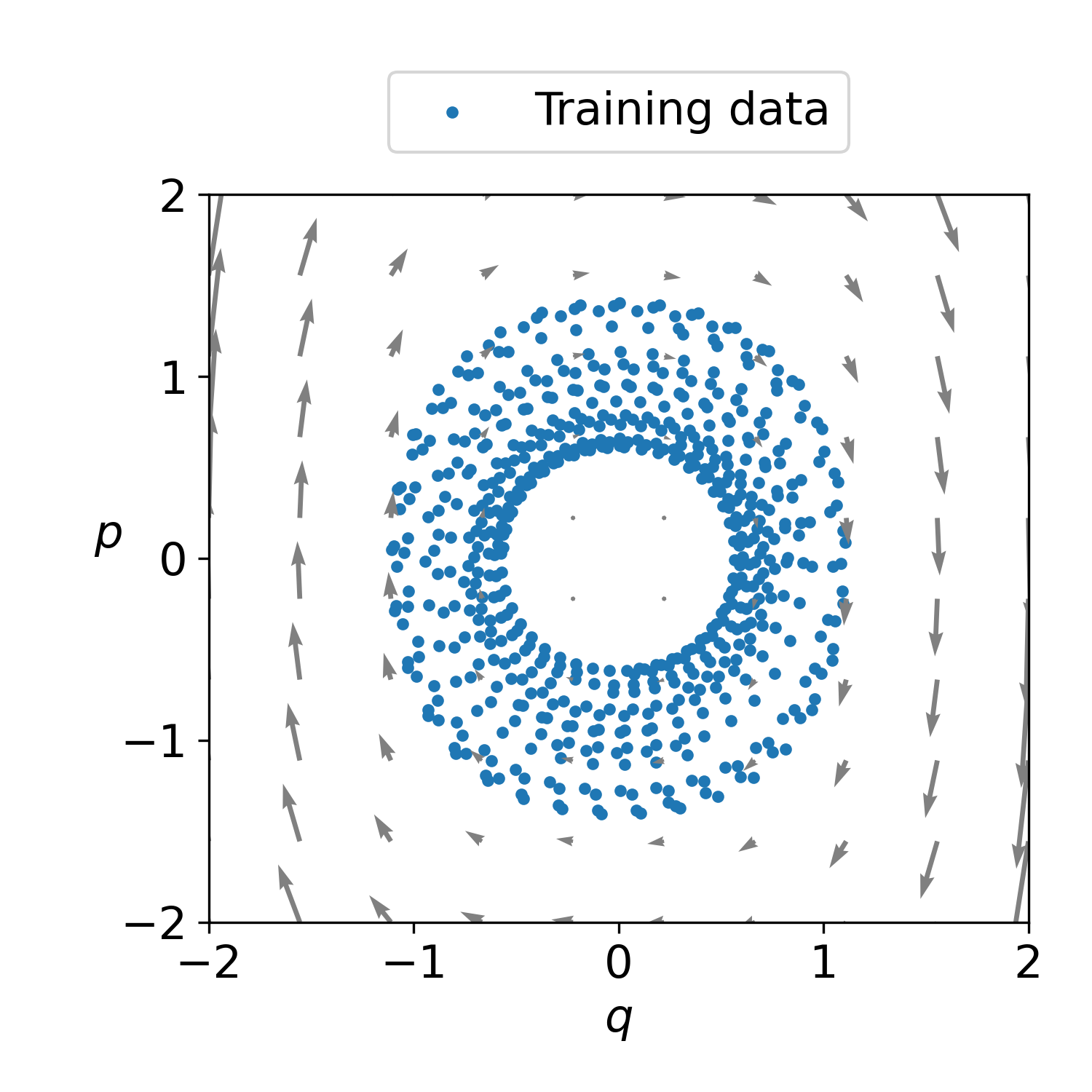}
		\caption{Training data.}
		\label{fig:no_trainingdata}
	\end{subfigure}
	\begin{subfigure}[t]{0.3\textwidth}
		\includegraphics[width=1\linewidth]{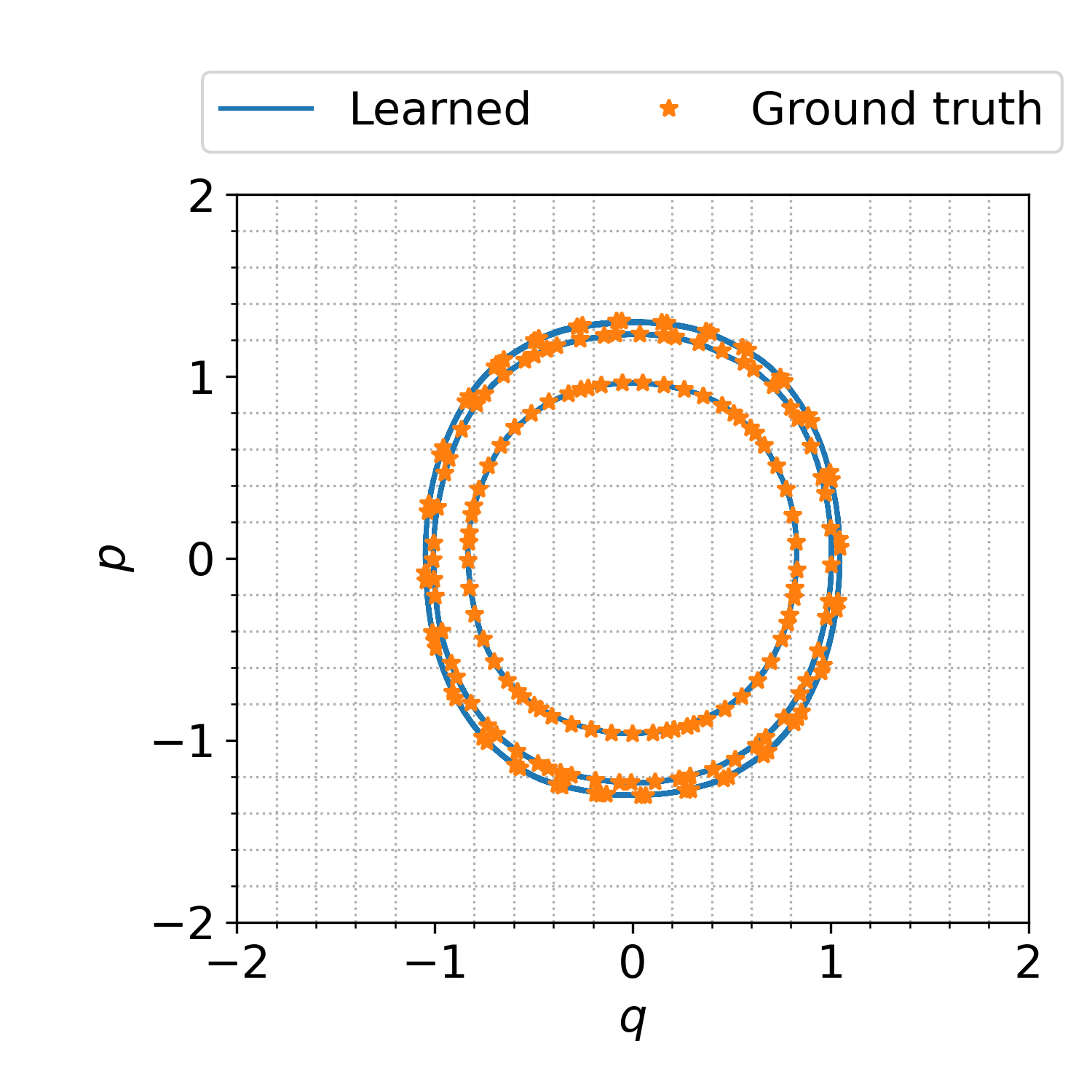}
		\caption{A comparison of the learned model with the ground truth in phase space.}
		\label{fig:oscil-phase}
	\end{subfigure}
	\begin{subfigure}[t]{0.3\textwidth}
		\includegraphics[width=1\linewidth]{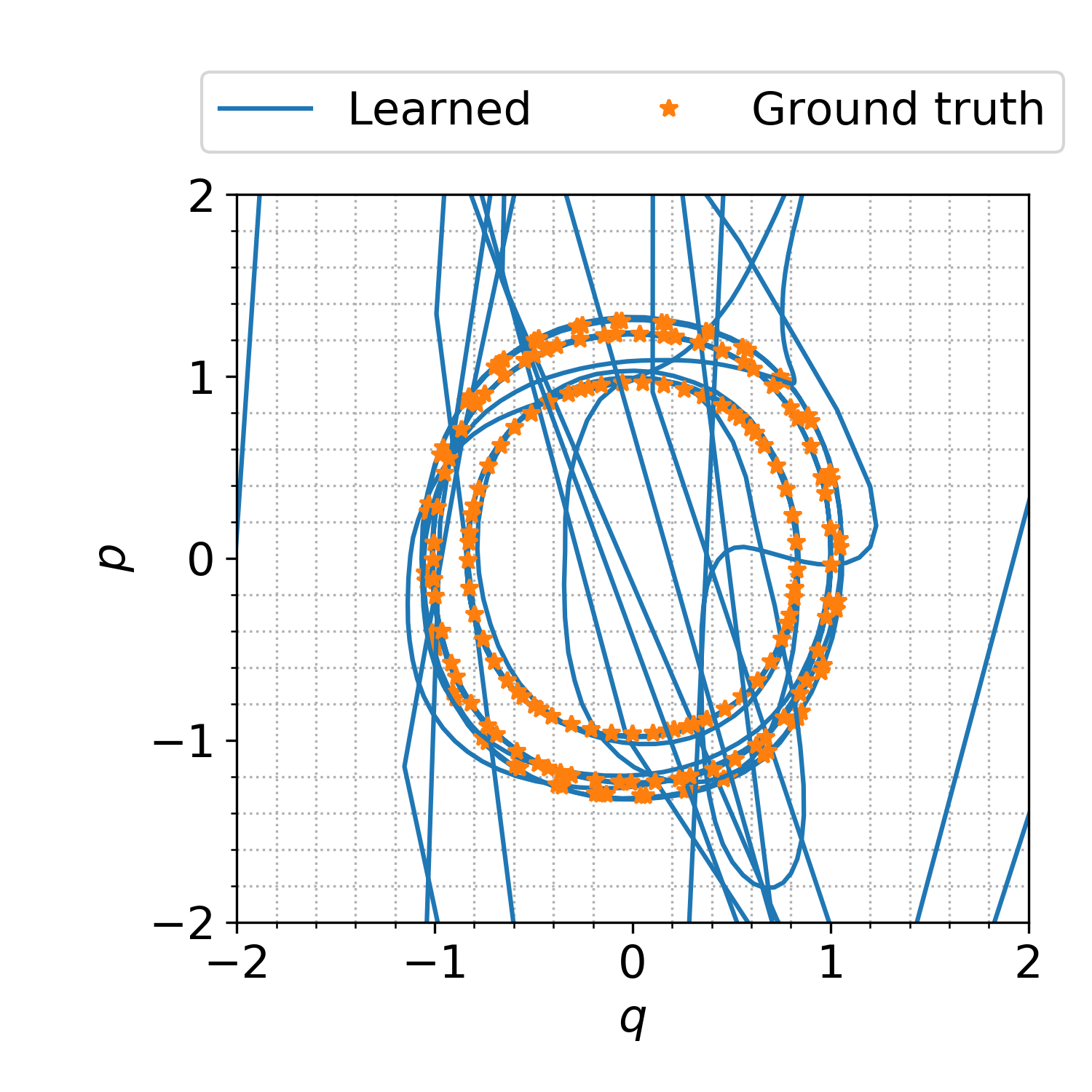}
		\caption{A comparison of the learned model with the ground truth in phase space, using a non-symplectic auto-encoder.}
		\label{fig:oscil-phase_nosymp}
	\end{subfigure}
	\caption{Nonlinear oscillator: Plot (a) shows training data in phase space, and Plot (b) shows a comparison of the learned model with the ground truth in phase space with three random initial test conditions. Plot (c) shows a comparison of the learned model with the ground truth in phase space with the same three random initial test conditions, where the symplecticity condition of the auto-encoder is not enforced.}
\end{figure*}
In \Cref{fig:oscil-time}, we demonstrate the temporal evolution of the learned model, 
the ground truth model for a nonlinear oscillator, 
and the corresponding absolute error in the time domain for a randomly chosen initial condition. The figure shows that the dynamics are well captured
over a long time horizon, exceeding the final training time $T=4$.

\begin{table}
	\caption{The table contains all the loss values in the final epoch for the nonlinear (an-harmonic) oscillator example. The first row shows the results for weakly symplectic auto-encoders, and the second row shows the results when the weak symplecticity condition is not enforced on the auto-encoder. }
		\label{tab:no-loss}
	\begin{tabular}{|c|c|c|c|}	
			\hline
			&$\mathcal L_{\text{encdec}}$	& $\tilde{\mathcal L}_{\text{symp}} $ &$\tilde{\mathcal L}_{\dot z\dot x} $ \\ 
			\hline
			symplectic &$5.67\cdot 10^{-4}$	     &  $ 4.87\cdot 10^{-5} $     & $4.72\cdot 10^{-5}$\\
			\hline
			non-symplectic &$1.08\cdot 10^{-3}$	     &  $ 6.33\cdot 10^{-1} $       & $6.69\cdot 10^{-5}$\\
			\hline
	\end{tabular}
\end{table}

\begin{figure*}[tb]
	\centering
	\begin{subfigure}[b]{0.3\textwidth}
		\includegraphics[width=1\linewidth]{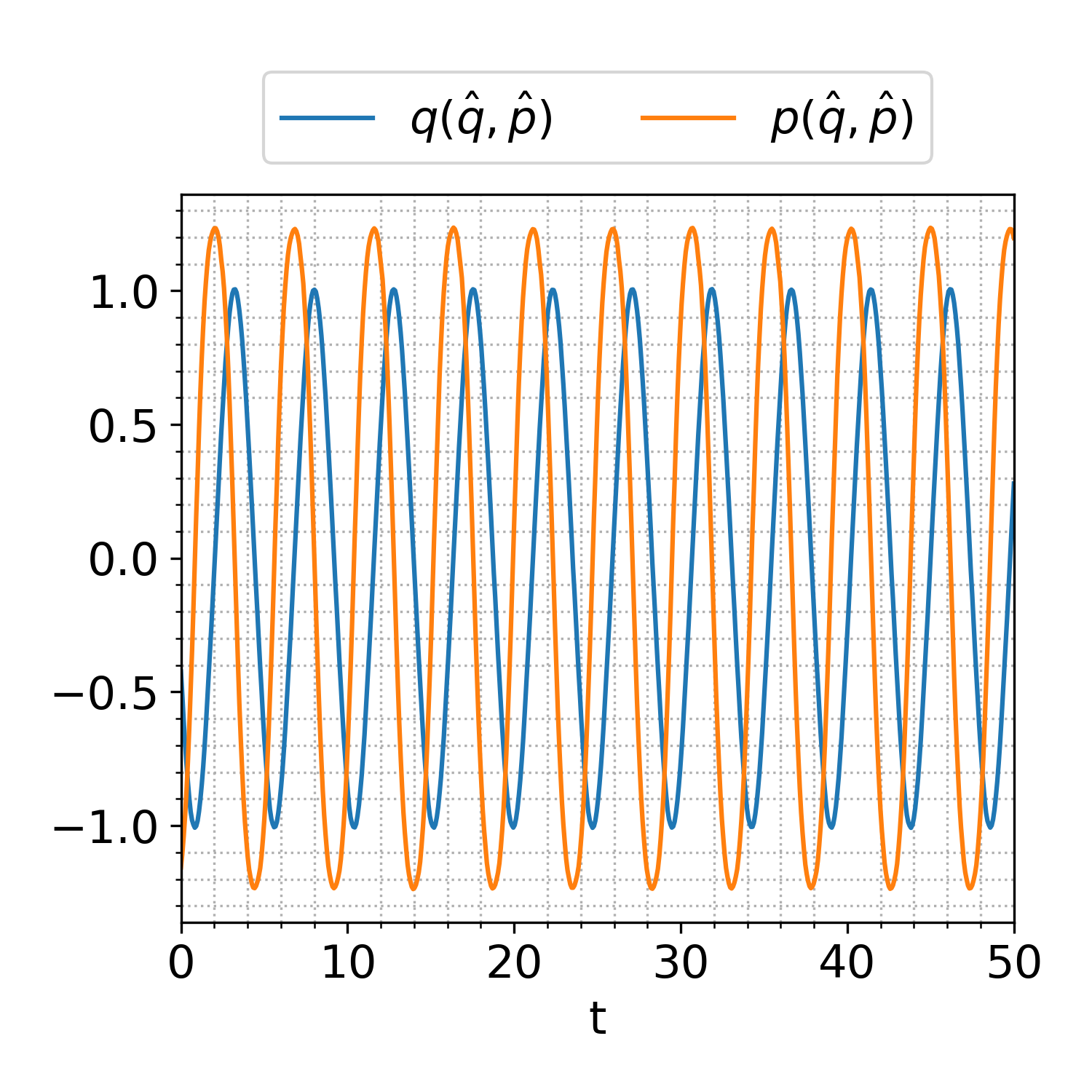}
		\caption{Learned model}
	\end{subfigure}
	\begin{subfigure}[b]{0.3\textwidth}
		\includegraphics[width=1\linewidth]{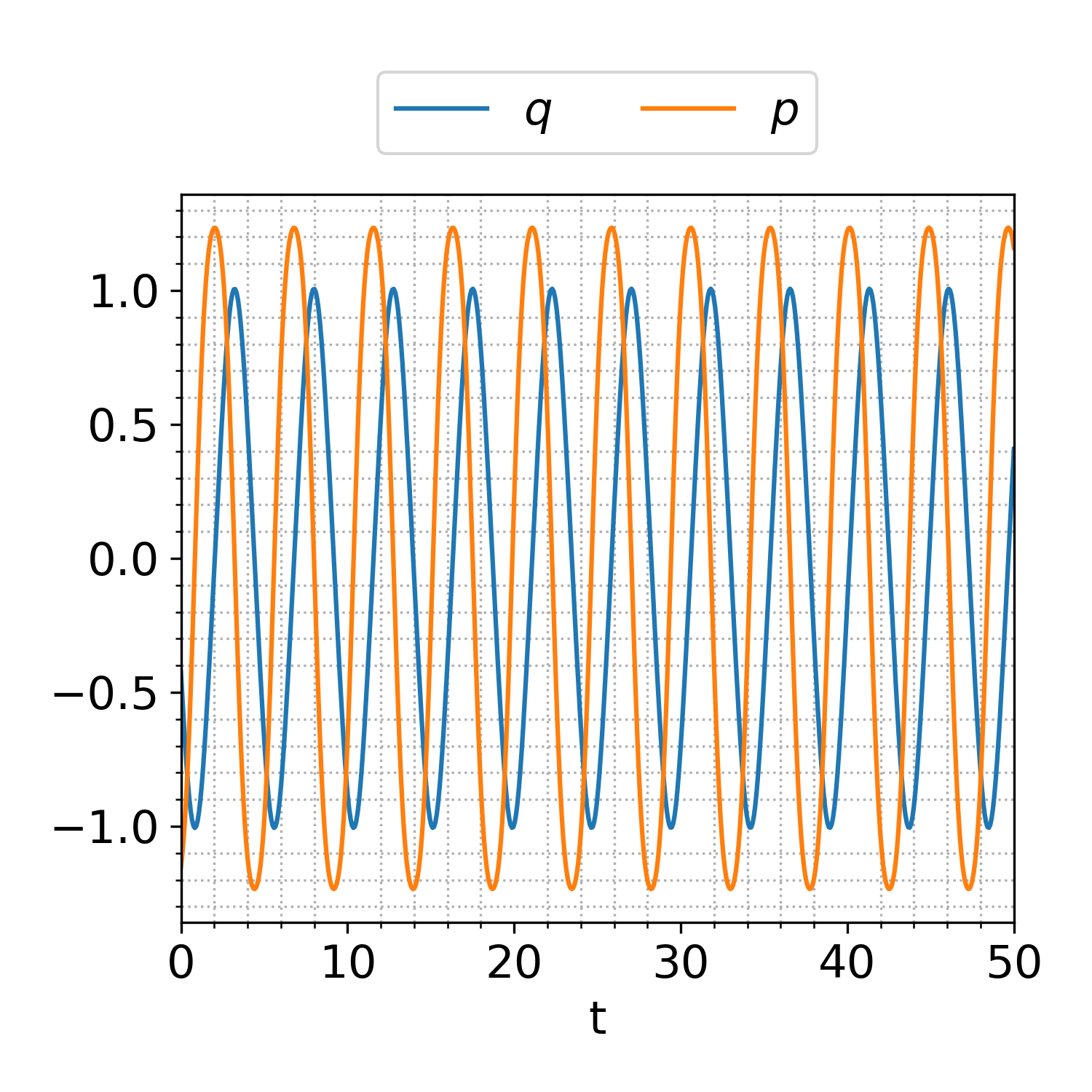}
		\caption{Ground truth}
	\end{subfigure}
	\begin{subfigure}[b]{0.3\textwidth}
		\includegraphics[width=0.91\linewidth]{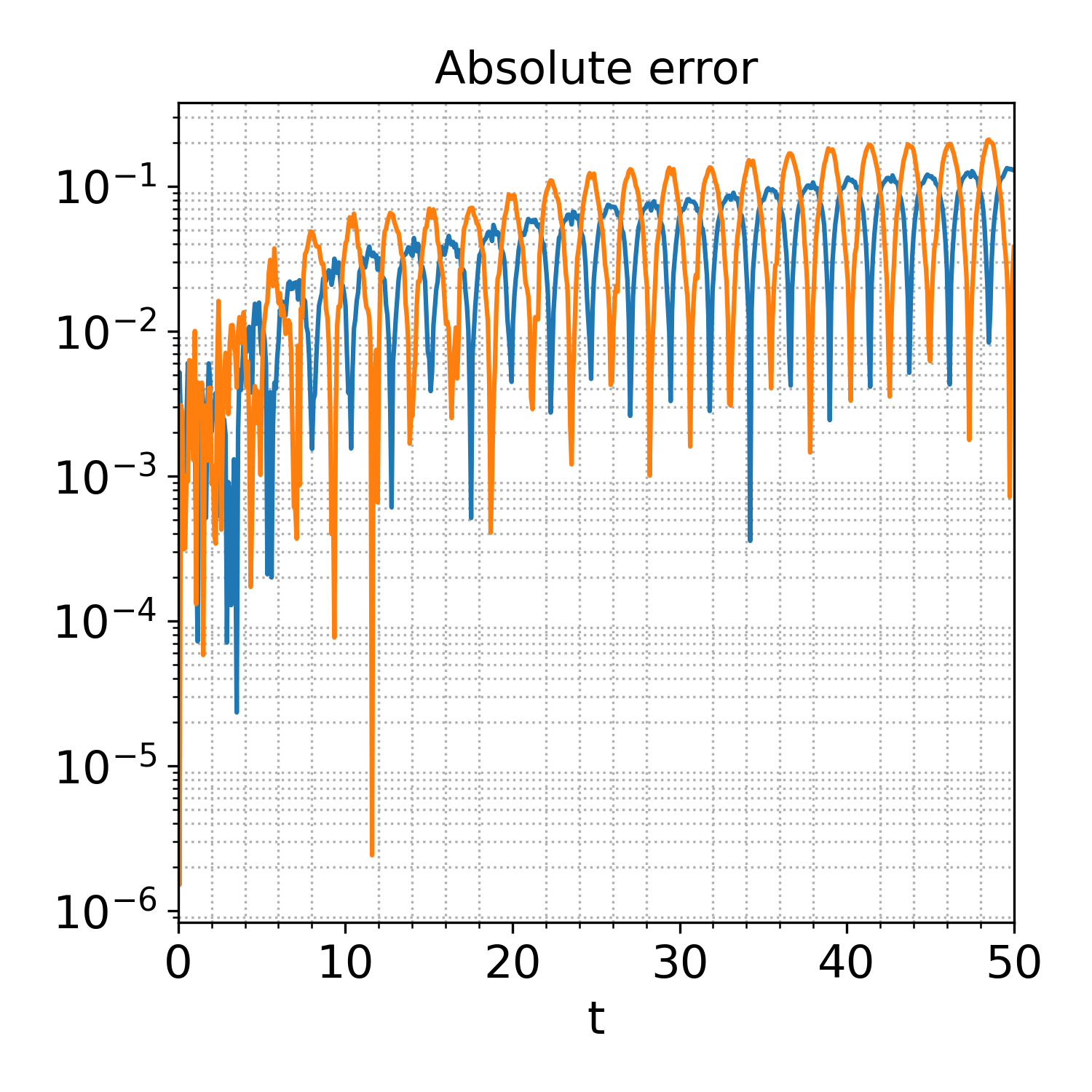}
		\caption{Absolute error}
	\end{subfigure}
	\caption{Nonlinear oscillator: Comparison of the learned model with the ground truth model 
		for the harmonic oscillator over time using a random initial condition.}
	\label{fig:oscil-time}
\end{figure*}
Furthermore, in \Cref{fig:oscil-ham}, we plot the learned and canonical Hamiltonians for the nonlinear oscillator,
which demonstrates that all Hamiltonians remain constant over time with minor fluctuations, as seen in previous examples.
\begin{figure*}[tb]
	\centering
	\begin{subfigure}[b]{0.3\textwidth}
		\includegraphics[width=1\linewidth]{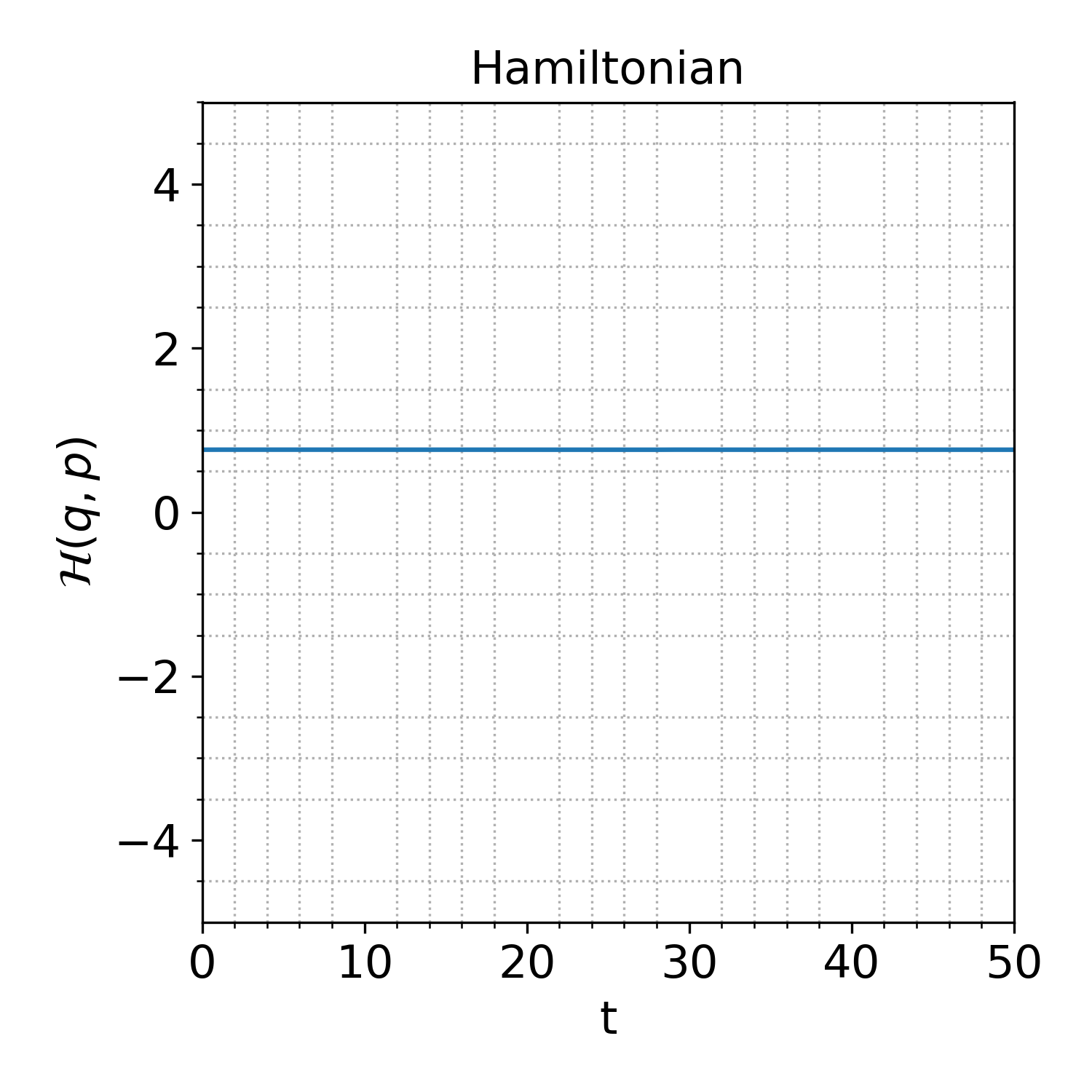}
		\caption{Ground truth}
	\end{subfigure}
	\begin{subfigure}[b]{0.3\textwidth}
		\includegraphics[width=1\linewidth]{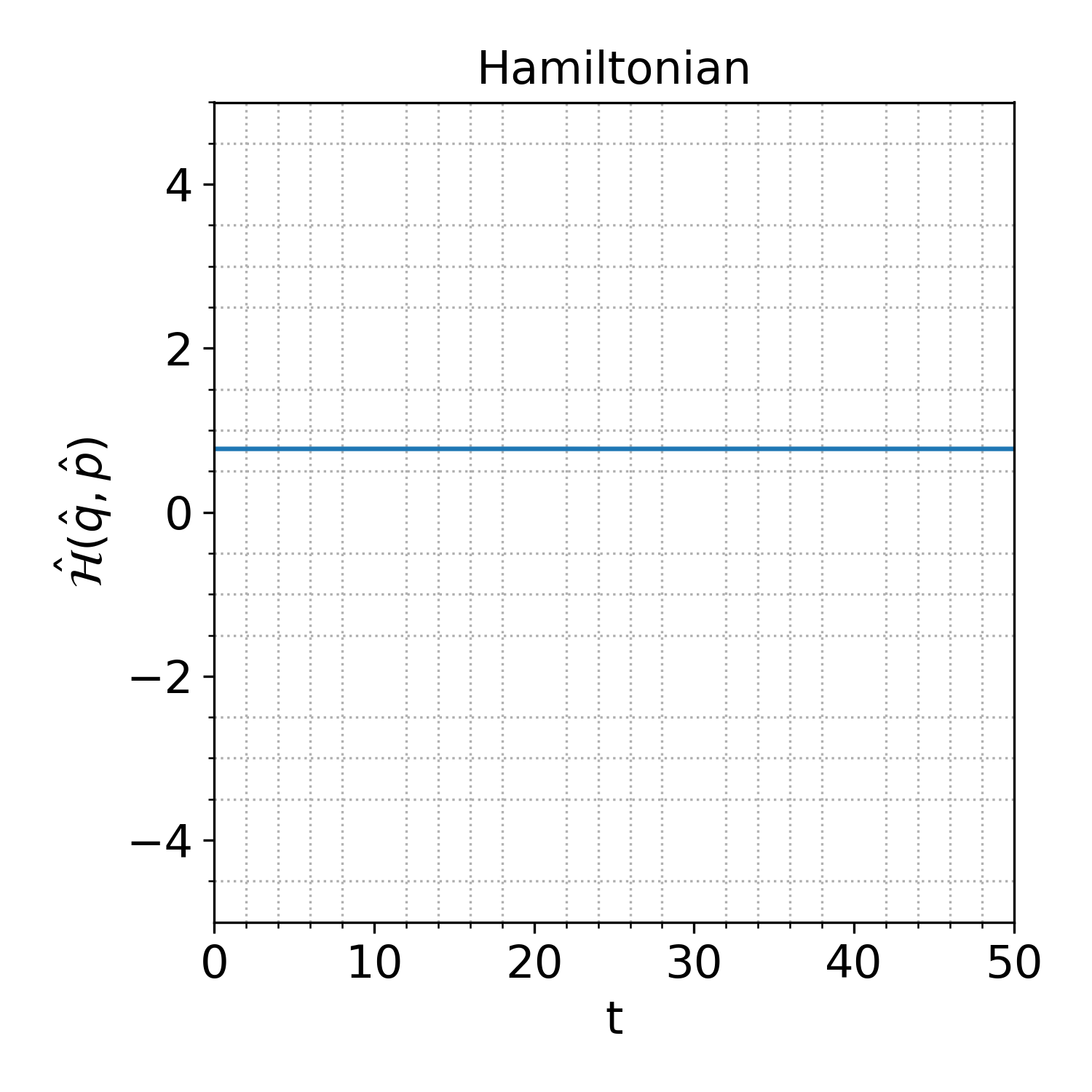}
		\caption{Learned model}
	\end{subfigure}
	\begin{subfigure}[b]{0.3\textwidth}
		\includegraphics[width=1\linewidth]{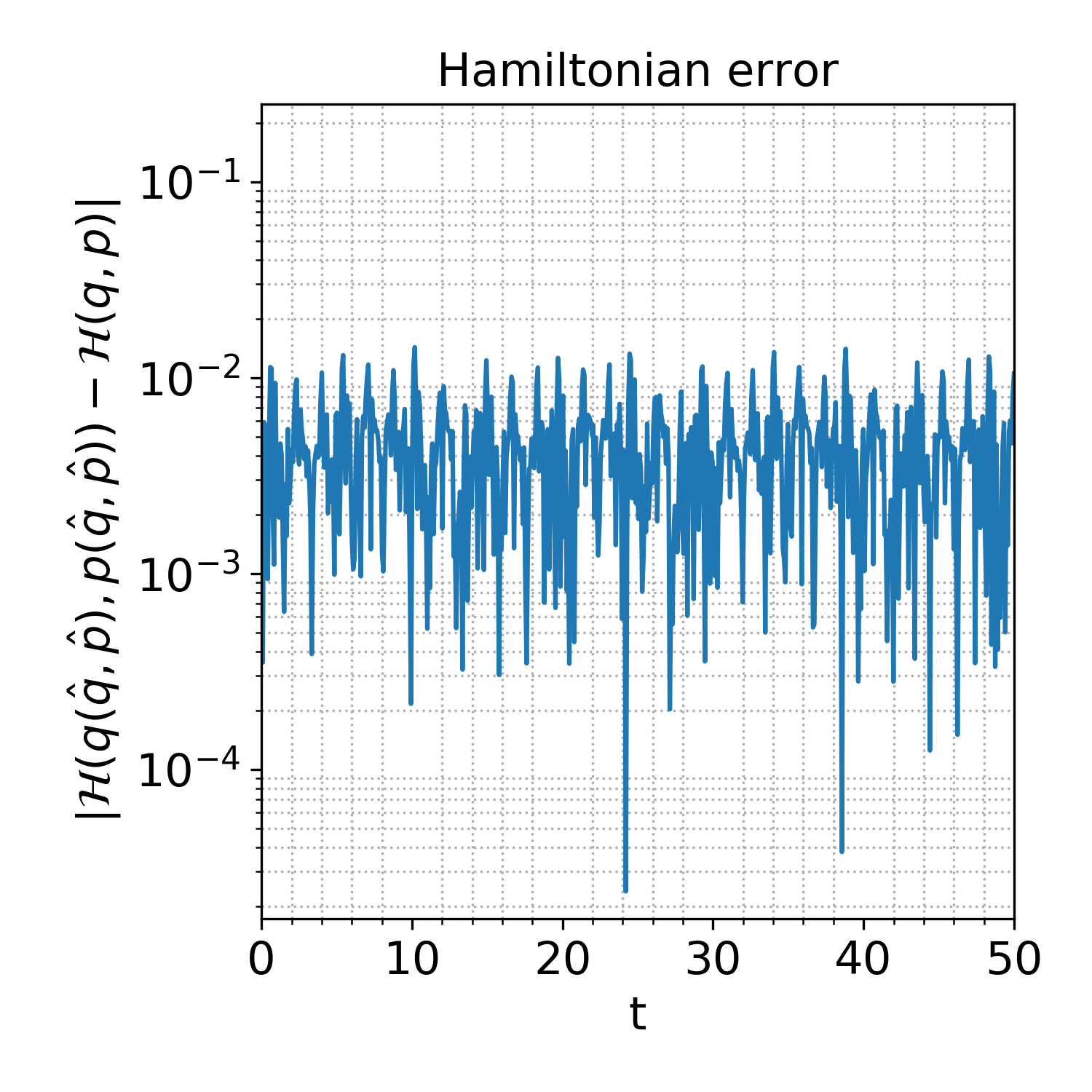}
		\caption{Absolute error}
	\end{subfigure}
	\caption{Nonlinear oscillator: A comparison of the Hamiltonian in canonical coordinates for the ground truth model $\Hamiltonian(q,p)$, the learned Hamiltonian $\hat{\Hamiltonian}(\hat q, \hat p)$ in the latent space, 
	and the difference between the ground truth model and the learned model in the original space $\Hamiltonian(q(\hat q, \hat p), p(\hat q, \hat p))$ along time using a random test initial condition.}
	\label{fig:oscil-ham}
\end{figure*}

To visualize the importance of the symplectic structure of the auto-encoder to accurately learn the system dynamics as a quadratic Hamiltonian system, we learn the an-harmonic oscillator in the same setting, except with the weight for the symplecticity condition of the auto-encoder in \eqref{eq:total_loss} set to $\lambda_2 = 0$. This leads to an unstable system, c.f. \Cref{fig:oscil-phase_nosymp}, despite the solution being periodic. The learned systems follows the true solution for a few periods, but then becomes unstable. This shows that the structure of a weakly-enforced symplectic auto-encoder together with the Hamiltonian dynamics in the latent space is needed to accurately learn and extrapolate the dynamics of the nonlinear oscillator. Moreover, \Cref{tab:no-loss} demonstrates the training performance of both cases, where the table shows that although ground truth and learned model are both canonical Hamiltonian systems, the auto-encoder is not symplectic for $\lambda_2 = 0$.    
\subsubsection{Wooden Pendulum Clock}
In our final low-dimensional example, we consider a real-world experiment involving a pendulum clock, as detailed in the referenced study \cite{choudhary2021forecasting}. The pendulum's oscillations are governed by a descending weight and a deadbeat escapement, which compensates energy loss, leading to a nearly constant frequency and amplitude of motion. The pendulum's angles and angular velocities are obtained from a video that records its swings for 100 seconds at 30 frames per second. This is achieved through video tracking the ends of the compound pendulum. The angles are derived from the coordinate differences using trigonometry, and the angular velocities are estimated using finite differences with a Savgol filter \cite{savitzky1964smoothing}. We used these angle coordinate data from the experiment, which are available on the GitHub page of the study \cite{choudhary2020}.

To train our model, we used $100$ discrete time instances of the given trajectory, spanning the training time domain $[0,3.314]$. Furthermore, in this example, we compared the learned dynamics of the pendulum clock across different latent dimensions, using the same parameters for training. In \Cref{fig:pend_wood-exp-time}, we display the experimental angles $q$ and corresponding angular velocities $p$, which shows that the experimental data used for training contains some measurement errors. In order to compare the dynamics learned through various latent dimensions, we trained our model using both the latent dimensions $2$ and $4$. The reconstructed angles and angular velocities for these dimensions are shown in \Cref{fig:pend_wood-lat2-time} and \Cref{fig:pend_wood-lat4-time} for latent dimensions $2$ and $4$, respectively. \Cref{fig:pend_wood-time} illustrates that both models yield results similar to the experimental data. However, the accuracy of the methods becomes more apparent when observing the absolute errors between the experimental and learned dynamics, as shown in \Cref{fig:pend_wood-time-abs}.  Finally, we present the experimental data and learned dynamics for latent dimensions $2$ and $4$ in the phase space in \Cref{fig:pend_wood-phase}. This shows that although the experimental dynamics are dissipative, the learned model is not affected by dissipation.

\begin{figure*}[tb]
	\centering
	\begin{subfigure}[b]{0.3\textwidth}
		\includegraphics[width=1\linewidth]{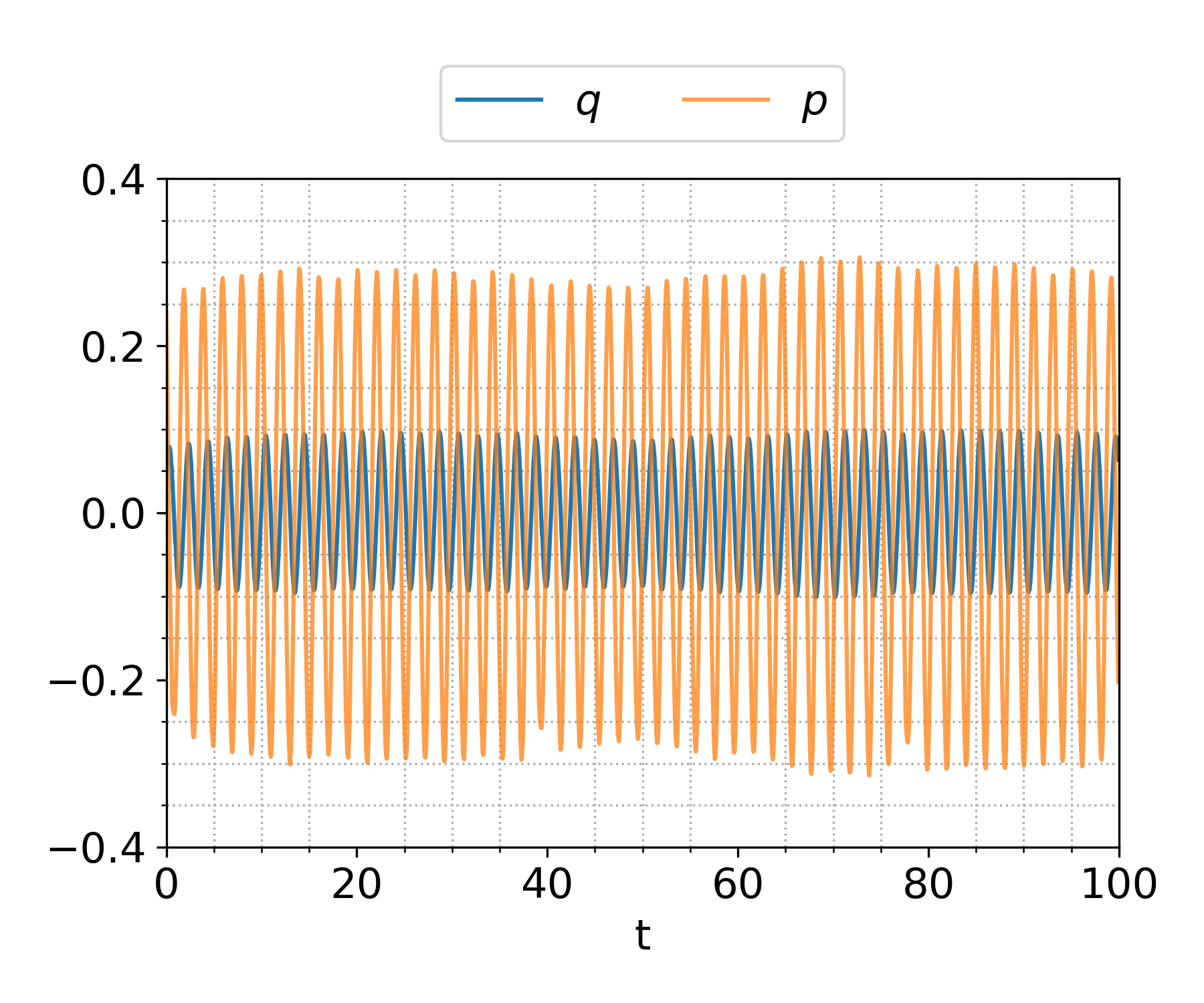}
		\caption{Experimental data\\   \vspace*{4mm}}
		\label{fig:pend_wood-exp-time}
	\end{subfigure}
	\begin{subfigure}[b]{0.3\textwidth}
		\includegraphics[width=1\linewidth]{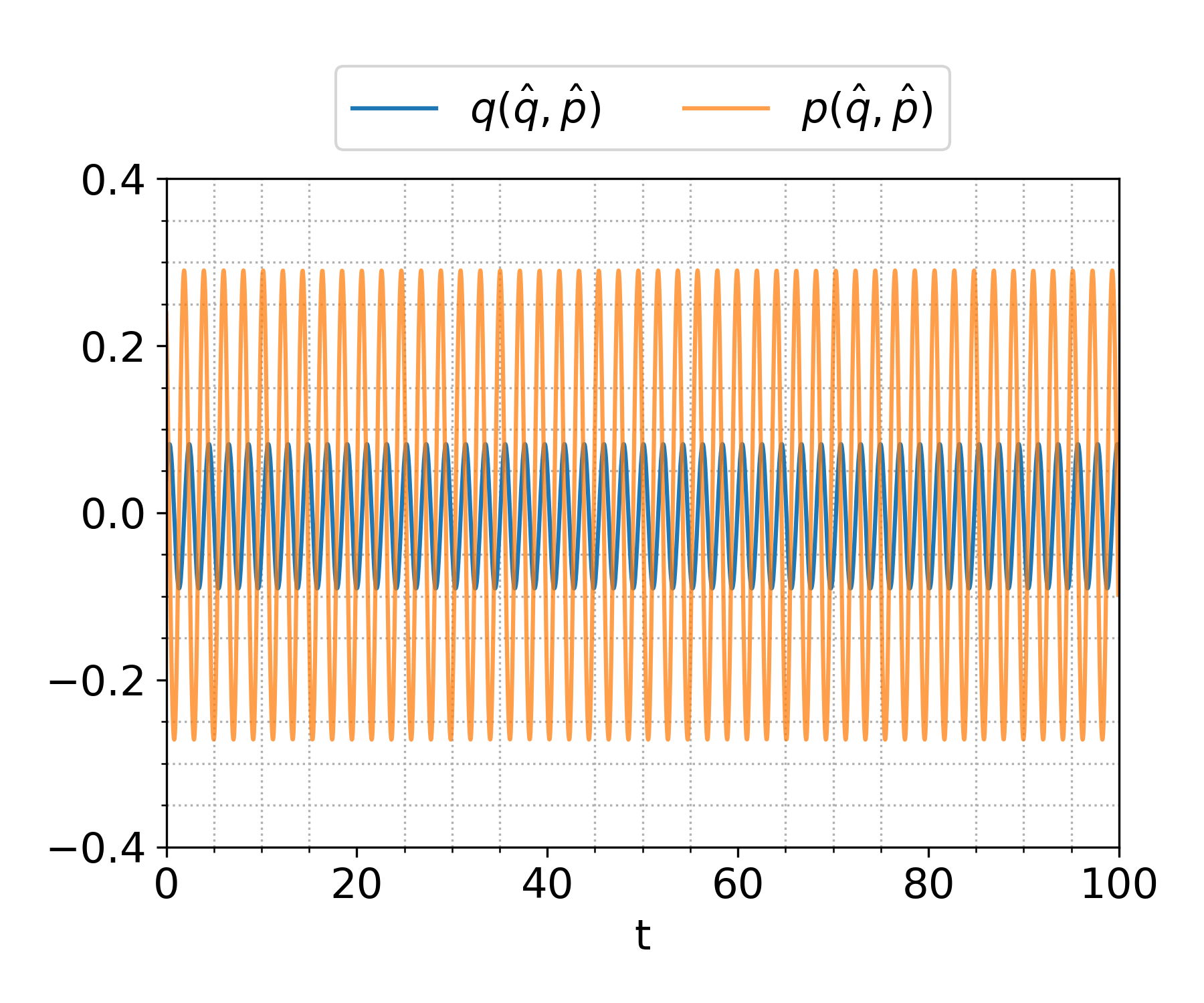}
		\caption{Learned model with\\ \quad latent dimension $2$}
		\label{fig:pend_wood-lat2-time}
	\end{subfigure}
	\begin{subfigure}[b]{0.3\textwidth}
		\includegraphics[width=1\linewidth]{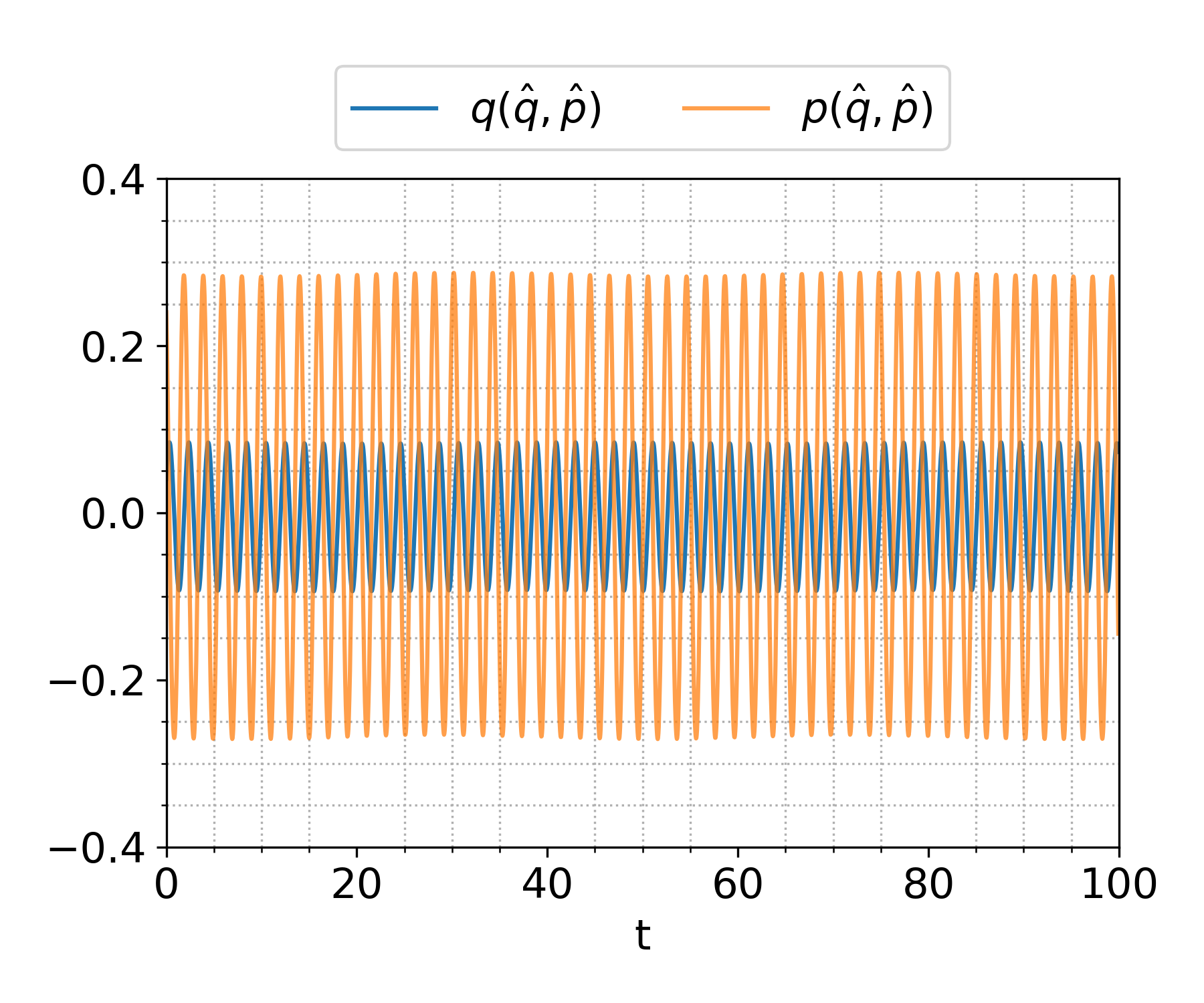}
		\caption{Learned model with\\ \quad latent dimension $4$}
		\label{fig:pend_wood-lat4-time}
	\end{subfigure}
	\caption{Wooden pendulum: A comparison of the learned model with the experimental data.}
	\label{fig:pend_wood-time}
\end{figure*}

\begin{figure*}[tb]
	\centering
	\begin{subfigure}[b]{0.42\textwidth}
		\includegraphics[width=0.8\linewidth]{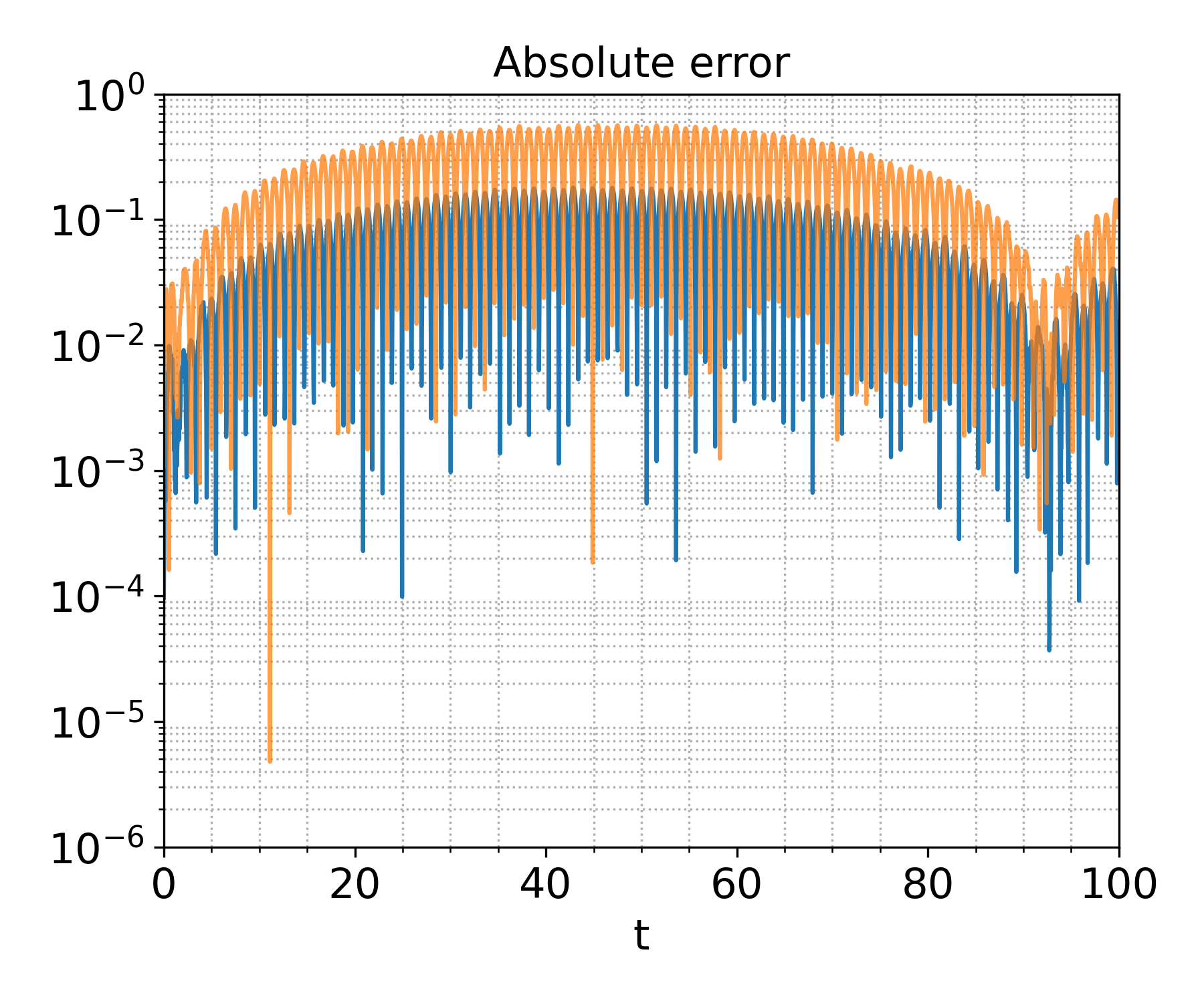}
		\caption{Learned model with\\ \quad latent dimension $2$}
	\end{subfigure}
	\begin{subfigure}[b]{0.42\textwidth}
		\includegraphics[width=0.8\linewidth]{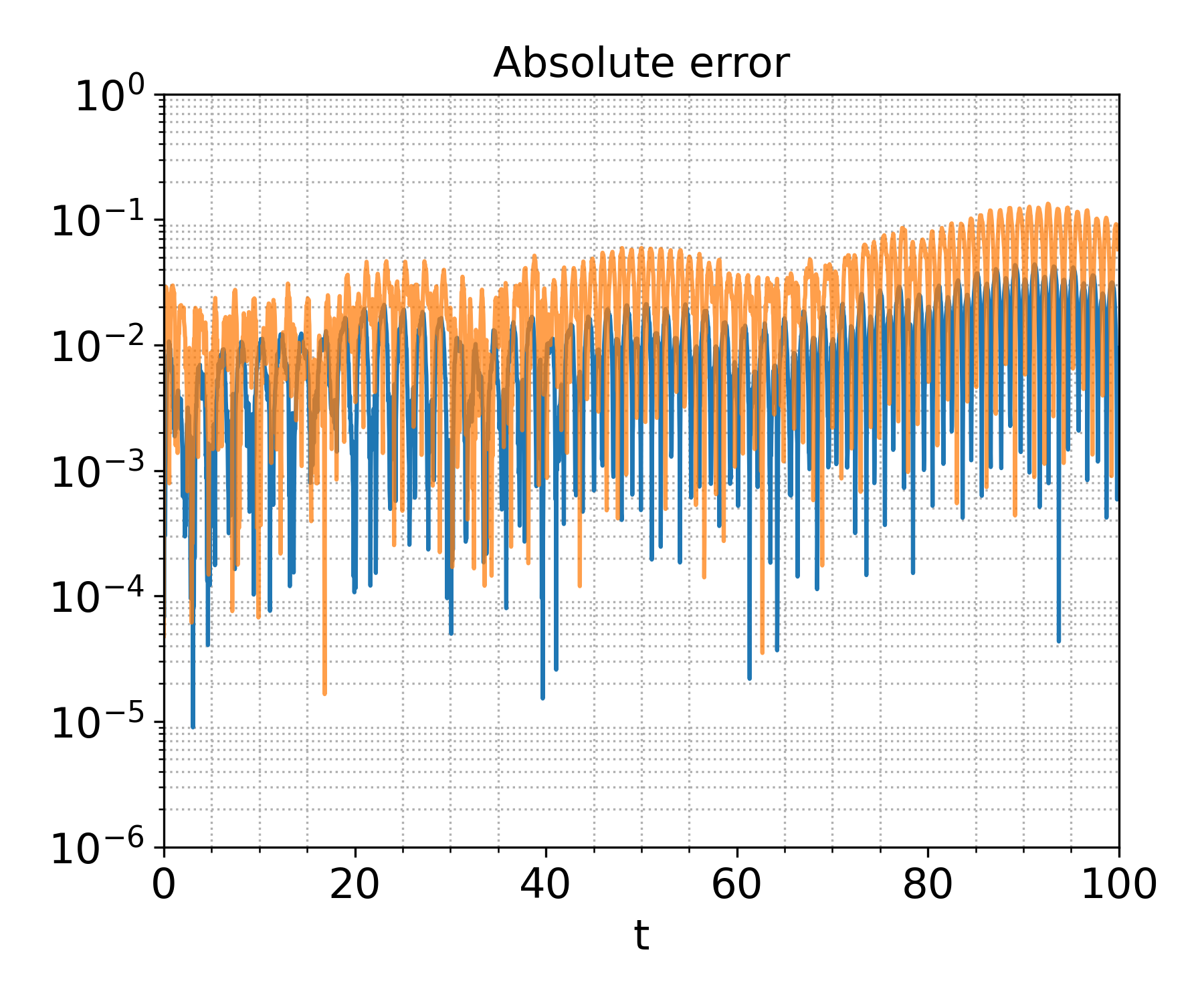}
		\caption{Learned model with\\ \quad latent dimension $4$}
	\end{subfigure}
	\caption{Wooden pendulum: Absolute error between experimental data and learned model of latent dimensions $2$ and $4$.}
	\label{fig:pend_wood-time-abs}
\end{figure*}

\begin{figure*}[tb]
	\centering
	\begin{subfigure}[b]{0.42\textwidth}
		\includegraphics[width=0.8\linewidth]{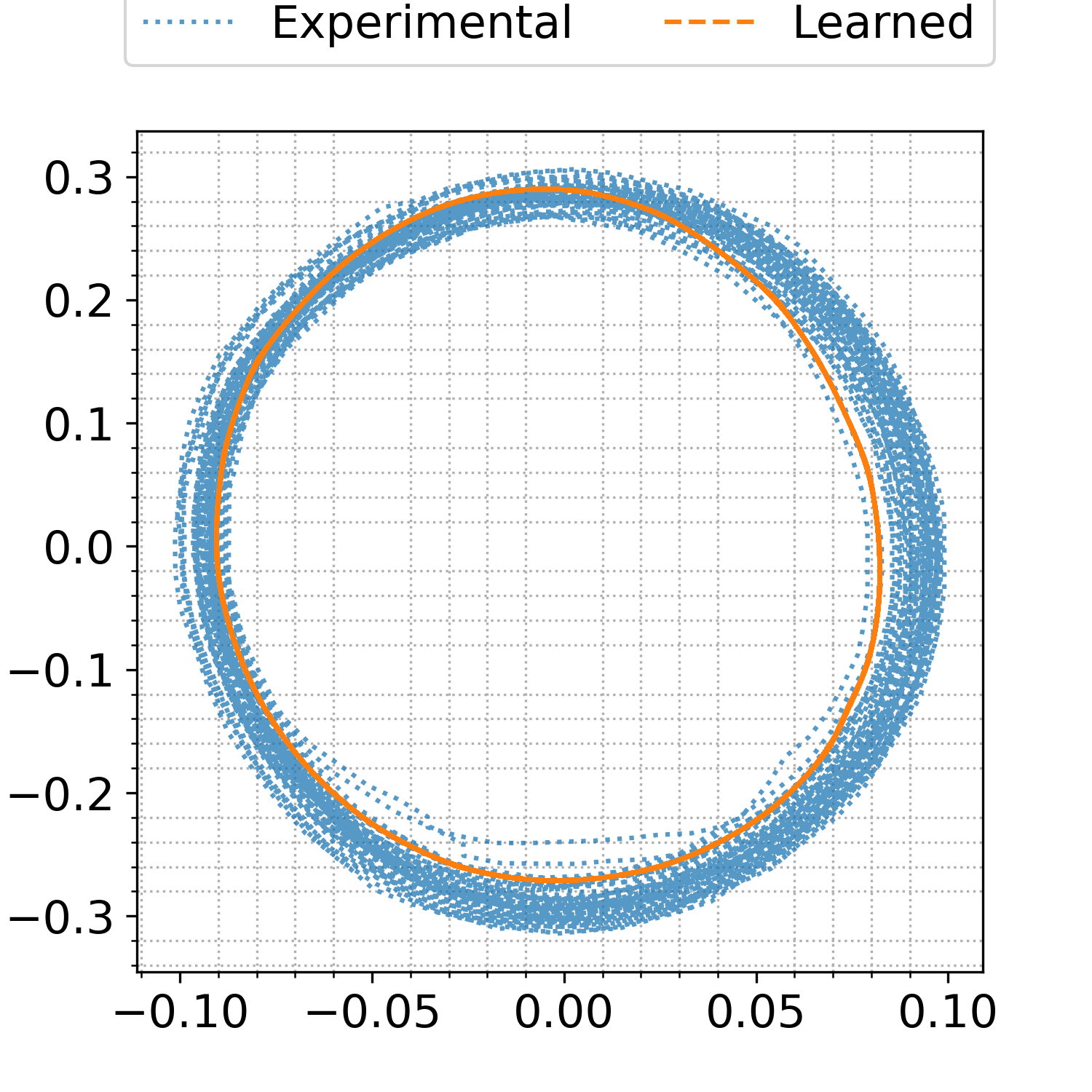}
		\caption{Learned model with\\ \quad latent dimension $2$}
	\end{subfigure}
	\begin{subfigure}[b]{0.42\textwidth}
		\includegraphics[width=0.8\linewidth]{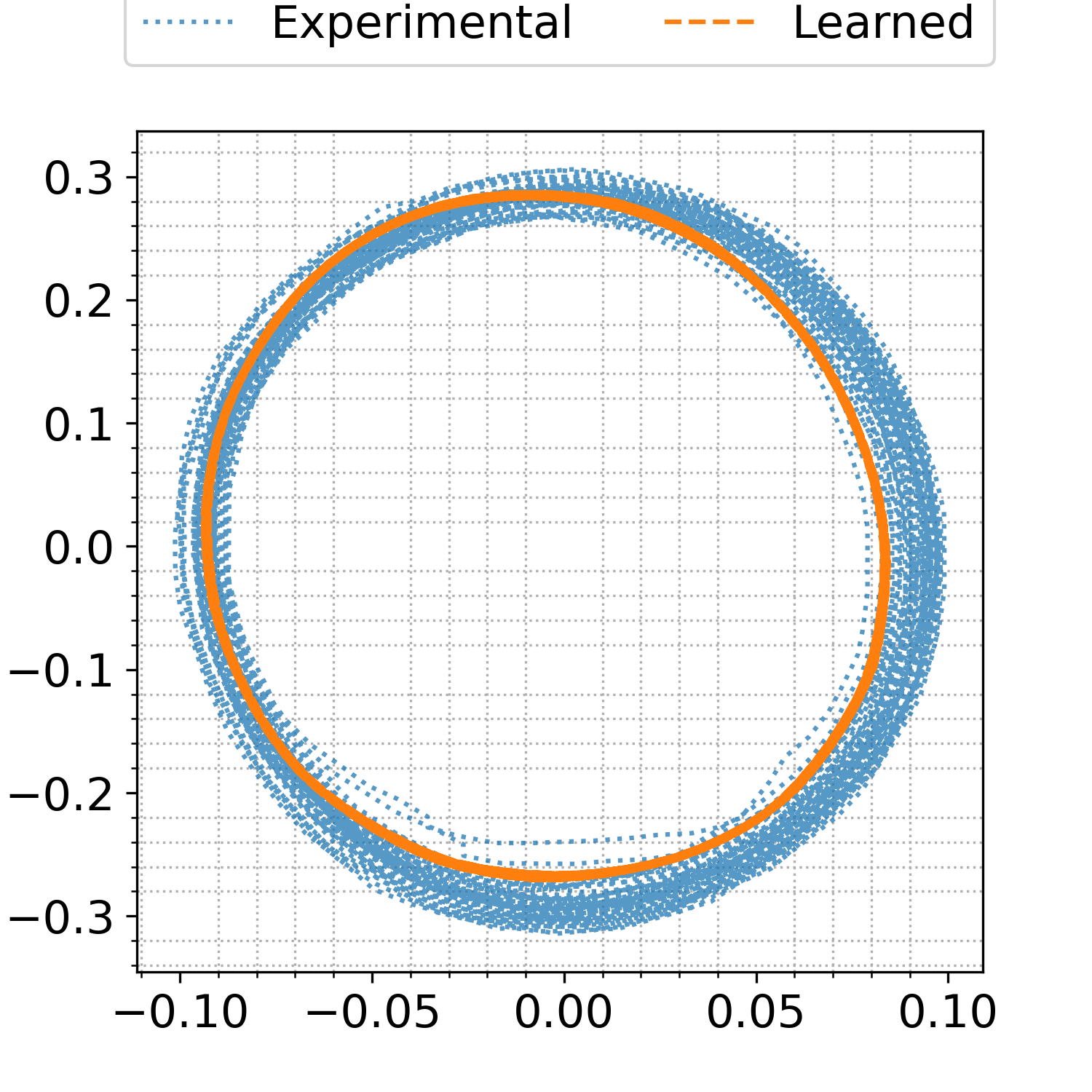}
		\caption{Learned model with\\ \quad latent dimension $4$}
	\end{subfigure}
	\caption{Wooden pendulum: Phase space comparison of the learned model with the experimental data.}
	\label{fig:pend_wood-phase}
\end{figure*}

\subsection{High-dimensional Systems}\label{sec:highdim} 
Next, we focus on learning low-dimensional models for high-dimensional data coming from high-dimensional systems. 
\subsubsection{Linear Wave Equation}
We begin by considering a simple linear wave equation of the form:
\begin{equation}\label{eqn:wave}
\begin{aligned}
&u_{tt} =cu_{xx},\\
& u(t_0,x) = u^0(x), & x \ \text{in}\;\Omega,  
\end{aligned}
\end{equation}
where $c$ is the transport velocity, and boundary conditions are set to be periodic. The wave equation is an example of a Hamiltonian PDE. 
By defining the variables $p=u_t$ and $q=u$, we can obtain 
the Hamiltonian form of 
the wave equation \cite{bridges2006}, which is given by
\begin{equation}\label{eqn:Ham_wave}
\frac{\partial z}{\partial t}
=
\begin{bmatrix}
0&1\\-1&0
\end{bmatrix}
\frac{\delta \Hamiltonian}{\delta z}
,\quad z=
\begin{bmatrix}
q\\p
\end{bmatrix},
\end{equation}
where $\frac{\delta \Hamiltonian}{\delta z} $ is the variational derivative and the Hamiltonian is given as
\begin{equation*}
\Hamiltonian(u) = \dfrac{1}{2}\int_{\Omega}
cq_x^2+p^2\,dx.
\end{equation*}

Next, we discretize the Hamiltonian form of the wave equation and obtain the following semi-discrete Hamiltonian ODE system:
\begin{equation}\label{eqn:Wave-Ham-ODE}
\frac{d \bl z}{d t}=\bl K \bl z, 
\end{equation}
where
$$\bl z=\begin{bmatrix}
\bl q\\ \bl p
\end{bmatrix}, \quad \bl{K}=\begin{bmatrix}
\bl{0}_N&\bl{I}_N\\ c\bl{D}_{xx}& \bl{0}_N
\end{bmatrix}, $$
$ \bl{D}_{xx}\in \mathbb{R}^{N\times N}$ is the three-point central difference approximation of $\partial_{xx} $, $\bl{0}_N\in \mathbb{R}^{N\times N}$ is a matrix of zeros, $\bl{I}_N\in \mathbb{R}^{N\times N}$ is the identity matrix,
and $(\bl q, \bl p) \in \R^{2N}$ are the discretized $(q,p)$.

In this task, we focus on learning a single wave equation over a single trajectory. For this purpose, we set the initial condition to $u^0(x)=\text{sech}(x)$.
For training purposes, we have generated data of the wave equation on the domain $\Omega=[-5,5]$ up to time $T=20$ with time-step size $\Delta t=0.05$. We set the spatial dimension for the ground truth model of wave equation \eqref{eqn:Wave-Ham-ODE} to $2N=1024$ and the learned problem dimension to $2n = 4$.

We compare the learned model with the ground truth in \Cref{fig:waveq-time,fig:wavep-time},
for the states $q$ and $p$, respectively. The figures show that the obtained model is stable and accurate over a long time horizon.
\begin{figure*}[tb]
	\centering
	\begin{subfigure}[t]{1\textwidth}
		\centering
		\includegraphics[width=0.8\linewidth]{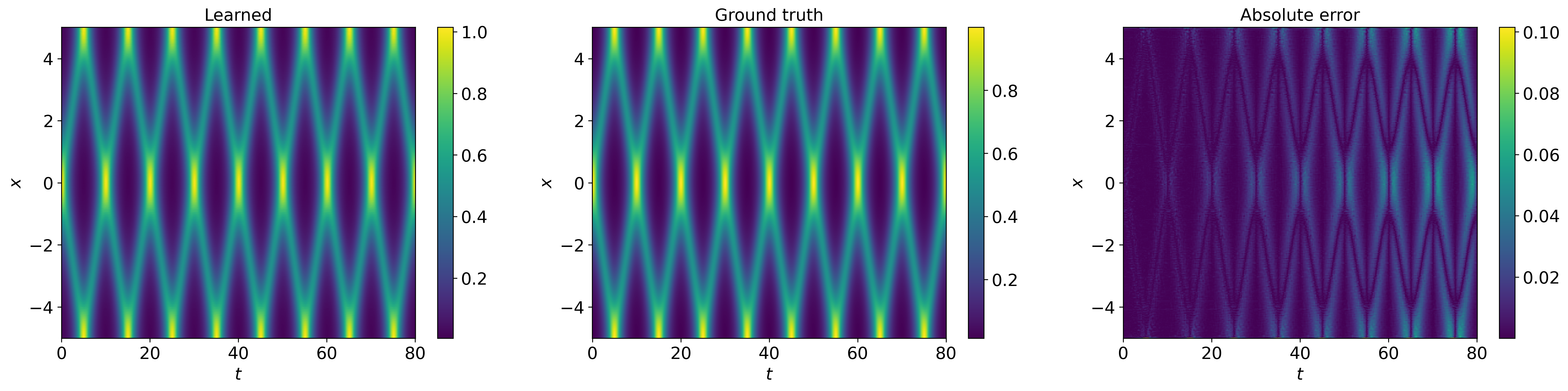}
		\caption{A comparison of the position $q$.}
		\label{fig:waveq-time}
	\end{subfigure}
	\begin{subfigure}[t]{1\textwidth}
		\centering
		\includegraphics[width=0.8\linewidth]{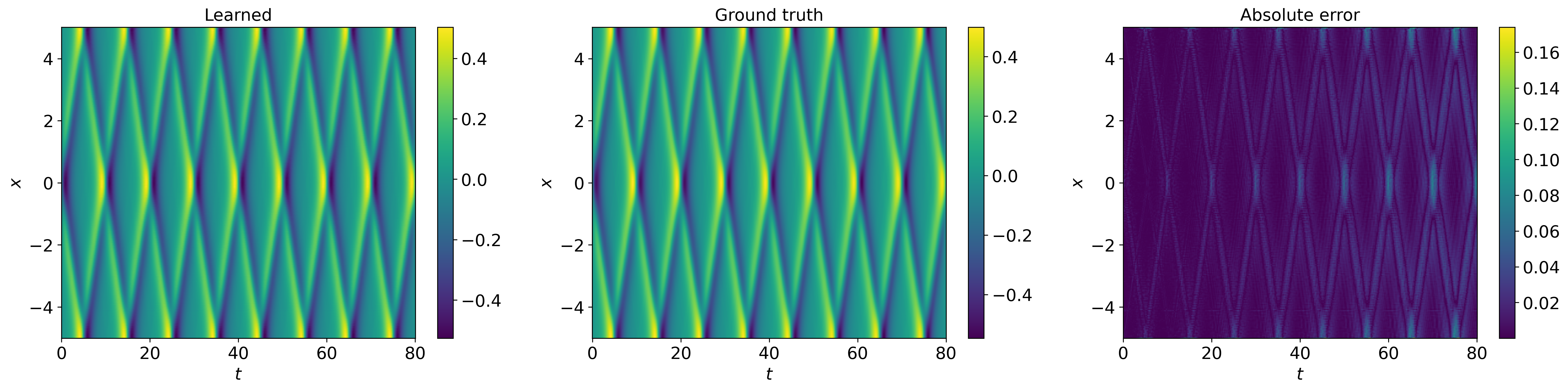}
		\caption{A comparison of the momenta $p$.}
		\label{fig:wavep-time}
	\end{subfigure}
	\caption{Wave equation: Comparisons of the position $q$ and momenta $p$ obtained using the learned model with the ground truth wave model \eqref{eqn:Ham_wave}.}
\end{figure*}
Next, we examine the learned variables in phase space and time domain in \Cref{fig:wave-phase}, which shows that the learned variables 
are orbiting on one particular energy level in phase space and are stable in the time domain as well.

\begin{figure*}[tb]
	\centering
	\includegraphics[width=0.8\linewidth]{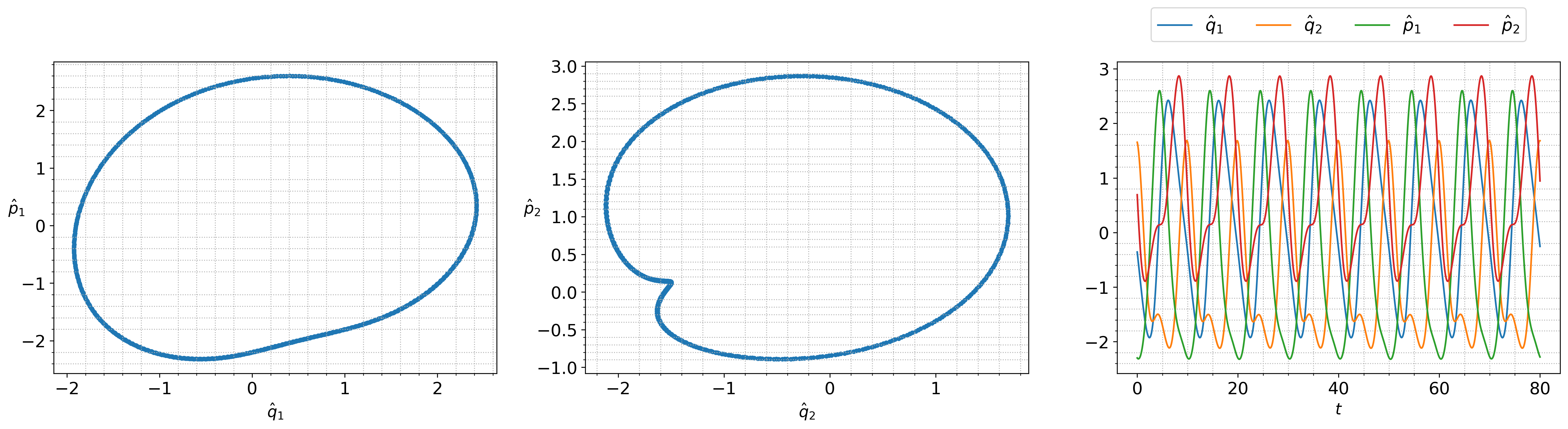}
	\caption{Wave equation: Learned variables in phase space and time domain.}
	\label{fig:wave-phase}
\end{figure*}

\subsubsection{Nonlinear Schrödinger Equation}
Finally, we test the ability of our model to learn the nonlinear Schrödinger (NLS) equation in the last example of a high-dimensional problem. 
The NLS equation has various use cases, e.g., small-amplitude gravity waves  on the surface of deep water with zero-viscosity, 
in the study of Bose-Einstein condensation, and the propagation of light in nonlinear optical fibers.
Specifically, we look at the cubic Schrödinger equation which is given \cite{karasozen2018energy} by
\begin{equation}\label{eqn:schrod}
    \begin{aligned}
    & i \dfrac{\partial u}{\partial t} + \alpha u_{xx}+ \beta |u|^2 u=0, \\
    & u(t_0,x) = u^0(x), & x \ \text{in}\;\Omega,
    \end{aligned}
\end{equation}
with periodic boundary conditions. In the NLS equation \eqref{eqn:schrod}, the parameter $\alpha$ is a non-negative constant and the constant parameter $\beta$ is the focusing (or attractive)---with negative---and defocusing (or repulsive) nonlinearity---with positive---values. In this example,
 we have fixed the parameters to $\alpha=\tfrac{1}{2}$, $\beta=1$, and the domain is fixed to $\Omega=[-10,10]$.

To obtain the canonical Hamiltonian form of the NLS equation \eqref{eqn:schrod}, we write the complex-valued solution $u$ in terms of its
imaginary and real parts as
$u=q+ i p$. Then, the Hamiltonian of the NLS equations reads as 
\begin{equation*}
    \Hamiltonian(u) = \dfrac{1}{2}\int_{\Omega} \bigg[
    \alpha (q_x)^2+
    \alpha (p_x)^2-
    \dfrac{\beta}{2} (q^2+p^2)^2
    \bigg]\,dx.
\end{equation*}

We used the same discretization as in the linear wave equation. The NLS equation was simulated with the initial condition $u^0(x)=\text{sech}(x)$ in the time domain up to final time $T=80$ with time-step size $\Delta t=0.05$.
The spatial dimension for the ground truth model was fixed to $2N=1024$. We used half of the obtained data, i.e., up to $T_{\text{train}}=40$, to train our model. We set the dimension of the
learned model to $2n =2$.

Having learned the desired embedding, in \Cref{fig:nlsq-time,fig:nlsp-time}, we present a comparison of the dynamics of the ground truth model \eqref{eqn:schrod} and the learned model, 
as well as the corresponding absolute error, in the time domain for the states $q$ and $p$, respectively. The figures show that there is a good qualitative agreement between
the solution of the learned model and the ground truth with an error of order $10^{-2}$. This shows that the learned model can infer the dynamics of the nonlinear Schrödinger equation \eqref{eqn:schrod}.
\begin{figure*}[tb]
	\centering
	\begin{subfigure}[t]{1\textwidth}
		\centering
   \includegraphics[width=0.8\linewidth]{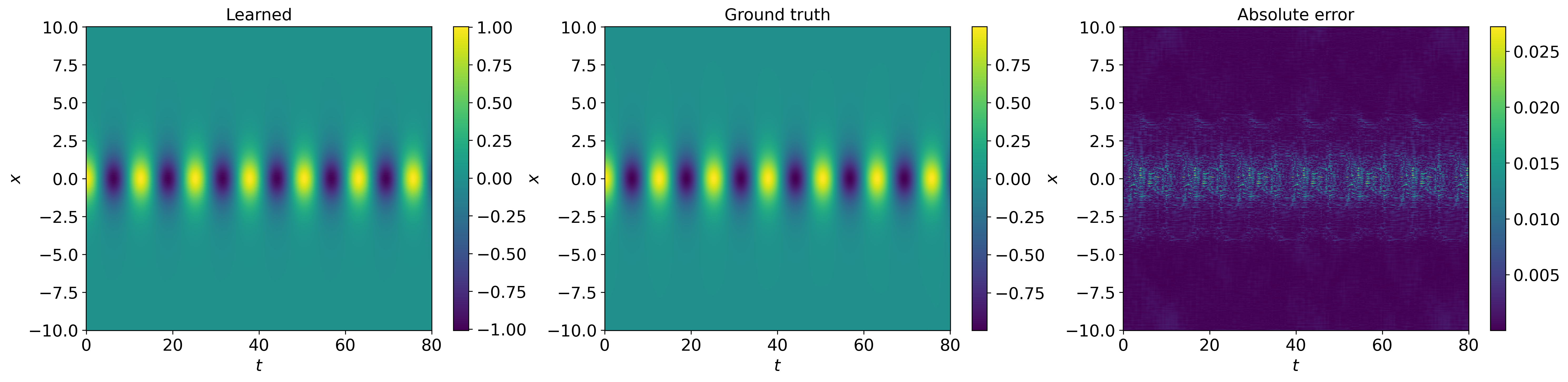}
		\caption{A comparison of the position $q$.}
   \label{fig:nlsq-time}
	\end{subfigure}
	\begin{subfigure}[t]{1\textwidth}
		\centering
   \includegraphics[width=0.8\linewidth]{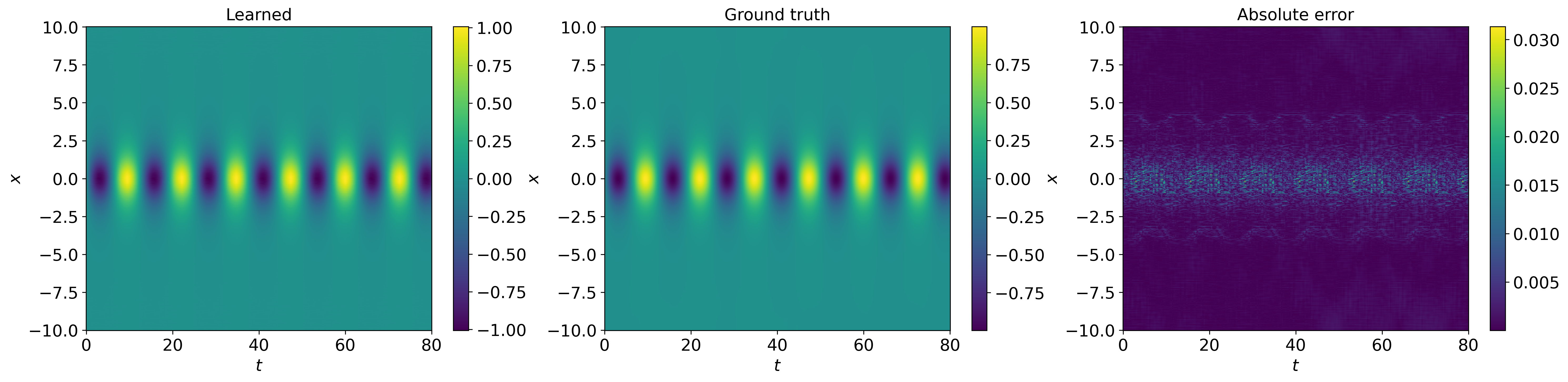}
		\caption{A comparison of the momenta $p$.}
   \label{fig:nlsp-time}
	\end{subfigure}
	\caption{Nonlinear Schrödinger equation: Comparisons of the position $q$ and momenta $p$ obtained using the learned model with the ground truth nonlinear Schrödinger model \eqref{eqn:schrod}.}
\end{figure*}
Lastly, we plot the dynamics of the learned latent phase space and time domain simulation from 
the nonlinear Schrödinger equation \eqref{eqn:schrod} in \Cref{fig:nls-phase}
 to present the stability of the learned dynamics. \Cref{fig:nls-phase} shows that the learned model is suitable for long time
 integration.

\begin{figure*}[tb]
   \centering
   \includegraphics[width=0.6\linewidth]{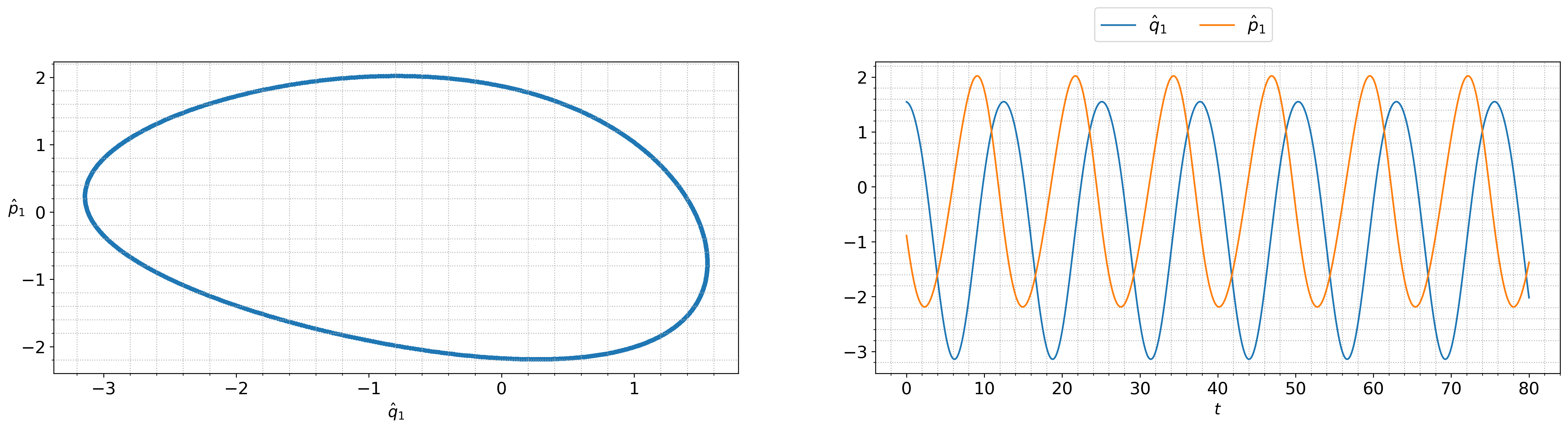}
   \caption{Nonlinear Schrödinger equation: Learned variables in phase space and time domain.}
   \label{fig:nls-phase}
\end{figure*}

 \section{Conclusions}\label{sec:conc}
 
 In this work, we have discussed the concept of data-driven quadratic symplectic representations of nonlinear Hamiltonian systems. We have defined an embedding as the lifting of original data coming from nonlinear Hamiltonian systems using a symplectic transformation, resulting in quadratic systems that describe the dynamics in the lifted space, with a cubic function as the Hamiltonian. The structure of the dynamics can be enforced by using weakly-enforced  symplectic auto-encoders and symmetric tensors.
 
This approach enables us to obtain a learned symplectic lifting. Additionally, for high-dimensional data, we discuss symplectic reduction to achieve a quadratic representation, leading to a low-dimensional quadratic Hamiltonian system. The advantage of this approach over structure-preserving model order reduction is that we directly learn the reduced dynamics fitting the data, eliminating the need for hyper reduction methods or taking gradients through the auto-encoder in the online phase. We note that the proposed methodology does not require to know the full-order model in a discretized form. However, we note that the learning stage of the proposed methodology is computationally demanding and its cost is dominated by the forward pass through an autoencoder and derivative through the autoencoder. As a result, for high-dimensional data, arising by collecting high-dimensional solutions for multiple initial conditions and parameter settings, the proposed methodology is computationally too expensive. Therefore, we will investigate in the future how to tailor the approach to make it more computationally efficient for high-dimensional big data.

We have demonstrated the efficiency of the proposed methodology by means of several low-dimensional and high-dimensional examples, illustrating the preservation of the Hamiltonian, i.e., energy, and long-term stability in extrapolation settings. 
We have shown that our model is able to generalize in the sense of this extrapolation and is able to compute trajectories for a set of initial data not seen in the training, where the training and extrapolation data have similar distributions.
In our future work, we investigate the effect of noise on the performance of the methodology and propose suitable treatments to it, for example, tailoring the approach proposed in \cite{morGoyB23}. Additionally, we will investigate different auto-encoder structures to obtain efficient models for high-dimensional parametric Hamiltonian systems. Lastly, extensions to discrete Hamiltonian systems and externally controlled Hamiltonian systems would be valuable contributions.  

\section*{Funding Statement}
Süleyman Y\i ld\i z and Peter Benner are partially supported by the German Research Foundation (DFG) Research Training Group 2297 ``MathCoRe'', Magdeburg.
\section*{Data Availability Statement}

Data and relevant code for this research work have been archived within the Zenodo Repository \cite{QuadHam23}.

\appendix

\appendix
\section{Implementation Details}\label{appendix:A}
\Cref{tab:hyperparameters-low,tab:hyperparameters-high} contain all the necessary hyper-parameters for our illustrative examples. We set the hyper-parameters experimentally by monitoring the performance of the learned model on training data. For the symplectic lifting case, we have set the hyper-parameters $(\lambda_{1}, \lambda_{2}, \lambda_{3})$ to $(10^{-1},1,1)$ by monitoring all the losses to obtain a balanced decrease of all the losses simultaneously, while in the symplectic reduction case, the hyper-parameters $(\lambda_{1}, \lambda_{2}, \lambda_{3})$ are set to $(1,10^{-1},10^{-1})$ for the same goal. In order to deal with the inaccuracy of the reconstruction due to the structure of the auto-encoder, we have applied the penalisation:
\begin{equation}\label{eq:pen1}
	\mathcal L_{\text{Rec}}=0.5\| x(t)-\phi(\psi(x(t)))\|_{\text{MAE}},
\end{equation}
where $\| \cdot \|_{\text{MAE}}$ denotes the mean absolute error, averaged over all samples and dimensions. Similarly, we have penalized the parameters of $\mathcal L_{\dot z\dot x}$ with the mean absolute error scaled with a hyper-parameter $10^{-5}$. Finally, we used fixed decay of the learning rate in both the symplectic reduction and lifting cases, using the \texttt{StepLR} implementation in \texttt{PyTorch}. We
experimentally fixed both the decay rate and the decay step by monitoring the decay of the total loss function.

For the symplectic lifting case, we have used a Multi Layer Perceptron (MLP) architecture with skip connections and three hidden layer.
\begin{table}
	\renewcommand{\arraystretch}{1.25}
\caption{	The table contains all the hyper-parameters to learn the dynamics of the low-dimensional examples.}
	\label{tab:hyperparameters-low}
	\begin{tabular}{|p{7em}|c|c|c|c|}
		\hline
		{\bfseries Parameters}                                                              & \begin{tabular}[c]{@{}c@{}}{\bfseries Pendulum} \\ {\bfseries example}\end{tabular} & \begin{tabular}[c]{@{}c@{}}{\bfseries Lotka-Volterra} \\ {\bfseries example}\end{tabular} & \begin{tabular}[c]{@{}c@{}}{\bfseries Nonlinear}\\{\bfseries oscillator} \\ {\bfseries example}\end{tabular}&
		\begin{tabular}[c]{@{}c@{}}{\bfseries Pendulum} \\ {\bfseries clock}\\ {\bfseries example}\end{tabular}
		\\ \hline
		Input dimension    & $ 2$      & $2$      &$2$       & $2$   \\ \hline
		Encoder layers [neurons]  & $ [64,64,64]$   & $[32,32,32] $   &  $[32,32,32]$ & $ [64,64,64]$  \\ \hline
		\begin{tabular}[c]{@{}c@{}}Lifted coordinate \\ latent dimension\end{tabular} & 
		$4$  & $4$    &$ 4$    &  $[2,4]$  \\ \hline
		Decoder layers [neurons] & $ [64,64,64]$ & $[32,32,32] $ & $[32,32,32]$ & $ [64,64,64]$ \\ \hline
		Learning rate   & $3\cdot10^{-3}$    &$3\cdot10^{-3}$    & $3\cdot10^{-3}$ & $1\cdot10^{-3}$ \\ \hline
		Batch size   & $5$ & $5$ & $20$& $10$    \\ \hline
		Activation function & \texttt{selu} & \texttt{selu}  & \texttt{selu}& \texttt{selu}   \\ \hline
		Weight decay & $10^{-5}$ & $10^{-5}$ & $10^{-5}$ & $10^{-5}$ \\ \hline
		Epochs  & $5501$& $4501$ & $3501$ & $2501$  \\ \hline
	\end{tabular}%
\end{table}
For the symplectic reduction case, we used a similar deep convolutional network (DCA) structure as the one given in \cite{buchfink2021symplectic}. In \Cref{fig:dcaschematic}, we give the details of the auto-encoder structure.
\begin{figure*}
	\centering
	\includegraphics[width=1\linewidth,height=0.15\textheight]{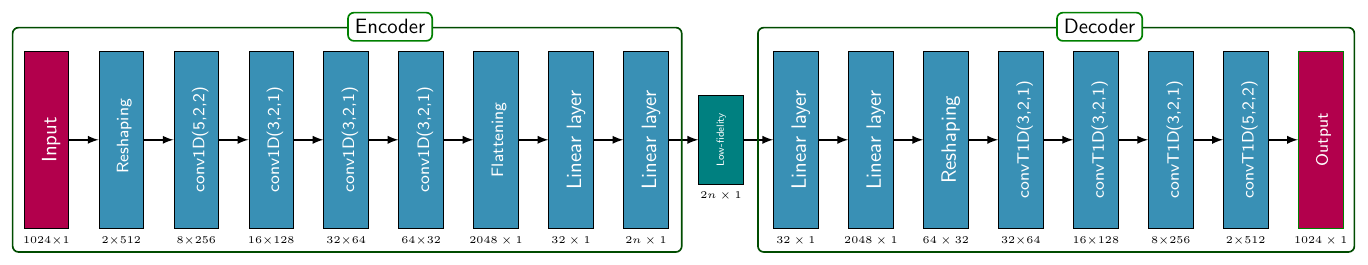}
	\caption{The figure summarises the encoder and decoder architectures. conv1D$(k,s,p)$ denotes a 1D convolution layer with kernel size $k$, stride size $s$, padding size $p$ and similarly, convT1D$(k,s,p)$ is a 1D transpose convolution layer with transpose kernel size $k$, stride size $s$, padding size $p$. We have used output padding size $1$ to obtain symmetric auto-encoder structure in 1D transpose convolution layers. We denote the size of the output block below each block.}
	\label{fig:dcaschematic}
\end{figure*}

\begin{table}
	\renewcommand{\arraystretch}{1.25}
	\caption{The table contains all the hyper-parameters to learn the dynamics of the high-dimensional examples. }
	\label{tab:hyperparameters-high}
	\begin{tabular}{|c|c|c|}
		\hline
		{\bfseries Parameters}                                                              & \begin{tabular}[c]{@{}c@{}}{\bfseries Wave} \\ {\bfseries example}\end{tabular} & \begin{tabular}[c]{@{}c@{}}{\bfseries NLS} \\ {\bfseries example}\end{tabular} \\ \hline
		\begin{tabular}[c]{@{}c@{}}Lifted coordinate \\ system dimension\end{tabular} & 4                                                           & 2                                                                                                                        \\ \hline
		Learning rate   & $10^{-3}$    &$10^{-3}$      \\ \hline
		Batch size                                                             & $50$                                                    & $50$                                                                                                                      \\ \hline
		Activation function                                                    & \texttt{selu}                              & \texttt{selu}                                                                       \\ \hline
		Weight decay                                                           & $10^{-5}$                                                   &     $10^{-5}$                                                                                         \\ \hline
		Epochs                                                                 & 4501                                                        &    3501                                                                                                                        \\ \hline
		Tolerance                                                                 & $5\cdot 10^{-2}$                                            &    $1\cdot 10^{-2}$                                                             \\ \hline
	\end{tabular}
\end{table}

\addcontentsline{toc}{section}{References}
\bibliographystyle{ieeetr}
\bibliography{ref}

\begin{thebibliography}{10}

\bibitem{arnol1989mathematical}
V.~I. Arnol'd, {\em Mathematical methods of classical mechanics}.
\newblock Springer New York, 1989.

\bibitem{siegel1995lectures}
C.~L. Siegel and J.~K. Moser, {\em Lectures on celestial mechanics}.
\newblock Springer Berlin, Heidelberg, 1995.

\bibitem{salmon1988hamiltonian}
R.~Salmon, ``Hamiltonian fluid mechanics,'' {\em Annu. Rev. Fluid Mech.},
  vol.~20, no.~1, pp.~225--256, 1988.

\bibitem{faou2012geometric}
E.~Faou, {\em Geometric numerical integration and {S}chr{\"o}dinger equations},
  vol.~15.
\newblock European Mathematical Society, 2012.

\bibitem{crutchfield2012between}
J.~P. Crutchfield, ``Between order and chaos,'' {\em Nat. Phys.}, vol.~8,
  no.~1, pp.~17--24, 2012.

\bibitem{lusch2018deep}
B.~Lusch, J.~N. Kutz, and S.~L. Brunton, ``Deep learning for universal linear
  embeddings of nonlinear dynamics,'' {\em Nat. Commun.}, vol.~9, no.~1,
  p.~4950, 2018.

\bibitem{vlachas2018data}
P.~R. Vlachas, W.~Byeon, Z.~Y. Wan, T.~P. Sapsis, and P.~Koumoutsakos,
  ``Data-driven forecasting of high-dimensional chaotic systems with long
  short-term memory networks,'' {\em Proc. Math. Phys. Eng. Sci.}, vol.~474,
  no.~2213, p.~20170844, 2018.

\bibitem{fang2020neural}
R.~Fang, D.~Sondak, P.~Protopapas, and S.~Succi, ``Neural network models for
  the anisotropic {R}eynolds stress tensor in turbulent channel flow,'' {\em J.
  Turbul.}, vol.~21, no.~9-10, pp.~525--543, 2020.

\bibitem{duraisamy2019turbulence}
K.~Duraisamy, G.~Iaccarino, and H.~Xiao, ``Turbulence modeling in the age of
  data,'' {\em Annu. Rev. Fluid Mech.}, vol.~51, pp.~357--377, 2019.

\bibitem{karim2019multivariate}
F.~Karim, S.~Majumdar, H.~Darabi, and S.~Harford, ``Multivariate {LSTM}-{FCN}s
  for time series classification,'' {\em Neural Netw.}, vol.~116, pp.~237--245,
  2019.

\bibitem{raissi2017inferring}
M.~Raissi, P.~Perdikaris, and G.~E. Karniadakis, ``Inferring solutions of
  differential equations using noisy multi-fidelity data,'' {\em J. Comput.
  Phys.}, vol.~335, pp.~736--746, 2017.

\bibitem{rudy2017data}
S.~H. Rudy, S.~L. Brunton, J.~L. Proctor, and J.~N. Kutz, ``Data-driven
  discovery of partial differential equations,'' {\em Sci. Adv.}, vol.~3,
  no.~4, p.~e1602614, 2017.

\bibitem{chen2019symplectic}
Z.~Chen, J.~Zhang, M.~Arjovsky, and L.~Bottou, ``Symplectic recurrent neural
  networks,'' {\em arXiv preprint arXiv:1909.13334}, 2019.

\bibitem{greydanus2019hamiltonian}
S.~Greydanus, M.~Dzamba, and J.~Yosinski, ``Hamiltonian neural networks,'' in
  {\em Advances in Neural Information Processing Systems} (H.~Wallach,
  H.~Larochelle, A.~Beygelzimer, F.~d\textquotesingle Alch\'{e}-Buc, E.~Fox,
  and R.~Garnett, eds.), vol.~32, Curran Associates, Inc., 2019.

\bibitem{finzi2020simplifying}
M.~Finzi, K.~A. Wang, and A.~G. Wilson, ``Simplifying {H}amiltonian and
  {L}agrangian neural networks via explicit constraints,'' in {\em Advances in
  Neural Information Processing Systems} (H.~Larochelle, M.~Ranzato,
  R.~Hadsell, M.~Balcan, and H.~Lin, eds.), vol.~33, pp.~13880--13889, Curran
  Associates, Inc., 2020.

\bibitem{offen2022symplectic}
C.~Offen and S.~Ober-Bl{\"o}baum, ``Symplectic integration of learned
  {H}amiltonian systems,'' {\em Chaos}, vol.~32, no.~1, p.~013122, 2022.

\bibitem{sharma2022preserving}
H.~Sharma and B.~Kramer, ``Preserving {Lagrangian} structure in data-driven
  reduced-order modeling of large-scale dynamical systems,'' {\em arXiv
  preprint arXiv:2203.06361}, 2022.

\bibitem{sharma2022hamiltonian}
H.~Sharma, Z.~Wang, and B.~Kramer, ``Hamiltonian operator inference:
  {Physics}-preserving learning of reduced-order models for canonical
  {Hamiltonian} systems,'' {\em Physica D: Nonlinear Phenomena}, vol.~431,
  p.~133122, 2022.

\bibitem{gruber2023canonical}
A.~Gruber and I.~Tezaur, ``Canonical and noncanonical {Hamiltonian} operator
  inference,'' {\em arXiv preprint arXiv:2304.06262}, 2023.

\bibitem{goyal2022generalized}
P.~Goyal and P.~Benner, ``Generalized quadratic-embeddings for nonlinear
  dynamics using deep learning,'' {\em arXiv preprint arXiv:2211.00357}, 2022.

\bibitem{buchfink2021symplectic}
P.~Buchfink, S.~Glas, and B.~Haasdonk, ``Symplectic model reduction of
  {H}amiltonian systems on nonlinear manifolds and approximation with weakly
  symplectic autoencoder,'' {\em {SIAM} J. Sci. Comput.}, vol.~45, no.~2,
  pp.~A289--A311, 2023.

\bibitem{peng2016symplectic}
L.~Peng and K.~Mohseni, ``Symplectic model reduction of {H}amiltonian
  systems,'' {\em {SIAM} J. Sci. Comput.}, vol.~38, no.~1, pp.~A1--A27, 2016.

\bibitem{afkham2017structure}
B.~Maboudi~Afkham and J.~S. Hesthaven, ``Structure preserving model reduction
  of parametric {H}amiltonian systems,'' {\em {SIAM} J. Sci. Comput.}, vol.~39,
  no.~6, pp.~A2616--A2644, 2017.

\bibitem{pagliantini2022gradient}
C.~Pagliantini and F.~Vismara, ``Gradient-preserving hyper-reduction of
  nonlinear dynamical systems via discrete empirical interpolation,'' {\em
  {SIAM} J. Sci. Comput.}, vol.~45, no.~5, pp.~A2725--A2754, 2023.

\bibitem{pagliantini2021dynamical}
C.~Pagliantini, ``Dynamical reduced basis methods for {Hamiltonian} systems,''
  {\em Numerische Mathematik}, vol.~148, no.~2, pp.~409--448, 2021.

\bibitem{musharbash2020symplectic}
E.~Musharbash, F.~Nobile, and E.~Vidli{\v{c}}kov{\'a}, ``Symplectic dynamical
  low rank approximation of wave equations with random parameters,'' {\em BIT
  Numerical Mathematics}, vol.~60, pp.~1153--1201, 2020.

\bibitem{hesthaven2022reduced}
J.~S. Hesthaven, C.~Pagliantini, and G.~Rozza, ``Reduced basis methods for
  time-dependent problems,'' {\em Acta Numerica}, vol.~31, pp.~265--345, 2022.

\bibitem{sharma2023symplectic}
H.~Sharma, H.~Mu, P.~Buchfink, R.~Geelen, S.~Glas, and B.~Kramer, ``Symplectic
  model reduction of {H}amiltonian systems using data-driven quadratic
  manifolds,'' {\em Comput. Methods Appl. Mech. Eng.}, vol.~417, p.~116402,
  2023.

\bibitem{lee2012smooth}
J.~M. Lee, {\em Introduction to Smooth Manifolds}.
\newblock Springer, 2012.

\bibitem{savageau1987recasting}
M.~A. Savageau and E.~O. Voit, ``Recasting nonlinear differential equations as
  {S}-systems: a canonical nonlinear form,'' {\em Math. Biosci.}, vol.~87,
  no.~1, pp.~83--115, 1987.

\bibitem{gu2011qlmor}
C.~Gu, ``{QLMOR}: A projection-based nonlinear model order reduction approach
  using quadratic-linear representation of nonlinear systems,'' {\em IEEE
  Trans. Comput. Aided Des. Integr. Circuits. Syst.}, vol.~30, no.~9,
  pp.~1307--1320, 2011.

\bibitem{qian2020lift}
E.~Qian, B.~Kr{\"{a}}mer, B.~Peherstorfer, and K.~Willcox, ``Lift \& learn:
  Physics-informed machine learning for large-scale nonlinear dynamical
  systems,'' {\em Physica D}, vol.~406, no.~1, p.~art. 132401, 2020.

\bibitem{mattheakis2022hamiltonian}
M.~Mattheakis, D.~Sondak, A.~S. Dogra, and P.~Protopapas, ``Hamiltonian neural
  networks for solving equations of motion,'' {\em Phys. Rev. E}, vol.~105,
  no.~6, p.~065305, 2022.

\bibitem{strauch2009classical}
D.~Strauch, {\em Classical Mechanics: An Introduction}.
\newblock Springer Berlin Heidelberg, 2009.

\bibitem{KolB09}
T.~G. Kolda and B.~W. Bader, ``Tensor decompositions and applications,'' {\em
  {SIAM} Rev.}, vol.~51, no.~3, pp.~455--500, 2009.

\bibitem{comon2008symmetric}
P.~Comon, G.~Golub, L.-H. Lim, and B.~Mourrain, ``Symmetric tensors and
  symmetric tensor rank,'' {\em {SIAM} J. Matrix Anal. Appl.}, vol.~30, no.~3,
  pp.~1254--1279, 2008.

\bibitem{tong2021symplectic}
Y.~Tong, S.~Xiong, X.~He, G.~Pan, and B.~Zhu, ``Symplectic neural networks in
  {T}aylor series form for {H}amiltonian systems,'' {\em J. Comput. Phys.},
  vol.~437, p.~110325, 2021.

\bibitem{choudhary2021forecasting}
A.~Choudhary, J.~F. Lindner, E.~G. Holliday, S.~T. Miller, S.~Sinha, and W.~L.
  Ditto, ``Forecasting {Hamiltonian} dynamics without canonical coordinates,''
  {\em Nonlinear Dynamics}, vol.~103, pp.~1553--1562, 2021.

\bibitem{savitzky1964smoothing}
A.~Savitzky and M.~J. Golay, ``Smoothing and differentiation of data by
  simplified least squares procedures.,'' {\em Analytical chemistry}, vol.~36,
  no.~8, pp.~1627--1639, 1964.

\bibitem{choudhary2020}
A.~Choudhary, ``Forecasting {Hamiltonian} dynamics without canonical
  coordinates..'' \url{https://github.com/anshu957/gHNN}, 2020.

\bibitem{bridges2006}
T.~J. Bridges and S.~Reich, ``Numerical methods for {Hamiltonian PDEs},'' {\em
  J. Phys. A Math. Theor.}, vol.~39, no.~19, p.~5287, 2006.

\bibitem{karasozen2018energy}
B.~Karas{\"o}zen and M.~Uzunca, ``Energy preserving model order reduction of
  the nonlinear {S}chr{\"o}dinger equation,'' {\em Adv. Comput. Math.},
  vol.~44, no.~6, pp.~1769--1796, 2018.

\bibitem{morGoyB23}
P.~Goyal and P.~Benner, ``Neural ordinary differential equations with irregular
  and noisy data,'' {\em Roy. Soc. Open Sci.}, vol.~10, no.~7, p.~221475, 2023.

\bibitem{QuadHam23}
S.~Y{\i}ld{\i}z, P.~Goyal, T.~Bendokat, and P.~Benner, ``Data-driven
  identification of quadratic representations for nonlinear {Hamiltonian}
  systems using weakly symplectic liftings.'' Zenodo, 2023.
\newblock \url{doi.org/10.5281/zenodo.10053535}.

\end{thebibliography}

\end{document}